\documentclass[12pt]{article}
\usepackage{amsmath,amssymb,mathtools,algpseudocode,algorithm,indentfirst,amsthm}
\usepackage[x11names,table]{xcolor}
\usepackage{etoolbox,graphicx}
\usepackage[toc,page]{appendix}
\usepackage{subcaption}
\usepackage[english]{babel}
\usepackage{verbatim}
\usepackage[margin=3cm]{geometry}
\usepackage{url}
\usepackage{bbm}
\usepackage{algorithm}
\usepackage{algorithmicx}
\usepackage{accents}
\usepackage{ulem}
\usepackage{hyperref}
\usepackage{authblk}

\DeclareMathOperator*{\argmax}{argmax}

\DeclarePairedDelimiter\abs{\lvert}{\rvert}

\newcommand{\dprod}[2]{\left\langle #1,#2\right\rangle}
\newcommand{\norm}[1]{\left\lVert #1\right\rVert}

\algnewcommand{\Initialize}[1]{
  \State \textbf{Initialize:}\:#1
}
\newcommand{\prob}[1]{\mathbb{P}\left(#1\right)}

\newtheorem{theorem}{Theorem}

\newtheorem{lemma}[theorem]{Lemma}
\newtheorem{proposition}[theorem]{Proposition}

\theoremstyle{definition}
\newtheorem{remark}{Remark}
\newtheorem{definition}[theorem]{Definition}
\numberwithin{equation}{section}
\AfterEndEnvironment{theorem}{\noindent\ignorespaces}
\AfterEndEnvironment{conjecture}{\noindent\ignorespaces}
\AfterEndEnvironment{definition}{\noindent\ignorespaces}

\newcommand{\gconv}{
	\star_{\scriptscriptstyle{G}}
}

\newcommand{\man}{
	\mathcal{M}
}

\newcommand{\Hspace}{
	\mathcal{H}
}

\newcommand{\I}{
	\mathcal{I}
}

\usepackage{stackengine}

\begin{document}
	\title{$G$-invariant Diffusion maps}
	\date{}

\author[1]{Eitan Rosen}
\author[2]{Xiuyuan Cheng}
\author[1]{Yoel Shkolnisky}

\affil[1]{Department of Applied Mathematics, Tel-Aviv University}
\affil[2]{Department of Mathematics, Duke University}

\maketitle
\begin{abstract}
	The diffusion maps embedding of data lying on a manifold has shown success in tasks such as dimensionality reduction, clustering, and data visualization. In this work, we consider embedding data sets that were sampled from a manifold which is closed under the action of a continuous matrix group. An example of such a data set is images whose planar rotations are arbitrary. The $G$-invariant graph Laplacian, introduced in Part I of this work, admits eigenfunctions in the form of tensor products between the elements of the irreducible unitary representations of the group and eigenvectors of certain matrices. We employ these eigenfunctions to derive diffusion maps that intrinsically account for the group action on the data. In particular, we construct both equivariant and invariant embeddings, which can be used to cluster and align the data points. We demonstrate the utility of our construction in the problem of random computerized tomography.
\end{abstract}

\section{Introduction}\label{secIntro}
Data analysis has become a cornerstone in many applications in science and engineering. In a typical setup, we are given a data set $X=\left\{x_{1},\ldots,x_{N}\right\}\subset \mathbb{C}^{n}$ where $n$ is large, and the goal is to discover structures in the data. Examples for such data include collections of images, point clouds, biological measurements, financial data, etc. A common modeling approach for high-dimensional data is that the given points lie on a manifold~$\man$ with intrinsic dimension $d\ll n$. 

In various applications, the data set under consideration is assumed to have been sampled from a manifold~$\man$ that is closed under the action of some known matrix group~$G$. In other words, for each point $x_i\in X$ and any $A\in G$, the data point~$Ax_i$ is also in~$\man$. Such manifolds are dubbed~$G$-invariant. A particularly important class of groups appearing in numerous scientific applications are compact matrix Lie groups, which includes the groups~$SO(2)$ and~$SO(3)$ of~2D and~3D spatial rotations, respectively, and low-dimensional tori, all ubiquitous in fields such as image processing and computer vision.
For example, the data set may be a collection of two-dimensional images whose in-plane rotation is arbitrary, as is the case for electron microscopy image data sets \cite{SingerSigworthCryoMethods}. In the electron microscopy setting, all images lie on a manifold of dimension three, and moreover, rotating each of the input images results in another image that may have been generated by the microscope. The current paper provides tools for processing such data sets. 

When processing and analyzing data that are assumed to have been sampled from a smooth manifold, a fundamental object of interest is the Laplace-Beltrami operator $\Delta_\man$~\cite{rosenberg1997laplacian}, which generalizes the classical Laplacian. The Laplace-Beltrami operator and its discrete counterpart, the graph Laplacian~\cite{belkin2003laplacian}, have been used to process surfaces, images, and general manifold data~\cite{meyer2014perturbation,taubin1995signal,vallet2008spectral,liu2014progressive,desbrun1999implicit,hein2006manifold,osher2017low}.
Formally, the graph Laplacian is defined as the~$N\times N$ matrix~$L$ given by
\begin{equation}\label{intro:classicalGLDef}
	L = D-W,\quad W_{ij} = e^{-\norm{x_i-x_j}^2/\epsilon},\quad D_{ii} = \sum_{j=1}^{N}W_{ij},
\end{equation}
where $\epsilon$ is a bandwidth parameter, $D$ is a diagonal matrix, and~$\norm{\cdot}$ is the~2-norm. The normalized graph Laplacian is defined by 
\begin{equation}\label{intro:normGLDef}
	\tilde{L} = D^{-1}L.
\end{equation}
In \cite{diffMaps}, it was shown that if $X$ was sampled uniformly from a smooth and compact manifold $\man$, then $\tilde{L}$ converges to the Laplace-Beltrami operator $\Delta_\man$ as~$\epsilon\rightarrow 0$ and~$N\rightarrow \infty$. Recently, it was also shown that the eigenvectors and eigenvalues of~$\tilde{L}$ converge to the eigenfunctions and eigenvalues of~$\Delta_\man$ \cite{lapEigConverge}. 

In Part~I of this work~\cite{GGL}, we laid the foundations for a graph Laplacian based framework for analyzing~$G$-invariant data, where~$G$ is a compact matrix Lie group, and demonstrated its usage for denoising manifold data. Specifically, we showed how to construct the graph Laplacian by using all the points in the data set $G\cdot X$, resulting from applying all the elements in~$G$ to the points in~$X$. This graph Laplcaian is termed the $G$-invariant graph Laplacian (abbreviated $G$-GL). We then obtained two fundamental results regarding the~$G$-GL. 
Firstly, we proved that like the graph Laplacian, the~$G$-GL also converges to the operator~$\Delta_\man$, but at a significantly improved rate, where the speedup scales with the dimension of~$G$. Secondly, we derived the eigenfunctions of the $G$-GL, showing that they admit the form of tensor products between the eigenvectors of a sequence of certain matrices and the elements of the irreducible unitary representations (abbreviated as IURs) of~$G$. This special form of the eigenfunctions gives rise to efficient algorithms for their computation. We then demonstrated how these eigenfunctions can be used to filter noisy samples from a~$G$-invariant manifold~$\man$.  

In the current work, we develop~$G$-GL based embeddings of the data. In particular, we derive several variants of one the most successful graph Laplacian based algorithms, known as `diffusion maps'~\cite{diffMaps}, which we now briefly describe. Let $\phi_1,\ldots, \phi_N$ and $\lambda_1,\ldots,\lambda_N$ be eigenvectors and eigenvalues of the matrix~$P=I-\tilde{L}$, respectively (where~$\tilde{L}$ was defined in~\eqref{intro:normGLDef}). Then, the data set $X$ can be embedded into a~$N$-dimensional Euclidean space by
\begin{equation}\label{intro:diffMapsDef}
	\Phi_t(x_i) = \left(\lambda_1^t\phi_1(i),\ldots,\lambda_N^t\phi_N(i)\right), 
\end{equation}
for some $t\in \mathbb{R}^+$ and $x_i\in X$, where $\phi_k(i)$ is the $i$-th entry of $\phi_k$.
The embedding \eqref{intro:diffMapsDef} has an elegant probabilistic interpretation. Observe that $P$ is a row stochastic matrix, which we can view as the transition probability matrix of a random walk over~$X$. We define a family of distances~$D_t(x_i,x_j)$ where $t=1,2,\ldots$,  between pairs of point $x_i,x_j\in X$ by 
\begin{equation}\label{intro:diffDist}
	D_t^2(x_i,x_j) = \norm{P^t_{i,\cdot}-P^t_{j,\cdot}}_{w}^2 = \sum_{k=1}^{N} (P^t_{ik}-P^t_{jk})^2 w_k, \quad w_k = \frac{1}{D_{kk}}, 
\end{equation}
where~$P^t_{i,\cdot}$ denotes the~$i$-th row of~$P^t$, and $P^t_{ij}$ is the $ij$-th element of the matrix $P^t$, and~$D_{kk}$ is defined in~\eqref{intro:classicalGLDef}.
The weights~$w_k$ are larger for nodes~$x_k$ which are weakly connected to their neighbors and are thus more difficult to reach.  
This fact makes the distance~\eqref{intro:diffDist} a natural choice for organizing data into clusters, as a cluster is in essence a set in which every data point resides in a neighborhood which is dense with other data points. Thus, any path between two data points residing in distinct clusters would have to pass through sparsely inhabited regions of the data landscape, making the distance~\eqref{intro:diffDist} for such paths longer than paths within a cluster.   
In \cite{diffMaps}, it is shown that 
\begin{equation}\label{intro:diffDistIdentity}
		D_t^2(x_i,x_j) = \norm{\Phi_t(x_i)-\Phi_t(x_j)}^2, 
\end{equation}
which implies that the Euclidean distance between the embeddings~$\Phi_t(x_i)$ and~$\Phi_t(x_j)$ is the distance at time $t$ between (the distributions of) a pair of random walks over~$X$ that depart from~$x_i$ and~$x_j$. 

In this work, we utilize the~$G$-GL's eigenfunctions to construct various diffusion maps embeddings, thereby, explicitly incorporating the action of~$G$ on the data set~$X$ into them. We distinguish between two types of embeddings - equivariant and invariant. For a fixed~$q\in\mathbb{N}$, an equivariant embedding is a map~$\phi:X\rightarrow \mathbb{C}^q$ such that 
\begin{equation}\label{intro:equivEmbedd}
	\phi(A\cdot x) = A\circ \phi(x), \quad A\in G,
\end{equation}
where `$\cdot$' denotes the action of $G$ on a data point $x\in \mathbb{C}^n$, and~`$\circ$' denotes the action of~$G$ on the embedded points, which typically reside in a different space than $\mathbb{C}^n$ of a much lower dimension. On the other hand, an invariant embedding is a map~$\psi:X\rightarrow \mathbb{C}^q$ such that
\begin{equation}\label{intro:invEmbedd}
	\psi(A\cdot x) = \psi(x), \quad A\in G,
\end{equation}
for all~$x\in X$. 
The difference between the two types of embeddings is that $\phi$ preserves the action of~$G$ through \eqref{intro:equivEmbedd}, while $\psi$ maps all the points $y$ in the orbit $G\cdot x$ into the same point~$\psi(x)\in \mathbb{C}^q$.

An important application of such embeddings is when~$X$ is a set of noisy~2D images, with the group $G=SO(2)$ of planar rotations acting on an image by rotating it around its center. In this case, an invariant embedding can be employed to cluster images which are (almost) identical up to a planar rotation. 
Specifically, the property~\eqref{intro:invEmbedd} implies that for a pair of images~$x_i,x_j\in X$ such that~$x_i \approx A\cdot x_j$ for some rotation matrix~$A\in SO(2)$, that is,~$x_i$ is approximately a rotation of~$x_j$, it holds that $\psi(x_i) \approx \psi(x_j)$. We call such images rotationally-invariant neighbors. On the other hand, the property~\eqref{intro:equivEmbedd} implies that if~$x_i \approx A\cdot x_j$, then~$\phi(x_i) \approx A\circ \phi(x_j)$. In other words, an approximate rotation of an image results in an approximate rotation of its embedding. 
Hence, we can utilize~$\phi$ to align the rotationally-invariant neighbors of each image~$x_i$ with it, by computing for each neighbor~$x_j$ the element~$A\in SO(2)$ that minimizes the distance~$\norm{\phi(x_i) - A\circ \phi(x_j)}$. 
Assuming that the noise in each image is additive, i.i.d with mean zero, we expect that averaging the aligned nearest neighbors of an image~$x_i$ will result in a denosied image which closely approximates~$x_i$. 
Furthermore, it is typically possible to align pairs of images using an embedding~$\phi$ of dimension far lower than that of the images in~$X$, which is faster than aligning the images by directly rotating them. 

Various methods have been proposed in literature for constructing group-invariant embeddings, especially for rotational-invariance in image processing and computer vision tasks~\cite{singer2012vector,zhao2014rotationally,multFreqVDM,compactRigidImProc}. Most of these methods typically follow one of two approaches.
The first approach is based on steerable-PCA~\cite{FFBsPCA2015,landa2017steerable} that computes the PCA of a data set of images and all of their planar rotations. In other words, steerable PCA computes an embedding of the images and all their infinitely many rotations into a low-dimensional subspace which best captures their geometry. Importantly, steerable PCA efficiently computes the sample covariance matrix of all infinitely many rotations of the images, avoiding any data augmentation, which can be computationally prohibitive. In~\cite{steerMaps}, the latter idea was extended to a data set of images residing on a compact manifold by introducing the steerable graph Laplacian, which is conceptually equivalent to the graph Laplacian~$L$ in~\eqref{intro:classicalGLDef} constructed from all the images and their infinitely many rotations. In Part~I of this work~\cite{GGL}, we generalized the steerable graph Laplacian to arbitrary compact manifolds, which are closed under the action of an arbitrary compact matrix Lie group.
The second approach, is to construct a graph Laplacian by using some rotationally-invariant pairwise distance between the points, and then use this graph Laplacians' eigenvectors and eigenvalues to define a rotationally-invariant embedding \cite{singer2012vector}. This approach was applied to handle the case of arbitrary compact Lie groups in~\cite{fan2019unsupervised}. However, it is not always obvious which invariant distance is the most appropriate for the task, or how to compute it efficiently. Furthermore, in general, the limiting operator resulting from such a construction is either unknown, or is not the Laplace-Beltrami operator~\cite{manLearnArbitraryNorm}, in which case its properties are not always well understood. 
On the other hand, the embedding proposed in this work employs the eigenfunctions and eigenvalues of the~$G$-GL, which converges to the Laplace-Beltrami operator on the manifold, allowing us to preserve the geometry of the underlying manifold. In particular, the group-invariance of the manifold is explicitly manifested in the form of the~$G$-GL's eigenfunctions, being a product between certain vectors and the IURs of~$G$ (given by Theorem~\ref{secGGL:GGLdecomp} in the following section.). As we show below, this enables us to embed not only the points in given the data set, but also all the points generated by the action of the group on these points. 

The contributions of this work are as follows. First, we employ the~$G$-GL's eigenfunctions to construct an equivariant diffusion maps embedding of the data set, and an associated equivariant diffusion distance, analogous to~\eqref{intro:diffMapsDef} and~\eqref{intro:diffDist}, respectively. We then analyze their properties. Next, we employ the~$G$-GL's eigenfunctions to construct an invariant diffusion maps embedding of the data, 
and an associated invariant diffusion distance, and analyze their properties. In particular, we show that the latter distance is associated with certain random walks on the data manifold, and derive a relation of the form~\eqref{intro:diffDistIdentity} for  the invariant distance and embedding. Finally, we demonstrate the utility of the proposed embeddings in reconstructing a 2D image from its noisy and shifted 1D Radon transform projections~\cite{glRandTomography}.  

This paper is organized as follows. In Section~\ref{sec:GGLreview}, we briefly review the $G$-invariant graph Laplacian framework developed in~\cite{GGL}. In Section~\ref{sec:GEqvDmaps}, we derive $G$-equivariant embeddings of the data, and diffusion distances over the set~$G\cdot X$, which are analogous to~\eqref{intro:diffMapsDef}. In Section~\ref{sec:invDmaps}, we derive and analyze $G$-invariant embeddings of the data. In Section~\ref{sec:CryoSimulations}, we demonstrate the utility of our framework using numerical simulations. Lastly, in Section~\ref{sec:Summary}, we summarize our results and discuss future work.  

\section{The $G$-invariant graph Laplacian}\label{sec:GGLreview}
In this section, we review the definition and eigendecompostion of the $G$-invariant graph Laplacian, introduced in \cite{GGL}. 
Let $X = \left\{x_1,\ldots,x_N\right\}$ be a data set sampled from a uniform distribution over a smooth~$G$-invariant manifold.

Denoting~$[N]=\left\{1,\ldots,N\right\}$, consider the set 
\begin{equation}\label{GinvDef:Gammadef}
	[N]\times G=\left\{(i,A): \; i\in[N], \; A\in G\right\},
\end{equation}
where we use the pair~$(i,A)\in[N]\times G$ to refer to the point~$A\cdot x_i$ resulting from applying the element~$A\in G$ (by matrix multiplication) to the point~$x_i\in X$. 
Now, let~$\mathcal{H}=L^2([N]\times G)$ be the Hilbert space of functions of the form 
\begin{equation}
	f(i,A) = f_i(A), \quad A\in G,
\end{equation}
where~$f_i\in L^2(G)$, endowed with the inner product 
\begin{equation}
	\dprod{f}{g}_{\Hspace} = \sum_{i=1}^{N}\int_G f_i(A)\overline{g_i(A)}d\eta(A),	
\end{equation}
where $\eta(\cdot)$ is the Haar measure on~$G$ (see \ref{secLieGroupAction} for a brief introduction to integration over Lie groups). 
Furthermore, let us define the action of any~$N\times N$ diagonal matrix~$D$ on functions~$f\in \Hspace$ by
\begin{equation}\label{GinvDef:DMatAction}
	\left\{Df\right\}(i,A) = D_{ii} \cdot f(i,A) = D_{ii}\cdot f_i(A).
\end{equation}
\begin{definition}[The~$G$-invariant graph Laplacian]
	Let $W:\Hspace\rightarrow\Hspace$ be the operator defined as
	\begin{equation}\label{GinvDef:Wdef}
		W\left\{f\right\}(i,A) = \sum_{j=1}^{N}\int_{G} W_{ij}(A,B)f_j(B)d\eta(B), \quad W_{ij}(A,B) = e^{-\norm{A\cdot x_i-B\cdot x_j}^2/\epsilon},
	\end{equation}
	where~$\norm{\cdot}$ is the~$2$-norm over~$\mathbb{C}^n$, and let~$D$ be the diagonal $N\times N$ matrix whose diagonal entries are given by
	\begin{equation}\label{GinvDef:DiiDef}
		D_{ii} = \sum_{j=1}^N\int_{G}W_{ij}(I,C)d\eta(C), \quad i\in\left\{1,\ldots,N\right\}. 
	\end{equation}
	The $G$-invariant graph Laplacian (abbreviated $G$-GL) $L:\Hspace\rightarrow\Hspace$ is defined as
	\begin{equation}\label{GinvDef:Ldef}
		Lf = Df-Wf, \quad f\in \mathcal{H}. 
	\end{equation}
	The normalized $G$-GL is defined as the operator
	\begin{equation}\label{GGL:normalizedGGLDef}
		\tilde{L} = D^{-1}L = I-D^{-1}W,
	\end{equation}
	that is, for any $f\in\mathcal{H}$ we have
	\begin{equation}\label{GinvDef:normLdef}
		\tilde{L}\left\{f\right\}(i,A) =f_i(A)-\frac{1}{D_{ii}}\cdot \sum_{i=1}^{N}\int_{G} W_{ij}(A,B)f_j(B)d\eta(B).
	\end{equation}
\end{definition}

We observe that since~$G$ is a unitary matrix group, $W_{ij}(A,B)$ in~\eqref{GinvDef:Wdef} satisfies
\begin{equation}\label{GinvDef:WijShiftInvar}
	W_{ij}(A,B) = e^{-\norm{Ax_i-Bx_j}^2/\epsilon} = e^{-\norm{x_i-A^*Bx_j}^2/\epsilon} = W_{ij}(I,A^*B),
\end{equation}
where~$A^*$ is the conjugate-transpose of the matrix~$A\in G$. 
By the Peter-Weyl theorem (see~\ref{secHarmAnalysis}), the function~$W_{ij}(I,\cdot)$ over~$G$ admits an expansion in a series of the entries of the irreducible unitary representations (abbreviated IURs) of~$G$. The IURs of~$G$ are a countably-infinite sequence~$\left\{U^\ell\right\}_{\ell\in \I_G}$ of unitary matrix-valued functions over~$G$, where~$U^\ell(A)$ is a unitary matrix of dimensions~$d_\ell \times d_\ell$ for each~$A\in G$, and~$\I_G$ is a set that enumerates the sequence. Explicitly, by using~\eqref{secLieGroupAction:GFourierMatCoeff} (see~\ref{secLieGroupAction}) we get that 
\begin{equation}\label{sec1:fourierSO3}
	W_{ij}(I,A^*B)=\sum_{\ell\in \I_G } d_\ell\cdot \text{trace}\left(\hat{W}^\ell_{ij}U^\ell(A^*B)\right),
\end{equation}
where~$\hat{W}_{ij}^{\ell}$ is the~$d_\ell\times d_\ell$ matrix given by
\begin{equation}\label{sec2:hat{W}Def}
	\hat{W}_{ij}^\ell= \int_{G} W_{ij}(I,A)\overline{U^\ell(A)}d\eta(A),\quad \ell\in \I_G,
\end{equation}
with~$\overline{U^\ell(A)}$ denoting the entry-wise complex-conjugate of the matrix~$U^\ell(A)$.
We now state a theorem (derived in \cite{GGL}) that characterizes the eigendecomposition of~$\tilde{L}$ of~\eqref{GGL:normalizedGGLDef} in terms of certain products between the rows of the IURs $U^\ell$ and the eigenvectors of the block matrices 
\begin{equation}\label{secGGL:blockFourierMat}
	\hat{W}^{(\ell)} = 	\begin{pmatrix}
		\hat{W}^\ell_{11} & \hat{W}^\ell_{12}& ... & \hat{W}^\ell_{1N}\\
		\vdots & \ddots &  & \vdots\\
		\vdots & & \ddots   & \vdots\\
		\hat{W}^\ell_{N1} & \hat{W}^\ell_{N2}& ... & \hat{W}^\ell_{NN}		
	\end{pmatrix}, 
	\quad \ell\in \I_G,
\end{equation}
where the matrix $\hat{W}^{(\ell)}$ of dimensions $Nd_\ell\times Nd_\ell$ has $\hat{W}^\ell_{ij}$ of~\eqref{sec2:hat{W}Def} as its~$ij$-th block.
Also, for any vector $v \in \mathbb{C}^{Nd_\ell}$ and $j\in\{1,\ldots,N\}$,  we denote by
\begin{equation}\label{secGGL:parVecNotation}
	e^j(v) = (v((j-1)\cdot d_\ell+1),\ldots,v(j\cdot d_\ell))^T\in \mathbb{C}^{d_\ell},
\end{equation}
the elements $(j-1)d_\ell+1$ up to $j\cdot d_\ell$ of $v$ stacked as a $d_\ell$-dimensional column vector.
\begin{theorem}(\cite{GGL})\label{secGGL:GGLdecomp}
	For each $\ell\in \I_G$, let 
	\begin{equation}\label{secGGL:DellDef}
		D^{(\ell)} = \textup{diag}\left(D_{11}\cdot I_{d_\ell\times d_\ell},\ldots,D_{NN}\cdot I_{d_\ell\times d_\ell}\right)
	\end{equation}	
	be the $Nd_\ell\times Nd_\ell$ block-diagonal matrix whose~$i^{\text{th}}$ block of size $d_\ell\times d_\ell$ on the diagonal is given by the product of the scalar $D_{ii}$ from \eqref{GinvDef:DiiDef} with the $d_\ell\times d_\ell$ identity matrix. 
	Then, the normalized $G$-invariant graph Laplacian $\tilde{L}$ in \eqref{GGL:normalizedGGLDef} admits the following:
	\begin{enumerate}
		\item A sequence of non-negative eigenvalues $\{\tilde{\lambda}_{1,\ell},\ldots,\tilde{\lambda}_{Nd_\ell,\ell}\}_{\ell\in \I_G}$, where $\tilde{\lambda}_{n,\ell}$ is the~$n^{\text{th}}$ eigenvalue of the matrix
		\begin{equation}\label{secGGL:fourierMatNorm}
			K^{(\ell)} = I-(D^{(\ell)})^{-1}\hat{W}^{(\ell)}.
		\end{equation}
		\item  A sequence $\{\tilde{\Phi}_{\ell,-\ell,1},\ldots,\tilde{\Phi}_{\ell,\ell,Nd_\ell}\}_{\ell\in \I_G}$ of eigenfunctions, which are complete in~$\mathcal{H}$, and are given by 
		\begin{equation}\label{secGGL:eigenFuncs}
			\tilde{\Phi}_{\ell,m,n}(i,A) = \sqrt{d_\ell}\cdot \overline{U^\ell_{m,\cdot}(A)}\cdot e^i(\tilde{v}_{n,\ell}),
		\end{equation}
		where $\tilde{v}_{n,\ell}$ is an eigenvector of $K^{(\ell)}$
		that corresponds to the eigenvalue~$\tilde{\lambda}_{n,\ell}$, and~$e^i$ is defined in~\eqref{secGGL:parVecNotation}. For each~$n\in \{1,\ldots, Nd_\ell\}$ and~$\ell\in\I_G$, the eigenfunctions $\{\tilde{\Phi}_{\ell,1,n},\ldots,\tilde{\Phi}_{\ell,d_\ell,n}\}$ correspond to the eigenvalue~$\tilde{\lambda}_{n,\ell}$ of~$\tilde{L}$. 
	\end{enumerate}
\end{theorem}
The following lemma, which implies that the elements~$e^i(v)$ of an eigenvector~$v$ of the matrix~$\hat{W}^{(\ell)}$ of~\eqref{secGGL:blockFourierMat} are~$G$-equivariant, lies at the heart of all subsequent derivations.  
\begin{lemma}\label{eqvDmaps:eqvEigenVecLemma}
	Suppose that $x_j =B\cdot x_i$ for some $B\in G$, and let $v$ be an eigenvector of~$\hat{W}^{(\ell)}$ given by~\eqref{secGGL:blockFourierMat}, that corresponds to an eigenvalue $\lambda>0$. Then, for all $\ell\in \I_G$ and $n\in \mathbb{N}$ we have
	\begin{equation}
		e^j(v) = \overline{U^\ell(B)}\cdot e^i(v).
	\end{equation}
\end{lemma}
\begin{proof}
	See Appendix~\ref{appSec3:eqvLemma}.
\end{proof}

\section{$G$-equivariant diffusion maps}\label{sec:GEqvDmaps} 
In this section, we derive an equivariant embedding of the data set~$X$, which we term `$G$-equivariant diffusion maps'. The derivation employs the eigenvectors and eigenvalues of the operator 
\begin{equation}\label{secEqDiffmaps:probOperatorDef}
	P_{op} = D^{-1}W = I-\tilde{L}, 
\end{equation}
where $\tilde{L}$ was defined in~\eqref{GGL:normalizedGGLDef}. Explicitly, by~\eqref{GinvDef:normLdef}, we get that
\begin{align}\label{eqvDmaps:PopExplicitDef}
		\left\{P_{op}f\right\}(i,A) =  \sum_{j=1}^{N}\int_G P((i,A),(j,B))f_j(B)d\eta(B), \quad P((i,A),(j,B)) = \frac{W_{ij}(A,B)}{D_{ii}}, 
\end{align}
where~$W_{ij}$ and~$D_{ii}$ are defined in~\eqref{GinvDef:Wdef} and~\eqref{GinvDef:DiiDef}, respectively. We observe that~$P_{op}$ may be viewed as a transition probability operator of a single step of a random walk over~$[N]\times G$ (defined in \eqref{GinvDef:Gammadef}).
Indeed, by~\eqref{GinvDef:Wdef} we have that~$P((i,A),(j,B))>0$, and moreover, if we denote by~$\mathbbm{1}_{\Hspace}
\in\Hspace$ the constant function that takes the value~$1$ over~$[N]\times G$, then by \eqref{eqvDmaps:PopExplicitDef} and \eqref{GinvDef:DiiDef} we have that
\begin{align}
	\left\{P_{op} \mathbbm{1}_{\mathcal{H}}\right\}(i,A) &=  \sum_{j=1}^{N}\int_GP((i,A),(j,B))d\eta(B) \label{eqvDmaps:probIdentity1}\\&=\sum_{j=1}^{N}\int_G\frac{W_{ij}(A,B)}{\sum_{k=1}^N \int_G  W_{ik}(A,C)d\eta(C)}d\eta(B)=1, \label{eqvDmaps:probIdentity}
\end{align}
for all $(i,A)\in [N]\times G$.
The expression on the right hand side of~\eqref{eqvDmaps:probIdentity1} is understood as the sum of transition probabilities from the point $(i,A)\in [N]\times G$ to all other points in~$[N]\times G$. 
It is easily seen that the eigenvalues of~$P_{op}$ are given by 
\begin{equation}\label{eqvDmaps:ProbOpEvals}
	\lambda_{n,\ell} = 1-\tilde{\lambda}_{n,\ell}, \quad n\in[N], \quad \ell\in \I_G,
\end{equation}
where $\{\tilde{\lambda}_{n,\ell}\}$ are the eigenvalues of~$\tilde{L}$ of~\eqref{GGL:normalizedGGLDef}, while for the eigenfunctions~$\{\Phi^{(p)}_{\ell,m,n}\}$ of~$P_{op}$ it holds that
\begin{equation}\label{eqvDmaps:ProbEvecs}
	\Phi^{(p)}_{\ell,m,n} = \tilde{\Phi}_{\ell,m,n}, \quad n\in [N], \quad \ell\in \I_G, \quad m\in [d_\ell],
\end{equation}
where~$\{\tilde{\Phi}_{\ell,m,n}\}$ from~\eqref{secGGL:eigenFuncs} are the eigenfunctions of~$\tilde{L}$. Therefore, the eigenfunctions of~$P_{op}$ can be computed by diagonalizing the matrices~$K^{(\ell)}$ in~\eqref{secGGL:fourierMatNorm}, as described by Theorem~\ref{secGGL:GGLdecomp}.  As the rest of this work concerns only the eigenpairs of~$P_{op}$, we associate the eigenvectors of~$K^{(\ell)}$ with the operator~$P_{op}$ by introducing the notation
\begin{equation}\label{eqvDmaps:vnlTilde2vnlp}
	v^{(p)}_{n,\ell} = \tilde{v}_{n,\ell}, \quad n\in[N], \quad \ell\in \I_G. 
\end{equation}

For $t\in\mathbb{N}$, we define the~$t$-step transition probability operator $P_{op}^t:\mathcal{H}\rightarrow \mathcal{H}$ to be the operator that acts on $f\in\mathcal{H}$ by applying~$P_{op}$ to~$f$ iteratively~$t$ times. It can be shown that (see e.g. \cite{gohberg})
\begin{equation}\label{eqvDmaps:probOperator}
\left\{P_{op}^tf\right\}(i,A) = \sum_{j=1}^{N}\int_G P^t((i,A),(j,B))f_j(B)d\eta(B), \quad f\in \mathcal{H},
\end{equation}
where we define for all~$A,B\in G$
\begin{align}
	P^1((i,A),(j,B)) &\triangleq P((i,A),(j,B)),\nonumber \\
	P^t((i,A),(j,B)) &\triangleq \sum_{k=1}^{N}\int_G P^{t-1}((i,A),(k,C))P^1((k,C),(j,B))d\eta(C), \quad t=2,3,\ldots\label{eqvDmaps:PtKernelDef}.	
\end{align}
By~\eqref{eqvDmaps:PtKernelDef} and~\eqref{eqvDmaps:probIdentity}, for any~$(i,A)\in[N]\times G$ we get by induction over $t$ that
\begin{align}\label{eqvDmaps:PtIsAProbDensity}
	\left\{P_{op}^t\mathbbm{1}_{\mathcal{H}}\right\}(i,A) = 1.
\end{align}
Furthermore, by~\eqref{GinvDef:Wdef}, \eqref{eqvDmaps:PopExplicitDef} and~\eqref{eqvDmaps:PtKernelDef}, we get that $P^t((i,A),(j,B))>0$ for all $(i,A),(j,B)\in [N]\times G$ and~$t\in\mathbb{N}$. Thus, we conclude that~$P_{op}^t$ is a probability transition operator for all~$t\in \mathbb{N}$, with a probability density kernel function given by~\eqref{eqvDmaps:PtKernelDef}. 

For a fixed point~$(i,A)\in [N]\times G$, the function~$P^t((i,A),(j,B))>0$ induces a probability distribution on~$[N]\times G$, with its density given by
\begin{equation}\label{eqvDmaps:PiAtcdotcdotDef}
	P^t_{i,A}(k,C)\triangleq  P^t((i,A),(k,C)), \quad (k,C)\in[N]\times G. 
\end{equation}
Analogously to~\eqref{intro:diffDist}, we now define a diffusion distance on~$[N]\times G$, as follows. 
\begin{definition}
	For all~$t=0,1,2,\ldots$, we define the equivariant diffusion distance between each pair of points~$(i,A),(j,B)\in [N]\times G$ as
	\begin{align}
		D_{p,t}((i,A),(j,B)) &= \norm{P^t_{i,A}-P^t_{j,B}}_{\mathcal{H},d\eta/D}\label{eqvDmaps:eqvDist} \\
		&\triangleq\left( \sum_{k=1}^{N}\int_G\left(P_{i,A}^t(k,C)-P_{j,B}^t(k,C)\right)^2 \frac{d\eta(C)}{D_{kk}}\right)^{\frac{1}{2}}\label{eqvDmaps:eqvDistExpanded}.
	\end{align}
\end{definition}
Similarly to~\eqref{intro:diffDist}, the diffusion distance~\eqref{eqvDmaps:eqvDist} can be computed by first embedding the points in~$[N]\times G$ into a Euclidean space, and then computing the Euclidean distance between the embedded points. The required embedding is defined as follows. 
\begin{definition}
	For all~$t=0,1,2,\ldots$, we define the equivariant embedding of~$[N]\times G$ by
	\begin{equation}\label{eqvDmaps:eqvEmbedding}
		\Phi^{(p)}_t(i,A) = \left(\lambda_{n,\ell}^t\Phi^{(p)}_{\ell,m,n}(i,A)\right)_{m=1,n=1,\ell\in \I_G}^{\ell,N},\quad (i,A)\in [N]\times G,
	\end{equation}
	where~$\lambda_{n,\ell}$ and~$\Phi^{(p)}_{\ell,m,n}$ are the eigenvalues and eigenfunctions, respectively, of~$P_{op}^t$ from~\eqref{eqvDmaps:probOperator}.
\end{definition}
We will show that the embedding~$\Phi^{(p)}_t$ is indeed equivariant (as defined in~\eqref{intro:equivEmbedd}), shortly. 
The following Theorem relates the embedding~\eqref{eqvDmaps:eqvEmbedding} with the distance~\eqref{eqvDmaps:eqvDist}.
\begin{theorem}\label{eqvDmaps:diffDistEigProp}
	For all~$t=0,1,2,\ldots$, and~$(i,A),(j,B)\in [N]\times G$, we have that
	\begin{equation}\label{eqvDmaps:prodDiffDist}
			D_{p,t}((i,A),(j,B))=\norm{\Phi^{(p)}_t(i,A)-\Phi^{(p)}_t(j,B)}_{\ell^2}.
	\end{equation}
\end{theorem}
\begin{proof}
	See Appendix~\ref{appendix:diffMapsEigPropPrf}.
\end{proof}

Note that for each $(i,A)\in[N]\times G$, the embedding \eqref{eqvDmaps:eqvEmbedding} is an infinite dimensional sequence. In practice, we truncate \eqref{eqvDmaps:eqvEmbedding} to obtain a finite-dimensional embedding, such that pairwise squared distances between the finite-dimensional embedded points closely approximate~\eqref{eqvDmaps:eqvDist}, as we now argue. 
The operator~$P_{op}^t$ defines a Markov chain over $[N]\times G$, with the $t$-step transition probability from a point $(i,A)$ to the points of a (Borel) measurable subset~$V\subseteq \left\{j\right\}\times G$ given by
\begin{equation}
	\int_{\left\{B\in G \; : \; \left(j,B\right)\in V\right\}} P_{i,A}^t(j,B)d\eta(B). 
\end{equation}
By definition, we have that $P^t_{i,A}>0$. By a result in \cite{lafonDissertation}, the largest eigenvalue of~$P_{op}^t$ is simple and equals~1, while the rest of the eigenvalues are strictly less than~1 (and non-negative due to Theorem \ref{secGGL:GGLdecomp}). 
This implies that, with the exception of the leading eigenvalue, all the eigenvalues~$\lambda_{n,\ell}^t$ (of~$P_{op}^t$ from~\eqref{eqvDmaps:probOperator}) decay to zero exponentially fast when~$t\rightarrow \infty$, and thus, so do the terms on the right hand side of~\eqref{eqvDmaps:eqvEmbedding}. 
Thus, for a fixed~$t\in\mathbb{N}$, we can set up a truncation rule for the sequence \eqref{eqvDmaps:eqvEmbedding} by retaining only the terms of~\eqref{eqvDmaps:eqvEmbedding} for which~$\lambda_{n,\ell}^t>\delta$, where $\delta >0$ is some parameter chosen so that the distances \eqref{eqvDmaps:eqvDist} are approximated up to a prescribed error, when replacing~$\Phi^{(p)}_t(i,A)$ and~$\Phi^{(p)}_t(j,A)$ in~\eqref{eqvDmaps:prodDiffDist} with their truncated versions.
We mention that the eigenfunction~$\Phi^{(p)}_{0,0,1}$ that corresponds to the leading eigenvalue~$\lambda_{0,1}=1$ can be shown to be constant over~$[N]\times G$, and thus, can be discarded from the embedding. 
Since~$\{\lambda_{n,\ell}^t\}$ are independent of the index~$m$ (that enumerates the rows of the IURs~$U^\ell$ in~\eqref{secGGL:eigenFuncs}), we define the truncated embedding
\begin{equation}\label{eqvDmaps:eqvEmbeddingFinite}
	\Phi^{(p)}_{\delta,t}(i,A) = \left(\lambda_{n,\ell}^t\Phi^{(p)}_{\ell,m,n}(i,A)\right)_{m=1,(n,\ell)\in \mathcal{I}_{\delta,t}}^N,
\end{equation}
where
\begin{equation}\label{eqvDmaps:truncIdx}
	\mathcal{I}_{\delta,t} = \left\{(n,\ell)\;:\; 0<\delta<\lambda_{n,\ell}^t<1\right\}. 
\end{equation}

By using Lemma~\ref{eqvDmaps:eqvEigenVecLemma}, we obtain the following result which shows that~\eqref{eqvDmaps:eqvEmbeddingFinite} induces a~$G$-equivariant embedding~$i\mapsto 	\Phi^{(p)}_{\delta,t}(i,I)$ of the data set~$X$, in the sense that the action of an element~$B\in G$ on a point~$x_i
\in X$ results in an action of an IUR of~$G$ on the embedding~$\Phi^{(p)}_{\delta,t}$. 
\newpage
\begin{proposition}\label{eqvDmaps:truncEmbedEqvProp}
	\text{}
	\begin{enumerate}
		\item For all~$(i,A)\in [N]\times G$ we have that
		\begin{equation}\label{eqvDmaps:eqvAction}
			\Phi^{(p)}_{\delta,t}(i,A) = \left(\lambda_{n,\ell}^t\sqrt{d_\ell}\cdot \overline{U^\ell\left(A\right)}\cdot e^i(v^{(p)}_{n,\ell})\right)_{(n,\ell)\in \mathcal{I}_{\delta,t}},
		\end{equation}
		where~$v^{(p)}_{n,\ell}$ and~$e^i(\cdot)$ were defined in~\eqref{eqvDmaps:vnlTilde2vnlp} and~\eqref{secGGL:parVecNotation}, receptively. 
		\item Let~$U(B)$ be the block-diagonal matrix with the IURs~$U^\ell(B)$ such that~$(n,\ell)\in\mathcal{I}_{\delta,t}$  on its diagonal, and suppose that~$x_j= B\cdot x_i$ for some~$B\in G$. Then, we have that
		\begin{equation}\label{eqvDmaps:embedEquivar}
			\Phi^{(p)}_{\delta,t}(j,I) = \overline{U(B)}\cdot \Phi^{(p)}_{\delta,t}(i,I).
		\end{equation}
	Furthermore, the function~$U(\cdot)$ is an IUR of~$G$. In particular, for each~$B\in G$ the matrix~$\overline{U(B)}$ is unitary. 
	\end{enumerate}
\end{proposition}
\begin{proof}
	See Appendix~\ref{appendix:truncEmbedEqvPropPrf}. 
\end{proof}

We comment that~Lemma~\ref{eqvDmaps:eqvEigenVecLemma} implies a similar result for the non-truncated embedding~\eqref{eqvDmaps:eqvEmbedding}. In this case,~$U(B)$ is the infinite-dimensional block-diagonal ``matrix" with all the IURs $U^\ell(B)$ of~$G$ for all~$(n,\ell)$ on its diagonal.

\section{$G$-invariant diffusion maps}\label{sec:invDmaps}
In the previous section, we saw that the diffusion distance~\eqref{eqvDmaps:eqvDist} between pairs of points~$x_i,x_j\in X$ gives rise to the equivariant embedding~\eqref{eqvDmaps:eqvEmbeddingFinite} of the data set~$X$. In various applications, the group action is viewed as a nuisance factor. Therefore, it is of interest to derive group-invariant embeddings that map all the points in the set~$G\cdot x_i = \left\{A\cdot x_i \; : \; A\in G\right\}$ (the orbit generated by the action of~$G$ on~$x_i$) into a single point in the embedding space. 

We can employ the distance~\eqref{eqvDmaps:eqvDist} and its relation~\eqref{eqvDmaps:prodDiffDist} to the equivariant embedding~\eqref{eqvDmaps:eqvEmbedding} to obtain a~$G$-invariant distance~$M_{p,t}: [N]\times [N]\rightarrow \mathbb{R}^+$ between points in~$X$ by defining 
\begin{equation}\label{eqvDmaps:minEqvDist}
	M_{p,t}^2(i,j) =\min_{A,B\in G} \norm{P_{i,A}^t-P_{j,B}^t}^2_{\mathcal{H},d\eta/D} = \min_{A,B\in G} \norm{\Phi^{(p)}_t(i,A)-\Phi^{(p)}_t(i,B)}^2_{\ell^2}, 
\end{equation}
as the following proposition establishes. 
\begin{proposition}\label{invDmaps:eqvInvDistanceProp}
	Suppose that~$x_k = Q\cdot x_i$ and~$x_r = R\cdot x_j$ for some~$Q,R\in G$. Then, we have that
	\begin{equation}
		M_{p,t}(k,r) = M_{p,t}(i,j).
	\end{equation}
Furthermore, we have that 
\begin{equation}\label{eqvDmaps:minEqvDistEffcient}
	M_{p,t}^2(i,j) = \min_{A\in G} \norm{\Phi^{(p)}_t(i,I)-\Phi^{(p)}_t(j,A)}^2_{\ell^2}.
\end{equation}	
\end{proposition}
\begin{proof}
	See Appendix~\ref{appendix:eqvInvDistancePropPrf}.
\end{proof}
The distance~$M_{p,t}$ can be approximated by solving the least-squares problem in~\eqref{eqvDmaps:minEqvDistEffcient},
with~$\Phi^{(p)}_t(i,A)$ replaced by the truncation~$\Phi^{(p)}_{\delta,t}(i,A)$ of~\eqref{eqvDmaps:eqvEmbeddingFinite}. By using the second assertion of Proposition~\ref{eqvDmaps:truncEmbedEqvProp}, and in particular, that for each~$B\in G$ the matrix~$U(B)$ in the assertion is a block-diagonal matrix with the IURs of~$G$ on its diagonal, the squared distances on the right hand side of~\eqref{eqvDmaps:minEqvDistEffcient} can be computed efficiently for low-dimensional groups~$G$ for which there exist FFT-type algorithms. We describe such a procedure in detail for the rotations group~$SO(3)$ in Part~I of this work~\cite{GGL}, and show that computing the squared distances~$\|\Phi^{(p)}_t(i,I)-\Phi^{(p)}_t(j,A)\|^2_{\ell^2}$ above over a~$K$ points discretization of~$SO(3)$ can be accomplished with~$O(K\cdot \log^2 K+\abs{\mathcal{I}_{\delta,t}})$ operations, where~$\abs{\mathcal{I}_{\delta,t}}$ (the cardinality of~\eqref{eqvDmaps:truncIdx}) is the dimension of the embedding~\eqref{eqvDmaps:eqvEmbeddingFinite}.

Unfortunately, even though the distance $M_{p,t}$ of~\eqref{eqvDmaps:minEqvDist} is group-invariant, in general, it is not guaranteed that there exists an embedding $\Upsilon$ such that $M_{p,t}(i,j) = \norm{\Upsilon(i)-\Upsilon(j)}$ for every pair~$i,j\in [N]$. Regardless, we can use the property~\eqref{eqvDmaps:minEqvDistEffcient} of~$M_{p,t}$ to~``align'' pairs of points~$x_i,x_j\in X$ that are close up to the action of~$G$. That is, we can compute the element~$A\in G$ that solves~\eqref{eqvDmaps:minEqvDistEffcient}, and then compute~$A\cdot x_j$. An important application (mentioned in Section~\ref{secIntro}) is to align pairs of images that are approximately rotations of each other.  

Next, we employ the eigenfunctions and eigenvalues of the~$G$-GL to define a group-invariant embedding of the data set~$X$, as follows. 
\begin{definition}
For all~$t=0,1,2,\ldots$, we define a~$G$-invariant embedding of the data set~$X$ by the function~$\Psi_t^{(p)}:[N]\rightarrow \ell^2$ defined as
\begin{equation}\label{invDmaps:probInvEmbeding}
	\Psi_t^{(p)}(i) = \left(  \sqrt{d_{\ell}}\cdot\lambda_{n,\ell}^{t}\lambda_{n',\ell}^{t}\cdot  \dprod{(e^i(v^{(p)}_{n,\ell}))}{e^i(v^{(p)}_{n',\ell})}\right)_{n,n'=1,\;\ell\in \I_G}^{N}.
\end{equation}
\end{definition}

By Lemma~\ref{eqvDmaps:eqvEigenVecLemma}, we get that~$\Psi_t^{(p)}(i)$ is~$G$-invariant, since if~$x_j = A\cdot x_i$ for some~$A\in G$, then we have that
\begin{equation}
	\dprod{e^j(v^{(p)}_{n,\ell})}{e^j(v^{(p)}_{n',\ell})} = \dprod{\overline{U^\ell(A)}e^i(v^{(p)}_{n,\ell})}{\overline{U^\ell(A)}e^i(v^{(p)}_{n',\ell})}=\dprod{e^i(v^{(p)}_{n,\ell})}{e^i(v^{(p)}_{n',\ell})},
\end{equation}
for all $\ell\in \I_G$ and $n\in[N]$. 
We can now define a~$G$-invariant diffusion distance between the points in~$X$.
\begin{definition}\label{invDmaps:invEmbdDef}
	For all~$t=0,1,2,\ldots$, we define the~$G$-invariant diffusion distance over~$X$ by
	\begin{equation}\label{invDmaps:probInvKernelMap}
		E_{p,t}(i,j) = \norm{\Psi_t^{(p)}(i)-\Psi_t^{(p)}(j)}_{\ell^2}, \quad i,j\in[N].
	\end{equation}
\end{definition}
\begin{remark}\label{invDmaps:truncationRemark}
	Similarly to what we did in the previous section for the equivariant embedding~$\Phi_t^{(p)}$ defined in~\eqref{eqvDmaps:eqvEmbedding}, we can truncate the embedding~$\Psi_t^{(p)}$ to obtain a finite-dimensional~$G$-invariant embedding by thresholding the sequence $\left\{\lambda_{n,\ell}^{t}\lambda_{n',\ell}^{t}\right\}$ using a rule analogous to~\eqref{eqvDmaps:truncIdx}. 
\end{remark}

In Section~\ref{sec:GEqvDmaps}, we defined the diffusion distance~$D_{p,t}(i,j)$ in~\eqref{eqvDmaps:eqvDist} as the distance between the probability densities of random walks over~$[N]\times G$ that depart from the points~$x_i$ and~$x_j$. We then showed that both the distance~$D_{p,t}(i,j)$ and the induced embedding~$\Phi^{(p)}_t(i,A)$ in~\eqref{eqvDmaps:eqvEmbedding} are~$G$-equivariant. 
We now show that the~$G$-invariant distance~$E_{p,t}$ defined above in~\eqref{invDmaps:probInvKernelMap} can be expressed as a distance between certain probability densities related to random walks over~$[N]\times G$. 

To that end, we first illustrate that the embedding~\eqref{eqvDmaps:eqvEmbedding} is~$G$-equivariant due the fact that for each~$i\in[N]$ the correspondence 
\begin{equation}\label{eqvDmaps:xi2densityMap}
	i\mapsto P^t_{i,I}
\end{equation}
between the point~$x_i$ (which we identify with~$(i,I)\in [N]\times G$) and the probability density $P^t_{i,I}$ (defined in~\eqref{eqvDmaps:PiAtcdotcdotDef}) of the random walk on~$[N]\times G$ that departs from~$(i,I)$ is by itself~$G$-equivariant. Indeed, consider the action `$\circ$' of~$G$ on functions in~$\Hspace$, and in particular  on~$P_{i,I}^t\in \Hspace$, defined by
\begin{equation}\label{eqvDmaps:fTranslationDef}
	\left\{Q\circ P^t_{i,I}\right\}(k,C) = P^t_{i,I}(k,Q^*C), \quad (k,C)\in[N]\times G.
\end{equation} 
For any fixed~$k_0\in [N]$, the function~${P^t_{i,I}(k_0,Q^*C)}\in L^2(G)$ is known as the left-translation of~$P^t_{i,I}(k_0,C)$ by~$Q$ (see~\cite{lieGroupsHall}), and thus by extension, we refer to~$P^t_{i,I}(k,Q^*C)\in L^2([N]\times G)$ in~\eqref{eqvDmaps:fTranslationDef} as the left-translation of~$ P^t_{i,I}(k,C)$ by~$Q$. 
The following proposition show that the correspondence~\eqref{eqvDmaps:xi2densityMap} is~$G$-equivariant with respect to left-translations.
\begin{proposition}\label{eqvDmaps:eqvActLemma}
	Suppose that $x_j=Q\cdot x_i$ for some~$Q\in G$. Then, we have that
	\begin{equation}\label{eqvDmaps:probDensityRelation}
			P^t_{j,I}(k,C) =\left\{Q\circ P^t_{i,I}\right\}(k,C), \quad (k,C)\in [N]\times G.
	\end{equation}
\end{proposition}
\begin{proof}
	See Appendix~\ref{eqvDmaps:eqvActLemmaPrf}.
\end{proof}
In words, equation~\eqref{eqvDmaps:probDensityRelation} implies that the density~$P^t_{j,I}$ is the left translation by~$Q$ of~$P^t_{i,I}$. This shows that a translation of~$x_i\in X$ by~$Q\in G$ results in a translation of~$P^t_{j,I}\in \Hspace$ by~$Q$.

We now define a~$G$-invariant correspondence that maps each $x_i\in X$ to a certain probability density related to random walks over~$[N]\times G$, and show that the~$G$-invariant distance~$E_{p,t}(i,j)$ of~\eqref{invDmaps:probInvKernelMap} is the distance between the densities to which~$x_i$ and~$x_j$ are mapped. 
We begin with the following definition. 
\begin{definition}
	Given a pair of functions~$f,g\in L^2(G)$, their cross-correlation function~$f\star g \in L^2(G)$ is defined as
	\begin{equation}\label{invDmaps:crossCorrDef}
	(f\star g)(C) = \int_G f(A)\cdot g(A C)d\eta(A). 
\end{equation}
\end{definition}
Now, consider the Hilbert space~$L^2\left([N]^2\times G\right)$ of functions of the form~$h(i,j,A)$ such that for each fixed $i_0,j_0\in [N]$ we have that $h(i_0,j_0,A)\in L^2(G)$.
Given a pair of functions~$f,g\in\Hspace$, we denote by `$\gconv$' the operation defined as
\begin{equation}\label{invDmaps:starGdef}
	(f\gconv g)(k,r,R) \triangleq \left\{f(k,\cdot)\star g(r,\cdot)\right\}(R), \quad (k,r,R)\in [N]^2\times G.
\end{equation}
That is, the function $f\gconv g\in L^2([N]^2\times G)$ is the cross-correlation  over~$G$ of~$f(k,\cdot)$ and~$g(r,\cdot)$, for each~$(k,r)\in [N]^2$. 
Now, for each~$x_i\in X$, consider the correspondence 
\begin{equation}\label{invDmaps:invProbKernel}
	i\mapsto P_{i,I}^t\gconv P_{i,I}^t,
\end{equation}
where by~\eqref{invDmaps:crossCorrDef}, we have that
\begin{align}
	\left\{P_{i,I}^t\gconv P_{i,I}^t\right\}(k,r,R) &= \left\{P_{i,I}^t(k,\cdot)\star P_{i,I}^t(r,\cdot)\right\}(R)\label{invDmaps:coupledRandomWalkDensity} \\
	&=\int_G P_{i,I}^t(k,C)\cdot P_{i,I}^t(r,CR)d\eta(C).\label{invDmaps:coupledRandomWalkDensity1}
\end{align}\\
The following proposition asserts that the correspondence~\eqref{invDmaps:invProbKernel} is indeed~$G$-invariant. 
\begin{proposition}\label{invDmaps:invMapInvarianceProp}
	Supppose that $x_j=Q\cdot x_i$ for some~$Q\in G$. Then, we have that 
	\begin{equation}
		P_{i,I}^t\gconv P_{i,I}^t = P_{j,I}^t \gconv P_{j,I}^t.
	\end{equation}
\end{proposition}
\begin{proof}
	See Appendix~\ref{appendix:invMapInvariancePropPrf}. 
\end{proof}

\begin{remark}
	Proposition~\ref{invDmaps:invMapInvarianceProp} implies that the correspondence~\eqref{invDmaps:invProbKernel} is~$G$-invariant. 
	We can also consider mapping each point~$i$ by
	\begin{equation}\label{invDmaps:invKernel2}
		i\mapsto P_{i,I}^{t_1} \gconv P_{i,I}^{t_2},
	\end{equation}
	for arbitrary pairs of times $t_1$ and~$t_2$. Repeating the steps of the proof of Proposition~\ref{invDmaps:invMapInvarianceProp}, we get that the correspondence~\eqref{invDmaps:invKernel2} is also~$G$-invariant for any pair of positive integer times~$t_1$ and~$t_2$. The latter suggests that we can construct embeddings of~$X$ by combining distinct diffusion time pairs.
\end{remark}

The following theorem asserts that the function on the left hand side of~\eqref{invDmaps:coupledRandomWalkDensity} is a probability density over $[N]^2\times G$, and that the distance~$E_{p,t}(i,j)$ in \eqref{invDmaps:probInvKernelMap} equals the distance between the probability densities corresponding to~$i\in [N]$ and~$j\in[N]$ via~\eqref{invDmaps:invProbKernel}. 
\begin{theorem}\label{invDaps:distEquivalenceThrm}
	$\text{}$
	\begin{enumerate}
		\item For all $i\in [N]$ and~$t\in \mathbb{N}$, we have that~$P_{i,I}^t\gconv P_{i,I}^t\geq 0$, and furthermore, that
		\begin{equation}\label{invDmaps:jointDensEq2}
			\sum_{k,r=1}^{N} \int_G \left\{P_{i,I}^t\gconv P_{i,I}^t\right\}(k,r,R)d\eta(R) =1.
		\end{equation}
	In other words, the function~$P_{i,I}^t\gconv P_{i,I}^t$ is a probability density over $[N]^2\times G$.
	   
	    \item The diffusion distance~$E_{p,t}$ in~\eqref{invDmaps:probInvKernelMap} is given by
	    \begin{equation}\label{invDmaps:ProbDiffDist}
	    	E_{p,t}(i,j) = \norm{ P_{i,I}^t\gconv P_{i,I}^t-P_{j,I}^t\gconv P_{j,I}^t}_{L^2\left([N]^2\times G,d\eta/D\otimes D\right)},
	    \end{equation}
	    where~$L^2\left([N]^2\times G,d\eta/D\otimes D\right)$ is the Hilbert space of functions~$f(i,j,A)$ such that for each fixed $i_0,j_0\in [N]$ we have that $f(i_0,j_0,A)\in L^2(G)$, and with inner product given by
	    \begin{equation}\label{invDmaps:weightedHspace}
	    	\dprod{f}{g}_{L^2\left([N]^2\times G,d\eta/D\otimes D\right)}\triangleq\sum_{i,j=1}^N\int_G f(i,j,C)\cdot \overline{g(i,j,C)} \frac{d\eta(C)}{D_{ii}\cdot D_{jj}}.
	    \end{equation}
	\end{enumerate}
\end{theorem}
\begin{proof}
	See Appendix~\ref{appSec4:invDiffDistPrf}.
\end{proof}

We now relate the probability density~$P_{i,I}^t\gconv P_{i,I}^t$ to random walks over the data. Assuming that~$A\cdot x_i\neq x_i$ for (almost) all $A\in G$ and all $x_i\in X$ (as in Theorem~\ref{secGGL:GGLdecomp}), we get that there exists a $1-1$ correspondence between each orbit~$G\cdot x_i$ and~$G$, given by the map $A\mapsto A\cdot x_i$. Thus, we may think of the domain $[N]\times G$ as a set of coordinates over~$G\cdot X$, where the coordinate of a point~$A\cdot x_i\in G\cdot X$ is given by~$(i,A)\in [N]\times G$. 
In terms of the latter, we say that a random walk over $G\cdot X$ can either move along the ``$[N]$-direction'' or the ``$G$-direction'' (or both) in the domain~$[N]\times G$.
Now, let $X_{1,t}$ and~$X_{2,t}$ be the positions at time~$t$ of a pair of random walks over~$[N]\times G$ that depart together from~$(i,I)$ at time~$t=0$, both with probability density given by~$P^t_{i,I}$. Furthermore, let~$N_{1,t}$ and~$N_{2,t}$ denote the random variables that take the values~$k$ and~$r$, respectively, whenever $X_{1,t}=(k,A)$ and~$X_{2,t}=(r,B)$ for some~$A,B\in G$. In other words, $N_{1,t}$ and~$N_{2,t}$ are the coordinates of~$X_{1,t}$ and~$X_{2,t}$ in the~$[N]$-direction. Similarly, let $G_{1,t}$ and~$G_{2,t}$ denote the random variables that assume the values~$A$ and~$B$, respectively, whenever~$X_{1,t}=(k,A)$ and~$X_{2,t}=(r,B)$ for some~$k,r\in [N]$, and so, $G_{1,t}$ and~$G_{2,t}$ are the coordinates of the random walks in the~$G$-direction. Lastly, we define the ``displacement'' of~$X_{1,t}$ and~$X_{2,t}$ in the $G$-direction'' as the random variable~$R_t=G_{1,t}^*\cdot G_{2,t}$, which is the relative position of the random walks in the~$G$-direction. We now have the following proposition.
\begin{proposition}\label{invDmaps:genDiplacementLemma}
	Let $X_{1,t}$ and~$X_{2,t}$ be the positions at time~$t$ of a pair of independent random walks over~$[N]\times G$ that depart from~$(i,I)$, both with probability density given by~$P^t_{i,I}$.
	For any $k,r\in [N]$ and~$H\subseteq G$, consider the event
	\begin{equation}\label{invDmaps:displacementEvent}
		H_{kr} = \left\{N_{1,t}=k,\; N_{2,t}=r,\; R_t \in H \right\},
	\end{equation}
	whereby at time~$t$ the random walks~$X_{1,t}$ and~$X_{2,t}$ had reached the orbits~$G\cdot x_k$ and~$G\cdot x_r$, respectively, and their displacement~$R_t$ in the~$G$-direction~is in~$H$. Then, we have that
	\begin{equation}\label{invDmaps:displaceProbLaw}
		\prob{H_{kr}} = \int_H\left\{P_{i,I}^t \gconv P_{i,I}^t\right\}(k,r,R)d\eta(R), 
	\end{equation}
	where~`$\star_G$' is defined in~\eqref{invDmaps:starGdef}.
\end{proposition}
\begin{proof}
	See Appendix~\ref{invDmaps:genDiplacementLemmaPrf}
\end{proof}
In light of Proposition~\ref{invDmaps:genDiplacementLemma}, we conclude that Theorem~\ref{invDaps:distEquivalenceThrm} asserts that~$E_{p,t}$ from~\eqref{invDmaps:probInvKernelMap} is the distance between the densities of the displacements of two pairs of random walks over~$[N]\times G$. 
To further clarify the implications of Theorem~\ref{invDaps:distEquivalenceThrm} and Proposition~\ref{invDmaps:genDiplacementLemma}, consider the special case where~$\man = G\cdot x_k$ for some arbitrary fixed point~$x_k\in X$. That is, the entire manifold~$\man$ is the single orbit generated by the action of~$G$ on the single point~$x_k$. 
Since for any~$x_i\in \man$ there exists an element~$A\in G$ such that $x_i = A\cdot x_k$, we may think of the point~$x_k$ as the origin of the space~$\man =G\cdot x_k$, analogously to the origin in a Euclidean space, and consider the elements of~$G$ as coordinates on~$\man$. If we express every point that the random walks reach in~$[N]\times G$ in coordinates with respect to the origin~$x_k$, then every point takes the form~$(k,A)$ for some~$A\in G$ (and a fixed $k\in [N]$). Consequently, in these ``coordinates'', the random variables~$N_{1,t}$ and~$N_{2,t}$ above are both constant, taking the value~$k$. This implies that in this case~\eqref{invDmaps:displaceProbLaw} is a probability distribution of the displacement~$R_t$ (the relative position) of a pair of random walks on~$\man$ that depart simultaneously from some point~$x_i$, where the probability density of~$R_t$ is given, due to~Theorem~\ref{invDaps:distEquivalenceThrm} and Proposition~\ref{invDmaps:genDiplacementLemma}, by~$P_{i,I}^t\gconv P_{i,I}^t$ of~\eqref{invDmaps:coupledRandomWalkDensity}.  
While the distributions of the positions of a pair of random walks certainly depend on their departure point, the distribution of their relative position does not. Proposition~\ref{invDmaps:invMapInvarianceProp} formalizes the latter statement, by asserting that the density~$P_{i,I}^t\gconv P_{i,I}^t$ is~$G$-invariant, which is exactly the motivation behind the~$G$-invariant correspondence~\eqref{invDmaps:invProbKernel}.
In the general case, where~$\man$ consists of multiple orbits,  Proposition~\ref{invDmaps:invMapInvarianceProp} asserts that the density~$P_{i,I}^t\gconv P_{i,I}^t$ is unchanged by varying the departure point in the~$G$-direction, that is, when~$x_i$ is replaced by~$Q\cdot x_i$ for some~$Q\in G$. 

\section{Numerical experiments}\label{sec:CryoSimulations}
\subsection{Basic examples}
In this section, we corroborate the key properties of the distances derived in the previous sections using some simulated~$G$-invariant data sets. In particular, we demonstrate the~$G$-equivariance of~$D_{p,t}$~in~\eqref{eqvDmaps:eqvDist}, and~$G$-invariance of~the distance $E_{p,t}$ in~\eqref{invDmaps:probInvKernelMap}. 

In the following simulation, we consider the action of the group~$G=SO(2)$ of~2D rotations on the 2D torus~$\mathbb{T}^2\subset \mathbb{R}^3$, given by
\begin{equation}\label{numericsSec:rotActOnT2}
	\begin{pmatrix}
		\cos\varphi &-\sin\varphi&\\
		\sin\varphi & \phantom{tt}\cos\varphi&\\
		&&1
	\end{pmatrix} \begin{pmatrix}
	x\\y\\z
\end{pmatrix}, \quad \left(x,y,z\right)^T\in \mathbb{T}^2, 
\end{equation}
where~$\mathbb{T}^2$ is given by
\begin{align}
	&x(\theta,\varphi) = (R+r\cos\varphi)\cos\phi, \\
	&y(\theta,\varphi) = (R+r\cos\varphi)\sin\phi, \\
	&z(\theta,\varphi) = r\sin\theta,	
\end{align} 
with~$\theta,\varphi\in [0,2\pi)$, and $R=2,r=1$. 

To demonstrate that the distance $D_{p,t}$ in~\eqref{eqvDmaps:eqvDist} is~$SO(2)$-equivariant, we first sampled a data set~$X$ of~$N=10000$ points from a uniform distribution over~$\mathbb{T}^2$. Next, we randomly chose a point~$x_i\in X$ and added to~$X$ the point $x_{N+1}$ that results from a rotation of~$x_i$ by $180^\circ$ about the $z$ axis. We then used the data set $\tilde{X} = X\cup \left\{x_{N+1}\right\}$ to compute and factor the matrices~$W^{(\ell)}$ given by~\eqref{secGGL:blockFourierMat}. We used the resulting eigenvectors and eigenvalues to compute the~$SO(2)$-equivariant embedding~$\Phi_t^{(p)}$ of~$\eqref{eqvDmaps:eqvEmbedding}$, where we chose the time parameter~$t=3$, and truncated~\eqref{eqvDmaps:eqvEmbedding} by applying the rule~\eqref{eqvDmaps:truncIdx} with~$\delta=0.1$. 
Finally, we computed the distances~$D_{p,t}$ of each point~$x\in X$ from~$x_i$ and from~$x_{N+1}$. 
Figure~\ref{fig:eqvTori} depicts the diffusion distances of the points in~$X$ from~$x_i$ and~$x_{N+1}$, respectively, superimposed on the torus~$\mathbb{T}^2\subset \mathbb{R}^3$ as a heat map. In Figure~\ref{fig:toriDista} we plot the value~$D_{p,t}(i,j)$ at~$x_j\in \mathbb{T}^2$ for each~$j\in [N]$, and similarly, in Figure~\ref{fig:toriDistb} we plot the value~$D_{p,t}(N+1,j)$ at~$x_j$. 
We observe that for a fixed~$i\in [N]$, the distance~$D_{p,t}(i,\cdot)$ is localized in a neighborhood of~$x_i$, and furthermore, that the distances~$D_{p,t}(N+1,\cdot)$ are a rotation of~$D_{p,t}(i,\cdot)$ by~$180^\circ$, which demonstrates that~$D_{p,t}$ is~$SO(2)$-equivariant. 

\begin{figure}
	\centering
	\subfloat[]  	
	{
		\includegraphics[width=0.45\textwidth]{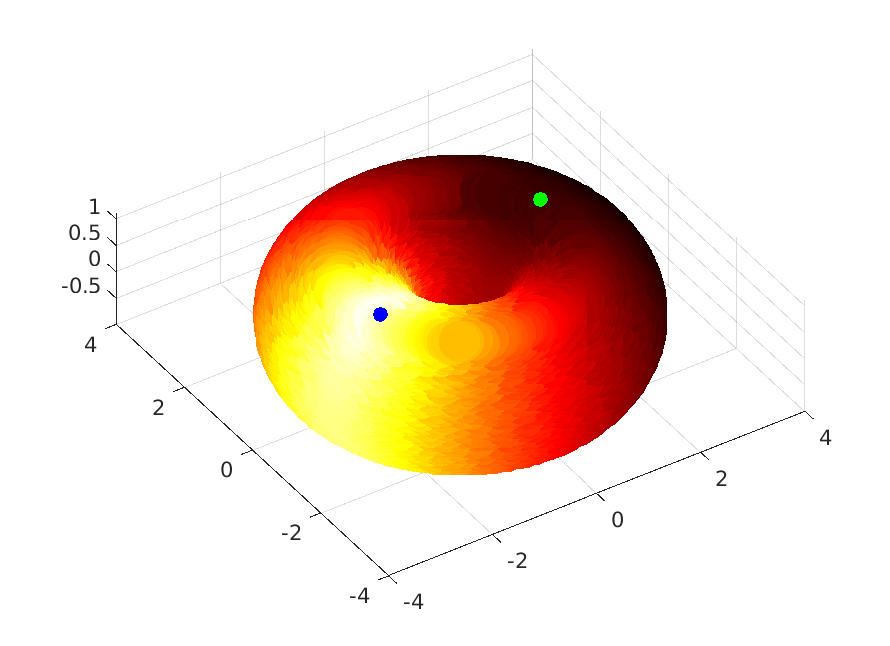}
		\label{fig:toriDista}
	}
	\subfloat[]    
	{ 
		\includegraphics[width=0.45\textwidth]{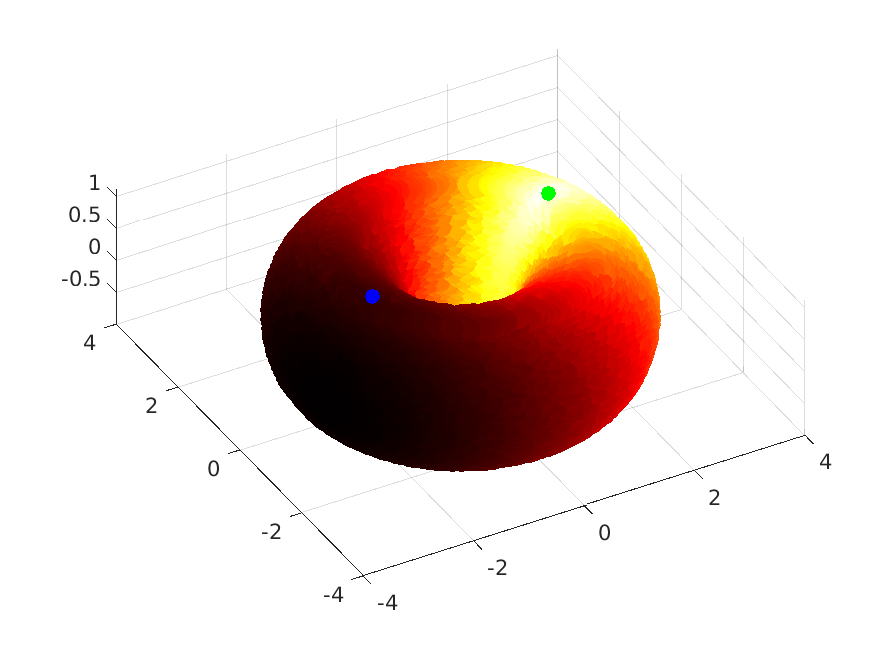}
		\label{fig:toriDistb}
	}
	\caption[]
	{Figure~(a) shows the $SO(2)$-equivariant diffusion distances~\eqref{eqvDmaps:eqvDist} of the points in the simulated data set~$X$ from~$x_i\in X$ (marked by a blue dot), depicted as a heat map superimposed on the torus~$\mathbb{T}^2\subset \mathbb{R}^3$. Figure~(b) shows the distances~\eqref{eqvDmaps:eqvDist} of the points in~$X$ from~$x_{N+1}$ (marked by the green dot). The images indicate that rotating~$x_i$ by~$180^\circ$ rotates the heat distribution by the same angle.} \label{fig:eqvTori}
\end{figure}

Next, we used the eigenvectors and eigenvalues of~$P_{op}^t$ for~$t=3$ to compute the~$SO(2)$-invariant embedding~$\Psi_t^{(p)}$ of~\eqref{invDmaps:probInvEmbeding}, truncated 
as suggested in Remark~\ref{invDmaps:truncationRemark} in Section~\ref{sec:invDmaps}, with~$\delta = 0.1$. Figure~\ref{fig:invDista} depicts the values of the~$SO(2)$-invariant diffusion distances~$E_{p,t}$ of~\eqref{invDmaps:probInvKernelMap} from~$x_i$ to the points in~$X$, superimposed on the torus~$\mathbb{T}^2\subset \mathbb{R}$ as a heat map, where as before, each value~$E_{p,t}(i,j)$ is plotted at~$x_j\in X$. It is seen that the heat distribution is constant along the direction of the action of $SO(2)$ in~\eqref{numericsSec:rotActOnT2}, that is, over the horizontal circles in~$\mathbb{T}^2$, which demonstrates that $E_{p,t}$ is invariant under the action~\eqref{numericsSec:rotActOnT2}. 

We also repeated the previously described simulations with points sampled from the two-dimensional unit sphere~$S^2\subset \mathbb{R}^3$ coupled with the action of~$SO(2)$ by rotations about the~$z$-axis. The results are depicted in Figures~\ref{fig:invDistb} and~\ref{fig:eqvDistSphere} in the same manner as for the case of~$\mathbb{T}^2$, implying the same qualitative picture as for~$\mathbb{T}^2$. 

\begin{figure}
	\centering
	\subfloat[] 	
	{
		\includegraphics[width=0.5\textwidth]{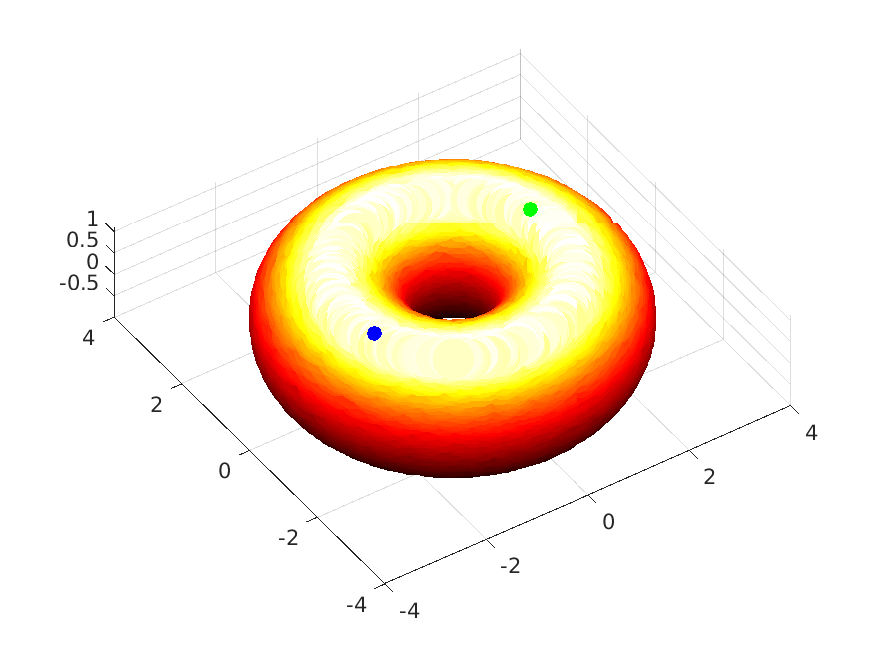}
		\label{fig:invDista}
	}
	\subfloat[]    
	{ 
		\includegraphics[width=0.45\textwidth]{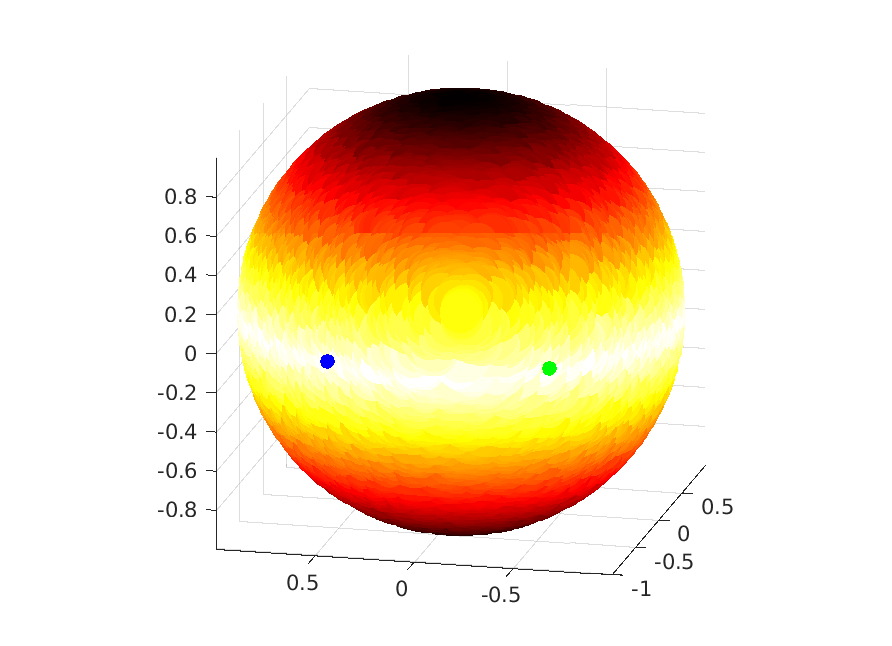}
		\label{fig:invDistb}
	}
	\caption[]
	{Figure (a) displays the~$SO(2)$-invariant diffusion distance of the points in the simulated data set~$X$ from~$x_i\in X$ (marked by the blue point). Figure (b) shows the same simulation repeated with points sampled from the 2-sphere. In both cases, the distances are constant over the orbits induced by the action of~$SO(2)$ by rotations about the~$z$-axis, namely, the horizontal circles.} \label{fig:invDistHeatmap}
\end{figure}

\begin{figure}
	\centering
	\subfloat[]  	
	{
		\includegraphics[width=0.43\textwidth]{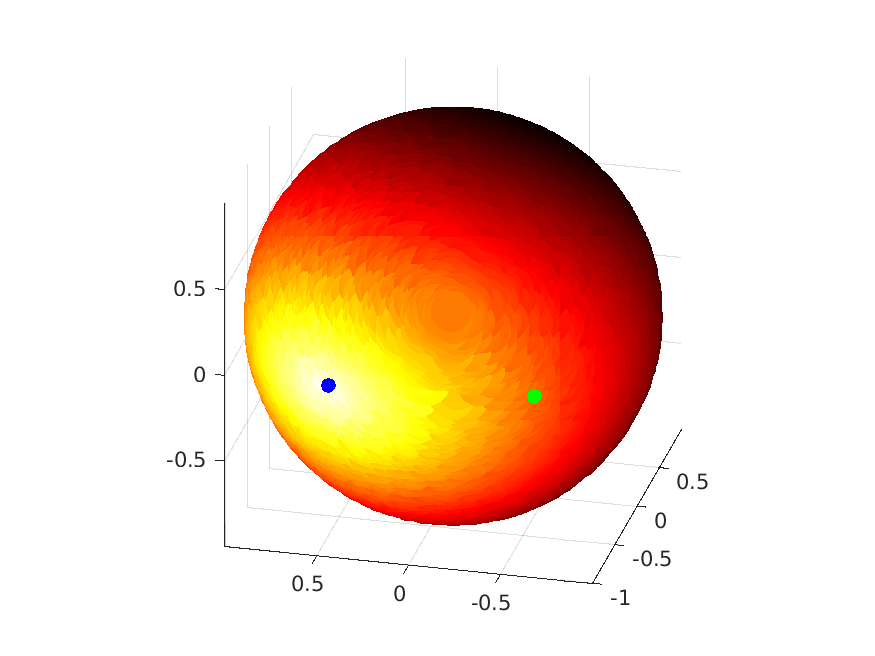}
	}\label{fig:sphereDista}
	\subfloat[]    
	{ 
		\includegraphics[width=0.45\textwidth]{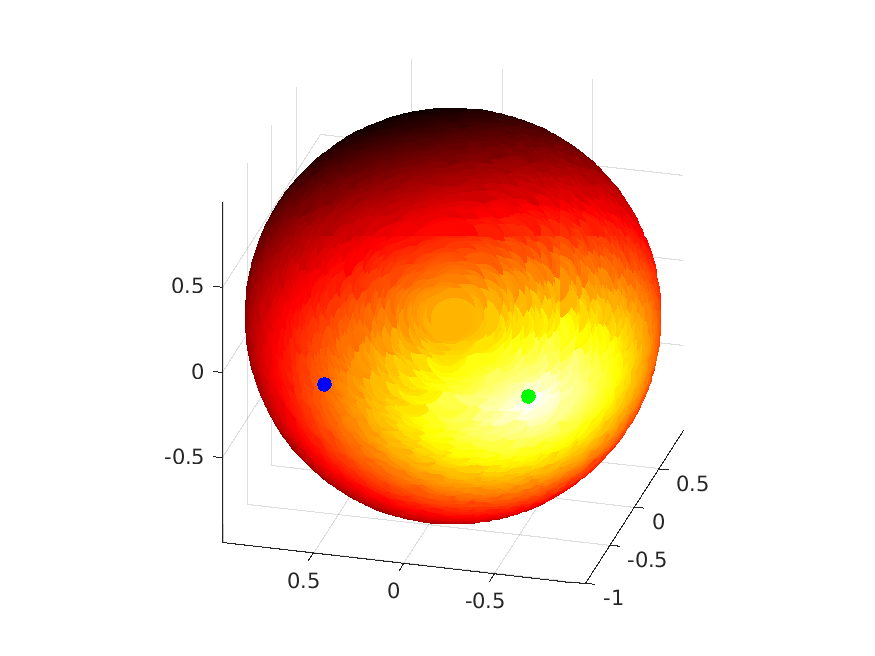}
	}\label{fig:sphereDistb}
	\caption[]
	{Heat distributions resulting from repeating the simulation depicted in Figure~\ref{fig:eqvTori} above with data sampled from the~$2$-sphere. In this case, the green point is a rotation of~$x_i$ (marked by the blue point) by~$60^\circ$ about the~$z$-axis. The distribution of the diffusion distances of the points in the simulated data set from the rotation of~$x_i$ by $60^\circ$ in Figure~(b) is obtained by rotating the distribution of the distances from~$x_i$ in Figure~(a) by~$60^\circ$, showing a similar qualitative picture to that of the simulation with the torus.} \label{fig:eqvDistSphere}
\end{figure}

\subsection{Random computerized tomography with random angles and shifts}\label{sec:randTomography}
In this section, we consider the problem of random computerized tomography~\cite{basuBresFeasability}, where the goal is to reconstruct a~2D image from its~1D Radon transform projections taken at random unknown orientations and shifts. We  demonstrate the utility of our framework in resolving this problem.

Let~$I(x,y)\in L^2(\mathbb{R}^2)$ be an image, finitely supported in~$\mathbb{R}^2$. The Radon transform~$P_\varphi(t)$ of~$I$ in the direction~$\varphi \in [0,2\pi)$ is defined as the line integral of~$I$ along the line~$L$, which is inclined at an angle~$\phi$ and is at distance~$t$ from the origin. That is,  
\begin{equation}\label{numerSec:radonProj}
	P_\varphi(t) = \int_L I(x,y)ds = \int_{-\infty}^\infty\int_{-\infty}^\infty I(x,y)\delta(x\cos\varphi+y\sin\varphi-t)dxdy, 
\end{equation}
where~$\delta(x)$ is the Dirac delta.
Given a finite set of samples~$X = \left\{x_1,\ldots,x_N\right\}$ of the Radon transform, where~$x_i=(P_{\varphi_i}(t_1),\ldots,P_{\varphi_i}(t_m))$, each generated at known angles~$\varphi_1,\ldots,\varphi_N$ and fixed equally spaced distances~$t_1,\ldots,t_m$ from the origin, tomographic reconstruction algorithms~\cite{natterFrank} estimate the image~$I$ from these samples. In random tomography, the goal is to reconstruct an image~$I$ from the set~$X$ above in the case where the angles~$\varphi_1,\ldots,\varphi_N$ are unknown but were sampled uniformly at random from~$[0,2\pi)$ (and as before~$t_1,\ldots,t_m$ are known).

In~\cite{glRandTomography}, it was shown that this problem can be reduced to tomographic reconstruction with known angles outlined above, by ordering the projections according to the angles at which they were generated, and reconstructing the image by setting~$\varphi_i = 2\pi i/N$ for each~$i\in [N]$.
Specifically, it was shown that the angles~$\varphi_1,\ldots, \varphi_N$ can be sorted as follows. First, use~$X$ to construct the density-invariant graph Laplacian matrix~$\overline{L}=D^{-1} L D^{-1}$ \cite{diffMaps}, where~$L$ and~$D$ are defined in~$\eqref{intro:classicalGLDef}$. Then, compute the diffusion maps embedding using the two non-trivial leading eigenvectors~$\phi_2,\phi_3$ of~$\overline{L}$. With a proper choice of the bandwidth~$\epsilon$, the embedded projections reside on a circle, and moreover, they are ordered according to the projection angels~$\varphi_i$ (up to a rotation by an arbitrary angle). Thus, we can estimate the angles of the projections via
\begin{equation}\label{numerSec:ordAng}
	\tilde{\varphi}_i = \text{atan2}(\phi_2(i),\phi_3(i)),\quad i\in [N],
\end{equation}  
and use these angles to order the projections.
Then, the image is reconstructed by reordering the projections with respect to~$\tilde{\varphi}_i$, and setting for the reordered projections $\varphi_{i}$ to be $2\pi i/N$. The procedure above is outlined in Algorithm~\ref{numerSec:ordAlg}.
\begin{algorithm}
	\caption{Image reconstruction from random projections}\label{numerSec:ordAlg}
	\begin{algorithmic}[1]
		\Statex{\textbf{Input:}}
		Projection images~$x_i = \left(P_{\varphi_i}(t_1),\ldots,P_{\varphi_i}(t_m)\right)$, for~$i=1,\ldots, N$. 	
		\State{Compute~$D$ and~$L$ from~\eqref{intro:classicalGLDef}, and construct the density-invariant graph Laplacian~$\overline{L} = D^{-1}LD^{-1}$.}
		\State{Compute the two leading non-trivial eigenvectors~$\phi_2,\phi_3$ of~$\overline{L}$. }
		\State{Sort the projections~$x_i$ according to $\tilde{\varphi}_i = \text{atan2}(\phi_2(i),\phi_3(i))$.}
		\State{Reconstruct a 2D image using the sorted projections $x_i$, setting the input angles~$\varphi_i$ to be $2\pi i/N$.} 
	\end{algorithmic}
\end{algorithm}

\begin{remark}\label{numericSec:degreeOfFreedom}
	As explained in~\cite{glRandTomography}, the eigenvectors~$\phi_2,\phi_3$ of~$\overline{L}$ approximate a pair of eigenfunctions of the Laplace-Beltrami operator on the circle~$S^1$, that correspond to an eigenvalue with multiplicity~2. 
	Thus, the computation of~$\phi_2,\phi_3$ may result in any orthogonal combination of these eigenvectors (depending on the numerical procedure used for their computation). 
	This implies that the diffusion maps embedding formed by these eigenvectors is only unique up to an arbitrary rotation and possibly a reflection (orientation of the curve).
	This degree of freedom is manifested in the fact that the reconstruction of the image is only possible up to a rotation and a possible reflection, or by the same token, 	in that the order of the projection angles can be recovered only up to a cyclic permutation, and that it may be possibly flipped.  
\end{remark}

Here, we tackle the case where each projection~$x_i\in X$ is not only generated at a random unknown angle, but also may be independently shifted by a random unknown shift~$s_i$ (to the left or to the right)~\cite{basuBresUniqueness}. In terms of the model formulated above, this corresponds to the setting where the equally spaced distances~$t_1,\ldots,t_m$ may be shifted by a random number~$s_i$, sampled uniformly at random from an interval~$S=[-s_{\max},s_{\max}]$, where~$s_{\max}$ is the maximal shift of a projection. To conclude, our goal is to reconstruct a~2D image from the set~$X = \left\{(P_{\varphi_i}(t_1+s_i),\ldots,P_{\varphi_i}(t_m+s_i))\right\}_{i=1}^N$, where~$\varphi_i$ and~$s_i$ are sampled independently and uniformly at random from~$[0,2\pi)$ and~$S$, respectively.
In this setting, the data points in~$X$ reside in the~two-dimensional manifold comprised of all the projections and their shifts in~$S$. Since Algorithm~\ref{numerSec:ordAlg} requires the projections to reside on a curve, we cannot except it to directly resolve this problem. Indeed, below we demonstrate that in this case Algorithm~\ref{numerSec:ordAlg} fails to recover the order of the projections in~$X$, and so also fails to recover the underlying image.

We now derive a method in which we first unshift the projections in~$X$ such that they reside on a curve~$C$, after which we apply~Algorithm~\ref{numerSec:ordAlg} to the unshifted projections. Formally, our method finds a set of shifts~$\tilde{s}_1,\ldots,\tilde{s}_N$ such that shifting~$x_i = (P_{\varphi_i}(t_1+s_i),\ldots,P_{\varphi_i}(t_m+s_i))$ by~$\tilde{s}_i$ approximates the unshifted projection $(P_{\varphi_i}(t_1),\ldots,P_{\varphi_i}(t_m))$, that is,
\begin{equation}\label{numericSec:reAlignEq}
	(P_{\varphi_i}(t_1+s_i-\tilde{s}_i),\ldots,P_{\varphi_i}(t_m+s_i-\tilde{s}_i)) \approx (P_{\varphi_i}(t_1),\ldots,P_{\varphi_i}(t_m)).
\end{equation}
In a nutshell, the unshifting works as follows. First, we compute an embedding of~$X$ which is invariant to the shifts~$s_i$. In other words, any two projections~$x_i$ and~$x_j$ such that~$x_j$ is a shift of~$x_i$ are embedded into the same point. This enables us to detect which unshifted projections~$(P_{\varphi_i}(t_1),\ldots,P_{\varphi_i}(t_m))$ reside at neighboring points on~$C$. 
Then, we compute an embedding of~$X$ which is equivariant to the shifts~$s_i$, which we employ to find for each such pair of neighboring projections the relative shift~$s_{ij}\in S$ that best aligns~$x_j$ with~$x_i$. 
Finally, we show how to employ an algorithm derived in~\cite{Singer2009AngularSB} to extract the shifts~$\tilde{s}_1,\ldots,\tilde{s}_N$ from the relative shifts~$s_{ij}$. 

In order to implement our method, we would like to construct a shift-invariant graph Laplacian, and compute its associated invariant and equivariant embeddings, derived in the previous sections. 
Unfortunatelly, the set of shifts~$S$ is clearly not a group. Thus, our strategy is to first embed~$S$ into the circle~$S^1$ (a one-dimensional Lie group) via an invertible map~$\Theta$, such that each element~$\nu \in S^1$ acts on~$x_i \in X$ by a unique shift~$s_\nu = \Theta^{-1}(\varphi)$ (where we identify~$S^1$ with~$[0,2\pi)$). Then, we can construct the~$S^1$-GL by using~$X$ with the aforementioned action of~$S^1$, and the~$S^1$-invariant and equivariant embeddings computed from its eigenfunctions and eigenvalues induce the desired embeddings. We now construct such an embedding~$\Theta$. 

Recall that the eigendecomposition of the~$S^1$-GL is derived by using the eigenvectors and eigenvalues of the matrices~$W^{(\ell)}$ from~\eqref{secGGL:blockFourierMat}. The matrices~$W^{(\ell)}$ in our construction are formed by the coefficients of the Fourier series of the functions~$d_{ij}(\varphi) \triangleq \exp\{-\norm{x_i-s_\nu\circ x_j}^2/\epsilon\}$ for all~$i,j\in[N]$. To obtain an everywhere convergent series for each~$i,j\in [N]$, we need to make sure that the definition of the map~$\Theta$ above guarantees that each function~$d_{ij}$ is~$2\pi$-periodic. For that, we observe that the projections in~$X$ are assumed to have been generated from an image~$I$ that is finitely supported in~$\mathbb{R}^2$, and are thus also finitely-supported. Now, Suppose, that their support lies within the interval~$Q = [-q,q]$ for some~$q>0$. Then, for each~$i,j\in [N]$, and any shift~$s$ such that~$\abs{s}\geq \overline{s}$, where~$\overline{s}=\max\{2q,2s_{\max}\}$, we have that
\begin{equation}\label{numericSec:normFiniteSupp}
	\norm{x_i - s\circ x_j}^2 = \norm{x_i}^2-2\dprod{x_i}{s\circ x_j}+\norm{x_j}^2 = \norm{x_i}^2+\norm{x_j}^2.
\end{equation}
Thus, if we define the map~$\Theta$ above by
\begin{equation}\label{numerSec:shiftToAngMapDef}
	\Theta(s) = \pi+\frac{2\pi s}{2\overline{s}}, \quad s\in \overline{S},
\end{equation}
which maps~$\overline{S}=[-\overline{s},\overline{s}]$ to~$S^1$ by identifying the boundaries of the interval~$\overline{S}$, \eqref{numericSec:normFiniteSupp} implies that
\begin{equation}
	d_{ij}(2\pi)= \exp\{-\norm{x_i-\overline{s}\circ x_j}^2/\epsilon\} = \exp\{-\norm{x_i-(-\overline{s})\circ x_j}^2/\epsilon\}= d_{ij}(0),
\end{equation}
as desired. Furthermore, $\Theta$ is invertible, and due to the definition of~$\overline{s}$ above, we also have that~$S\subseteq \overline{S}$, which together imply that each shift within~$S$ is embedded by~$\Theta$ into a unique point in~$S^1$. Therefore, if we compute the~$S^1$-invariant embedding~\eqref{invDmaps:probInvEmbeding} associated with the~$S^1$-GL constructed as just described, any two projections~$x_i,x_j \in X$ such that~$x_i = s\circ x_j$ for some~$s\in S$ get embedded into the same point. Moreover, the associated~$S^1$-equivariant embedding~$\Phi^{(p)}_{\delta,t}$ in~\eqref{eqvDmaps:eqvEmbeddingFinite} induces a shift-equivariant embedding of~$X$, where if~$x_i = s\circ x_j$ for some~$s\in S$ then~$\Phi^{(p)}_{\delta,t}(j,\pi) = \overline{U(\Theta(s))}\cdot \Phi^{(p)}_{\delta,t}(i,\pi)$, where~$\pi=\Theta(0)$ is the group element in~$S^1$ that corresponds to the zero shift (see~\eqref{eqvDmaps:embedEquivar} in~Proposition~\ref{eqvDmaps:truncEmbedEqvProp}).
\begin{remark}
	Note, that we may have to pad the projections~$x_i\in X$ with zeros to implement the shifts~$s\circ x_i$ for all~$s\in \overline{S}$, whose magnitude may be larger than~$s_{\max}$ (the magnitude of the maximal shift in the data). In particular, we avoid cyclically shifting the data which produces vectors that cannot have been obtained by the Radon projection~\eqref{numerSec:radonProj}.
\end{remark}

Since the unshifted projections of a 2D image reside on a smooth and closed curve~$C$, the data set of projections~$\left\{s_\nu\circ x_i \; : \; x_i\in X, \nu\in[0,2\pi) \right\}$ resides on the compact two-dimensional manifold~$\man$ generated by letting the elements~$\varphi\in S^1$ act on each point~$x\in C$ by~$s_\nu\circ x$, as described above. 
We now describe in detail our method to reconstruct a~2D-image from its random shifted projections, which uses the $S^1$-invariant and~$S^1$-equivariant diffusion maps derived from the~$S^1$-GL constructed by viewing the projections in~$X$ as samples from~$\man$. The method consists of~$5$ steps. 

\begin{algorithm}
	\caption{$S^1$-GL based shift-invariant~$K$ nearest neighbors 
	}\label{numerSec:shiftInvNearestNeighbors}
	\begin{algorithmic}[1]
		\Statex{\textbf{Input:}}
		\begin{enumerate}
			\item Data set~$X$ of shifted projections. 
			\item Maximal IUR index~$\ell_{\max}$, an integer~$K>0$ of nearest neighbors to compute, and a diffusion time parameter~$t$. 
		\end{enumerate}		
		\State{Use~$X$ to construct the~$S^1$-GL as described in Theorem~\ref{secGGL:GGLdecomp}, with the action of~$S^1$ on~$X$ given by the shifts~$\Theta^{-1}(\varphi)$, $\varphi\in[0,2\pi)$ (see~\eqref{numerSec:shiftToAngMapDef})}. 
		\State{Compute the truncated~$S^1$-invariant diffusion maps~$\Psi_{\ell_{\max},t}^{(p)}(i)$ via~\eqref{invDmaps:probInvEmbeding} with~$\ell\leq \ell_{\max}$}, for all~$i\in [N]$. 
		\State{For each~$x_i\in X$ compute the set~$\mathcal{N}_i$ of the~$K$ projections~$x_j\in X$ with smallest distance~$\norm{\Psi_{\ell_{\max},t}^{(p)}(i)-\Psi_{\ell_{\max},t}^{(p)}(j)}$ }. 
	\end{algorithmic}
\end{algorithm}

In Step~1, we use the data set~$X$ to construct and decompose the~$S^1$-invariant graph Laplacian, where we perceive~$X$ as being sampled from a~$S^1$-invariant manifold, as described above. 
Practically, this means that we construct and factor the matrices~$K^{(\ell)}$ in~\eqref{secGGL:fourierMatNorm} up to a certain threshold~$\ell\leq \ell_{\max}$. 
We then compute for each~$i\in [N]$ the truncated~$S^1$-invariant embedding~$\Psi_{\ell_{\max},t}^{(p)}(i)$ given by~\eqref{invDmaps:probInvEmbeding} with~$\ell\leq \ell_{\max}$, and~$n\leq n_{\ell}$ for each~$\ell$, where~$n_{\ell}$ is the number of leading eigenvectors of~$K^{(\ell)}$ we use to construct the embedding. Below, we describe how we chose~$\ell_{\max}$ and~$n_{\ell}$ in our simulations. 

In Step~2, we first fix an integer~$K\ll N$ and a diffusion time parameter~$t\geq0$, and compute for each projection $x_i$ the set~$\mathcal{N}_i$ of its~$K$ nearest $S^1$-invariant neighbors defined as the $K$ projections~$x_j$ with the smallest distance~$\norm{\Psi_{\ell_{\max},t}^{(p)}(i)-\Psi_{\ell_{\max},t}^{(p)}(j)}$ in the embedding space. Then, for each~$i\in [N]$, the set~$\mathcal{N}_i$ contains the~$K$ neighbors of~$x_i$ up to shifts. 
Step~1 and Step~2 are outlined in~Algorithm~\ref{numerSec:shiftInvNearestNeighbors}. 
	
In Step~3, for each projection~$x_i$ and each of its neighboring projections~$x_j\in \mathcal{N}_i$ determined in Step~2, we employ the truncated equivariant embedding of~$\eqref{eqvDmaps:eqvEmbeddingFinite}$ to compute the relative shift $s_{ij}$ that best aligns~$x_j$ with~$x_i$ by solving 
\begin{equation}\label{numericSec:bruteInvDist}
	s_{ij} = \min_{s\in S}\norm{\Phi^{(p)}_{\delta,t}(i,\Theta^{-1}(s))-\Phi^{(p)}_{\delta,t}(j,\pi)},
\end{equation}
where~$\Theta^{-1}$ is the inverse map of~$\Theta$ defined in~\eqref{numerSec:shiftToAngMapDef}, 
and~$\Phi^{(p)}_{\delta,t}(i,\Theta^{-1}(s))$ is the embedding of~$s\circ x_i$. 
In particular, by~\eqref{numerSec:shiftToAngMapDef}, we have that~$\Phi^{(p)}_{\delta,t}(j,\Theta^{-1}(0)) = \Phi^{(p)}_{\delta,t}(j,\pi)$, which is the embedding of the point~$x_j$.
We discuss below how we chose the threshold~$\delta$ in our simulations. We point out that it is also possible to compute the relative shifts~$s_{ij}$ by directly aligning the projections. That is, by solving
\begin{equation}\label{numericSec:bruteInvDist2}
	s_{ij} = \min_{s}\norm{s\circ x_i-x_j}, \quad s\in S,
\end{equation}
where~$s\circ x_i$ is defined as
\begin{equation}
	s\circ x_i =\left(P_{\varphi_i}(t_1+s_i-s),\ldots, P_{\varphi_i}(t_1+s_i-s)\right).
\end{equation}
However, as we demonstrate in simulations below, while both methods of computing the~$s_{ij}$ produce comparable results, the method~\eqref{numericSec:bruteInvDist} can be more computationally efficient. 

	\begin{algorithm}
		\caption{Global alignment of projections
		}\label{numerSec:globalAlign}
		\begin{algorithmic}[1]
			\Statex{\textbf{Input:}}
			\begin{enumerate}
				\item Pairwise relative shifts $s_{ij}$ for all~$i,j\in[N]$ such that~$i\in \mathcal{N}_j$ or~$j\in \mathcal{N}_i$. 
				\item Shifted projections~$x_i = \left(P_{\theta_i}(t_1+s_i),\ldots,P_{\theta_i}(t_m+s_i)\right)$, for~$i=1,\ldots ,N$. 	
			\end{enumerate}		
			\State{Compute~$\theta_{ij} = \Theta(s_{ij})$ for all~$i,j\in[N]$ (see~\eqref{numerSec:shiftToAngMapDef})}. 
			\State{Construct the~$N\times N$ matrix~$H$ defined in~\eqref{numericSec:globalSyncMatDef}}.  
			\State{Compute the leading eigenvector~$h$ of~$H$.}
			\State{Set~$\tilde{s}_i$ = $\Theta^{-1}(\arg(h(i)))$ for all~$i\in[N]$ ($\arg(\cdot)$ is the complex argument function)}.
			\State{Compute $\tilde{x}_i = \left(P_{\theta_i}(t_1+s_i-\tilde{s}_i),\ldots,P_{\theta_i}(t_m+s_i-\tilde{s}_i)\right)$, for~$i=1,\ldots,N$.}
		\end{algorithmic}
	\end{algorithm}

In Step~4, we use the relative shifts~$s_{ij}$ to derive a set of shifts~$\tilde{s}_1,\ldots,\tilde{s}_N$ such that the consistency relations
\begin{align}\label{numerSec:shiftConsistency}
	s_{ij} \approx \tilde{s}_j-\tilde{s}_i, \quad i\in\mathcal{N}_j \text{ or } j\in \mathcal{N}_i, 
\end{align}
approximately hold for all pairs $(i,j)$ simultaneously. We use the method of angular synchronization proposed in~\cite{Singer2009AngularSB}, which solves an equivalent problem for relative rotation angles instead of relative shifts (recall that each relative shift~$s_{ij}$ can be mapped to a relative rotation angle via~\eqref{numerSec:shiftToAngMapDef}). We now briefly describe this method.  
Consider the set of all angles $\theta_{ij}=\Theta(s_{ij})$, where either~$i\in\mathcal{N}_j$ or $j\in\mathcal{N}_i$ (or both). We construct an~$N\times N$ matrix~$H$ defined by
	\begin{equation}\label{numericSec:globalSyncMatDef}
		H_{ij} = \begin{cases}
			e^{i\theta_{ij}} & i\in \mathcal{N}_j \text{ or } j\in \mathcal{N}_i,\\ 
			0 & \text{otherwise}.
		\end{cases} 
	\end{equation}
It is shown in~\cite{Singer2009AngularSB} that if the neighbors of each point in~$X$ are identified accurately, then for a sufficiently large~$N$, the complex arguments~$\arg(h(i))$ of the elements of~$h$, the top eigenvector of~$H$, are a set of angles~$\tilde{\theta}_1,\ldots,\tilde{\theta}_N$ for which the consistency relations
	\begin{align}\label{numerSec:angularConsistency}
		\theta_{ij} \approx \tilde{\theta}_j-\tilde{\theta}_i, \quad i\in \mathcal{N}_j \text{ or } j\in \mathcal{N}_i, 
	\end{align}
hold with high probability. Computing~$\overline{\theta}_i$ by the method just described we obtain the set of shifts~$\tilde{s}_i = \Theta^{-1}(\tilde{\theta}_i)$ for which, by~\eqref{numerSec:angularConsistency} and the linearity of~\eqref{numerSec:shiftToAngMapDef}, the consistency relations~\eqref{numerSec:shiftConsistency} approximately hold. We can now shift the projections~$x_i$ by the obtained shifts~$\tilde{s}_i$, so that~\eqref{numericSec:reAlignEq} holds. 
The procedure just described is outlined in Algorithm~\ref{numerSec:globalAlign}. 
	
At this point, all the projections are aligned with respect to each other (due to~\eqref{numerSec:shiftConsistency}). However, it may be that all the projections are shifted together by a single global shift~$s$ with respect to the center of each projection. Thus, Step~5 (the last step of our method) is to resolve this last degree of freedom, as follows. Let~$r$ be the dimension of the 1D projections in~$X$ (in pixels). First, we form the 2D array of size~$N\times r$ pixels obtained by placing the~$1\times r$ aligned projections in a stack of height~$N$. We then center this array by shifting it to its center of mass, and use the resulting stack of projections as our data set. At this point we can input the resulting aligned and centered projections to Algorithm~\ref{numerSec:ordAlg} to obtain the order of the projections, and reconstruct the image. 

\begin{figure}
	\centering
	\subfloat[Clean]  	
	{
		\label{fig:sheppLoganClean1}
		\includegraphics[width=0.21\textwidth]{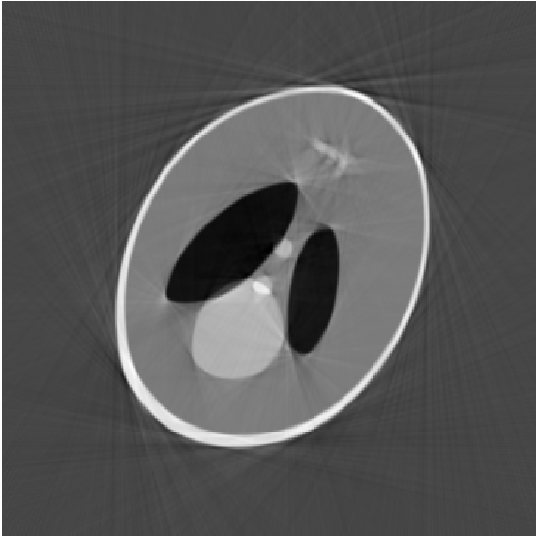}		
	}
	\subfloat[30dB]  	
	{
		\includegraphics[width=0.21\textwidth]{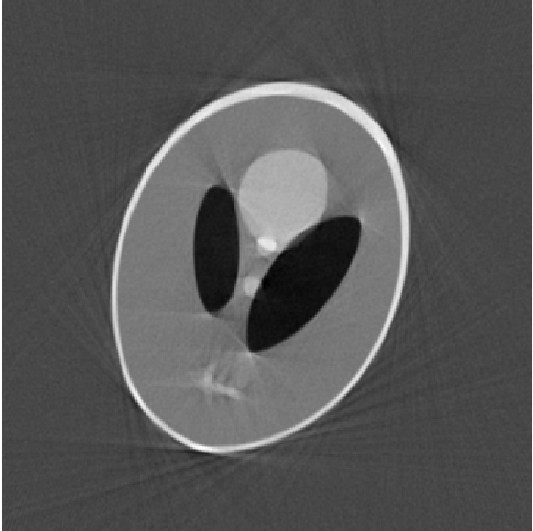}
		\label{fig:sheppLogan30dB}
	}
	\subfloat[10dB]  	
	{
		\includegraphics[width=0.21\textwidth]{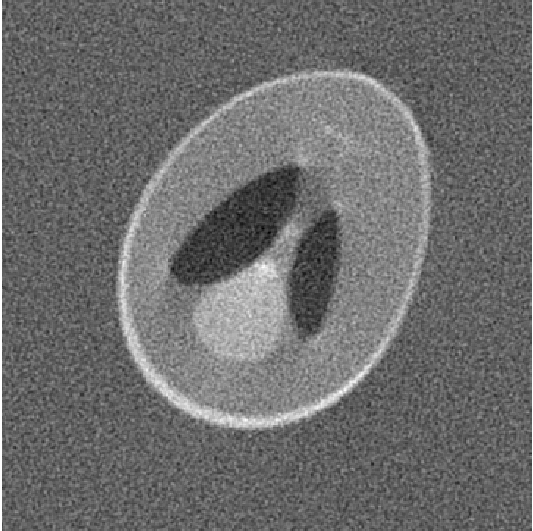}
		\label{fig:sheppLogan10dB}
	}
	
	\subfloat[2dB]    
	{ 
		\includegraphics[width=0.21\textwidth]{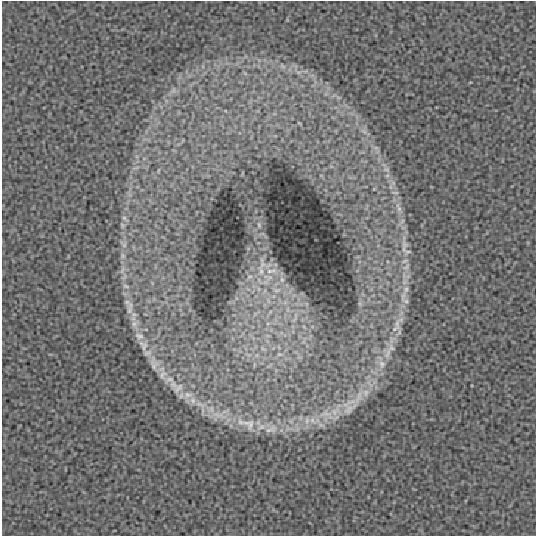}
		\label{fig:sheppLogan0dB}
	}
	\subfloat[-3dB]    
	{ 
		\includegraphics[width=0.21\textwidth]{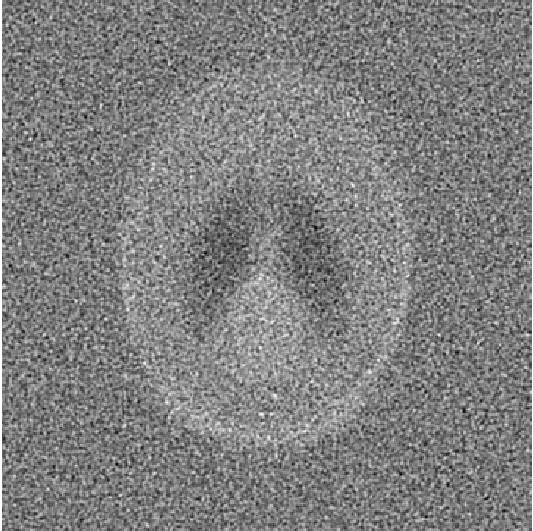}
		\label{fig:sheppLoganm1dB}
	}
	\subfloat[-4dB]    
	{ 
		\includegraphics[width=0.21\textwidth]{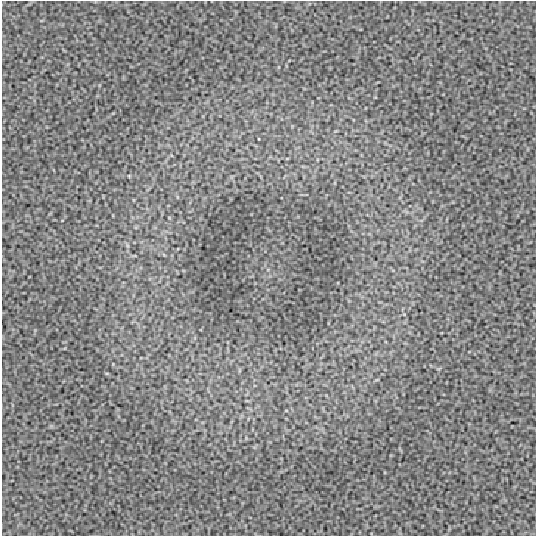}
		\label{fig:sheppLoganm2dB}
	}
	
	\caption{Shepp-Logan phantom reconstructed from 256 shifted random projections at various levels of noise, after centering them by using our method.} 
	\label{fig:sheppLogan}
\end{figure}  
	
\begin{figure}
	\centering
	\subfloat[Clean]  	
	{
		\label{fig:ordClean}
		\includegraphics[width=0.3\textwidth]{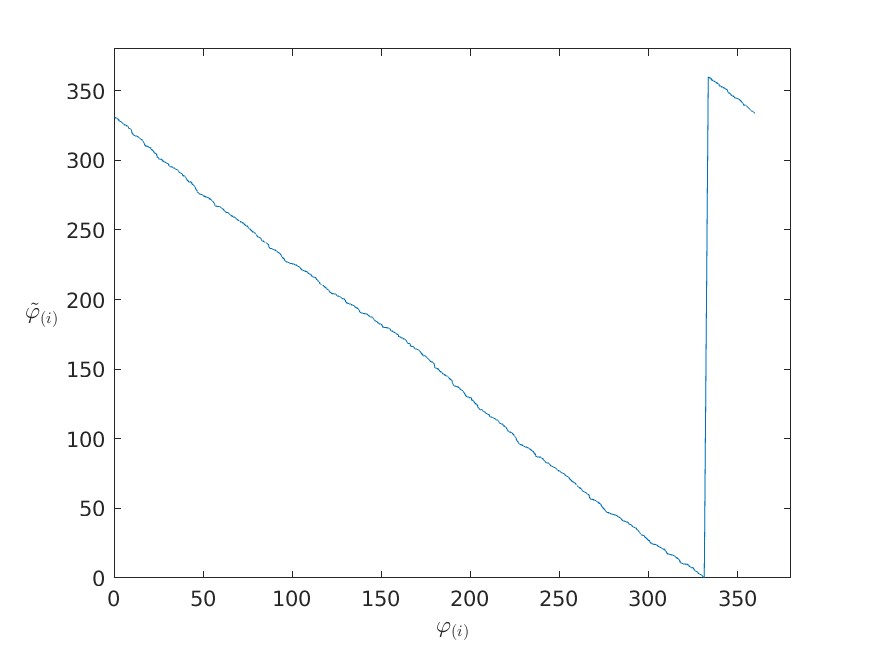}
	}
	\subfloat[30dB]  	
	{
		\includegraphics[width=0.3\textwidth]{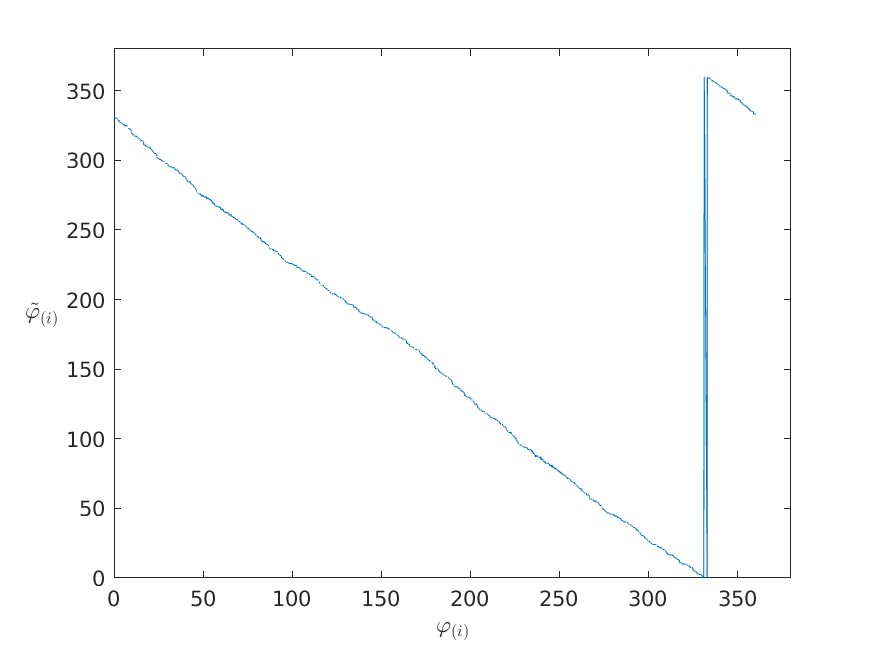}
	}\label{fig:ord30dB}
	\subfloat[10dB]  	
	{
		\includegraphics[width=0.3\textwidth]{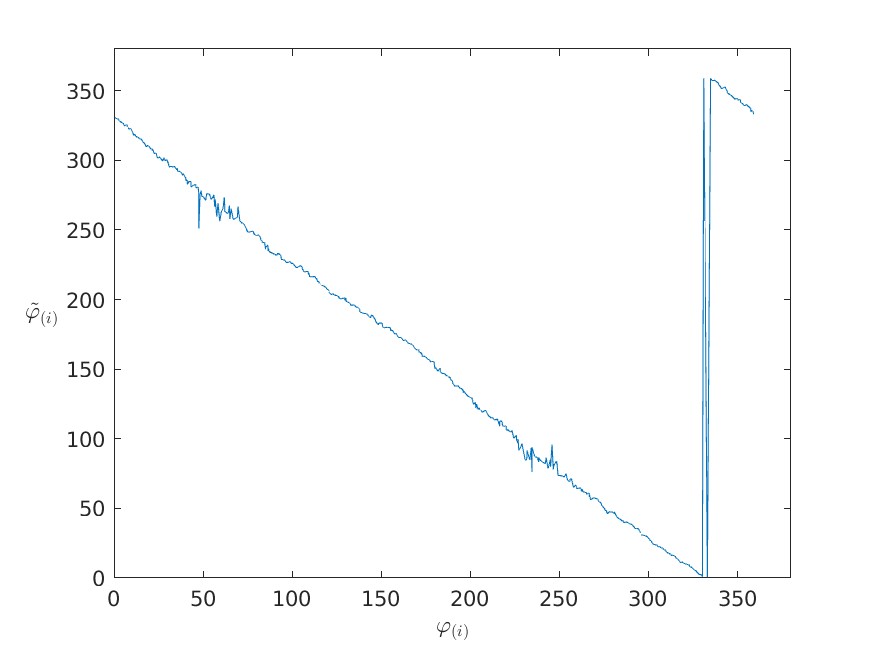}
	}\label{fig:ord10dB}
	
	\subfloat[2dB]  	
	{
		\includegraphics[width=0.3\textwidth]{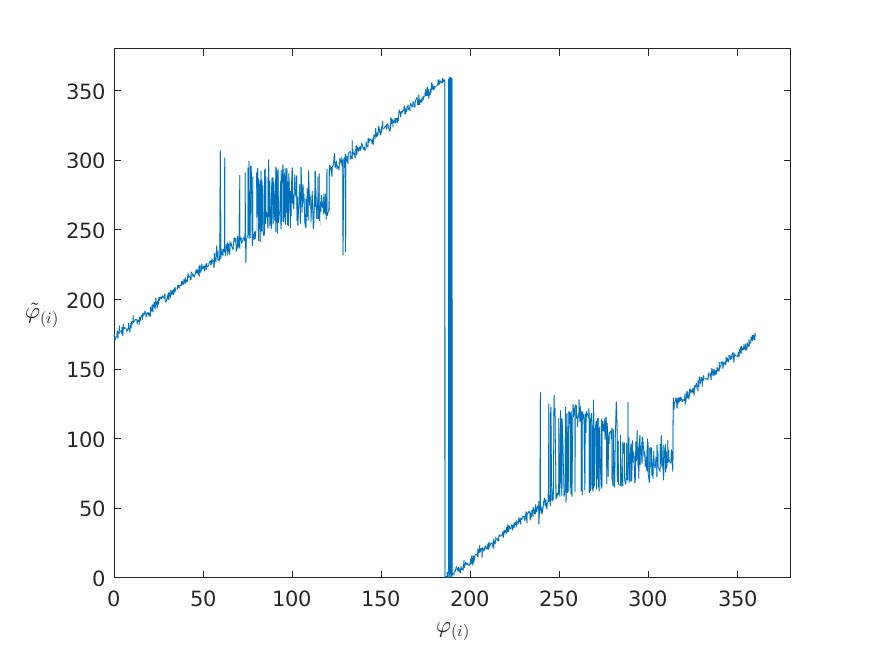}
	}\label{fig:ord0dB}
	\subfloat[-3dB]    
	{ 
		\includegraphics[width=0.3\textwidth]{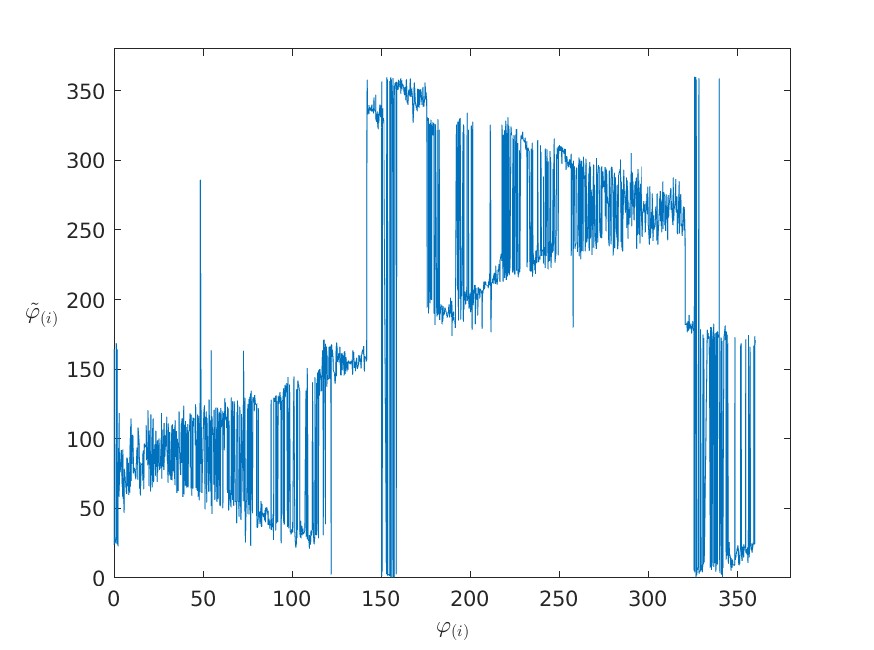}
	}\label{fig:ordm1Db}
	\subfloat[-4dB]    
	{ 
		\includegraphics[width=0.3\textwidth]{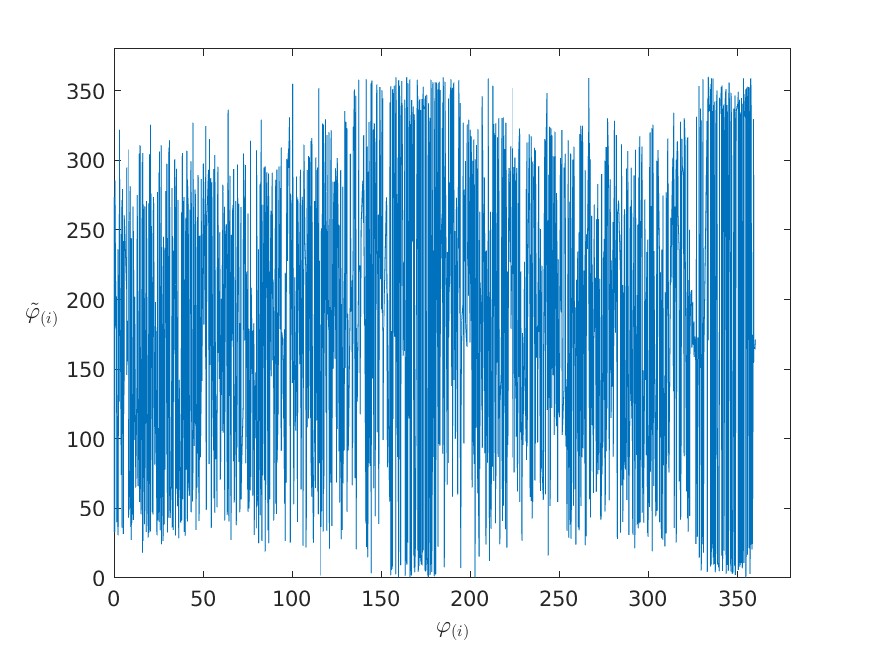}
	}\label{fig:ordm2Db}
	
	\caption{Angles~$\tilde{\varphi}_{(i)}$ obtained by ordering the true projection angles~$\varphi_i$ according to the order of the sorted angles~$\tilde{\varphi}_i$ of~\eqref{numerSec:ordAng}, plotted against the angles~$\varphi_{(i)}$ obtained by sorting $\varphi_i$. The angles~$\tilde{\varphi_i}$ were estimated by passing to Algorithm~\ref{numerSec:ordAlg} the class-averages of the~$32$ nearest neighbors of each~$x_i\in X$, determined by the~$S^1$-invariant diffusion maps.} 
	\label{fig:ordGraph}
\end{figure} 	
	
To demonstrate the method just described, we applied it to the reconstruction of the Shepp-Logan Phantom from its projections generated at random angles and random shifts. The result is depicted in~Figure~\ref{fig:sheppLoganClean1}. 
This figure was generated as follows. First, we generated~$N=1024$ uniformly distributed angles from~$[0,2\pi)$, denoted by~$\varphi_1,\ldots, \varphi_N$. For each~$\theta_i$, we evaluated the analytic expression of the Radon transform of the Shepp-Logan phantom at~$m=512$ equally spaced point between -1.5 and 1.5. Thus, each projection~$x_i$ is a 1D vector of 512 pixels. The maximal shift~$s_{\max}=102$ was chosen to be approximately~40$\%$ of the support of the signal in each projection, which is around~256 pixels on average. We then sampled~$N$ integer shifts~$s_i$ from a uniform distribution over~$S=\left\{-102,\ldots,102\right\}$, and shifted each vector~$x_i$ by~$s_i$ pixels. Next, we applied Steps~$1-5$ described above to the projections, as follows. 
The bandwidth~$\epsilon$ was chosen by using the approximation rule proposed in~\cite{glRandTomography}
	\begin{equation}\label{numericSec:optEpsRule}
		\epsilon_{\text{opt}} = \argmax_\epsilon \frac{\partial \log \text{Tr} \left\{D(\epsilon)\right\}}{\partial \log \epsilon}, 
	\end{equation}
where~$\text{Tr} \left\{D(\epsilon)\right\}$ is the trace of the diagonal matrix~$D(\epsilon)$, which is the matrix~$D$ in~\eqref{GinvDef:DiiDef}, whose elements were computed by using the bandwidth parameter~$\epsilon$. 
Then, we applied Algorithm~\ref{numerSec:shiftInvNearestNeighbors} to the data. Specifically, we computed the~$S^1$-invariant diffusion maps with diffusion time~$t=0$, where for each IUR index~$\ell\in \mathcal{I}_{S^1}$ (see Appendix~\ref{secHarmAnalysis}) we used the top~$n_{\ell}$ eigenvectors of the matrix~$K^{(\ell)}$ in~\eqref{secGGL:fourierMatNorm} that satisfy~$\lambda_{n,\ell}>0.1$. The maximal IUR index~$\ell_{\max}\in \mathcal{I}_{S^1}$ was chosen as follows. First, we set~$\ell_{\max}=1$ and computed the~$S^1$-invariant embedding~\eqref{invDmaps:probInvEmbeding}, where we used the~$n_{1}$ top eigenvectors of~$K^{(1)}$, chosen by using the condition described above. 
Then, for each~$x_i\in X$ we computed the~$K=32$ nearest neighbors~$x_j\in \mathcal{N}_i$ with the smallest Euclidean distance between their embedding and that of~$x_i$. We then computed the median~$\text{med}_1$ of these distances over all~$x_i$.
Next, we set~$\ell_{\max}=2$, and repeated the previous computation, constructing the diffusion maps by taking all the eigenvectors of the matrices~$K^{(1)}$ and~$K^{(2)}$, with the number of eigenvectors of each matrix~$K^{(\ell)}$  chosen using the same condition as for the case~$\ell_{\max}=1$.  Then, we used the resulting embedding to compute the~$32$ nearest neighbors of each~$x_i\in X$ the same way we did for~$\ell_{\max}=1$, and computed a new respective median nearest neighbor distance~~$\text{med}_2$. We then repeated this process of increasing~$\ell_{\max}$ until the relative change in the median~$(\text{med}_{i+1}-\text{med}_{i})/\text{med}_{i}$ was~$>-0.01$. 
Then, we computed the~$S^1$-equivariant embedding~\eqref{eqvDmaps:eqvEmbeddingFinite}, using the same eigenvectors that were employed for the construction of the~$S^1$-invariant embedding above. The resulting equivariant embedding~$\Phi^{(p)}_{\delta,t}$ has dimension~20, and corresponds to the threshold~$\delta = 0.1$. Then, for each~$x_i\in X$, we computed the relative shifts~$s_{ij}$ that best align the points~$x_j\in \mathcal{N}$ with~$x_i$, by solving~\eqref{numericSec:bruteInvDist}. 
Next, we applied Algorithm~\ref{numerSec:globalAlign}, using~$|\mathcal{N}_i|=32$ nearest neighbors computed as described above, to construct~$H$ in~\eqref{numericSec:globalSyncMatDef}, and aligned the projections by using the resulting shifts as on the l.h.s of~\eqref{numericSec:reAlignEq}. The aligned projections were then centered as described in Step~5 above. We then applied Algorithm~\ref{numerSec:ordAlg} to the aligned and centered data set to order the projections according to the angles which generated them. 
Finally, we reconstructed the image shown in Figure~\ref{fig:sheppLoganClean1} by using a subset of~$256$ of the ordered projections with equally spaced indexes~$1,5,\ldots,1019$ (as was done in~\cite{glRandTomography}). In Figure~\ref{fig:ordClean}, we demonstrate that our method manages to recovering the ordering of the true projection angles~$\varphi$, by graphing the angles
	\begin{equation}\label{numerSec:orderedTrueAng}
		\tilde{\varphi}_{(1)}, \tilde{\varphi}_{(2)},\ldots, \tilde{\varphi}_{(N)},
	\end{equation}
	obtained by ordering~$\varphi_i$ according to the order of the sorted angles~$\tilde{\varphi}_i$ of~\eqref{numerSec:ordAng} (obtained by Algorithm~\ref{numerSec:ordAlg}), against
	\begin{equation}\label{numerSec:sortedTrueAng}
		\varphi_{(1)} \leq \varphi_{(2)}\leq \ldots\leq \varphi_{(N)},
	\end{equation}
	the angles~$\varphi_i$ sorted in ascending order.  

\begin{figure}
	\centering
	\subfloat[Image reconstruction] 	
	{
		\label{fig:imRecClean}
		\includegraphics[width=0.35\textwidth]{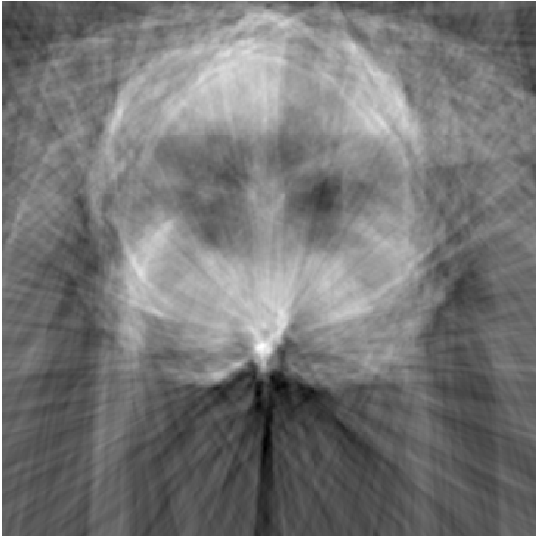}
	}
	\subfloat[Estimated ordered angles]  	
	{
		\label{fig:ordCleanDM}
		\includegraphics[width=0.47\textwidth]{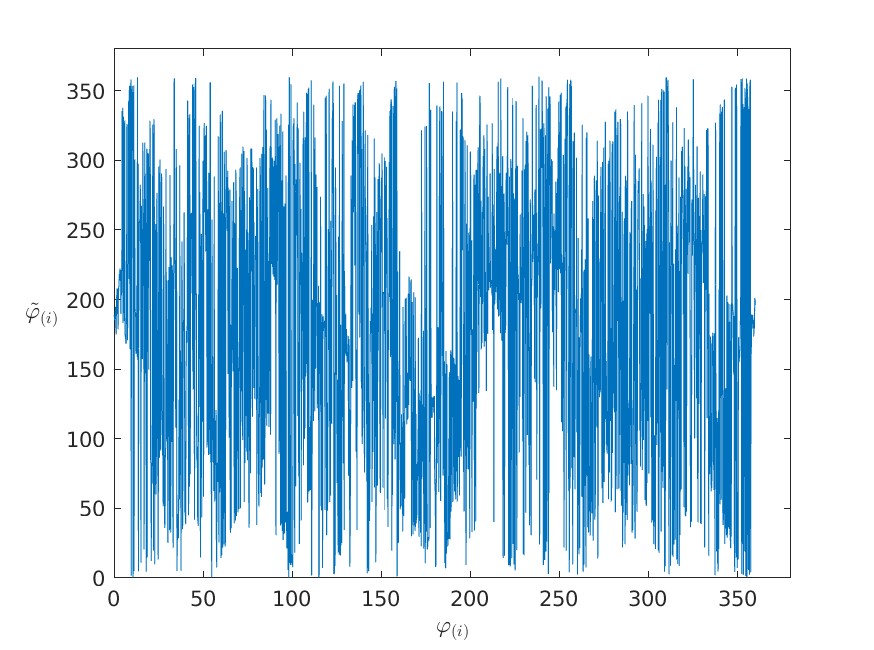}
	}
	
	\caption{Reconstruction by directly applying Algorithm~\ref{numerSec:ordAlg} to the shifted projections.} 
	\label{fig:DMperform}
\end{figure}

We also applied Algorithm~\ref{numerSec:ordAlg} directly to the shifted projections in~$X$. 
Figure~\ref{fig:imRecClean} shows the reconstructed phantom, demonstrating that
Algorithm~\ref{numerSec:ordAlg} fails when the projections are shifted. As we explained above, this is attributed to the fact that shifted projections are scattered on a two-dimensional manifold rather than a curve. Thus, the order of the projection angles cannot be recovered from the two non-trivial leading eigenvectors~$\phi_2,\phi_3$ of~$\tilde{L}$. This is demonstrated in Figure~\ref{fig:ordCleanDM}, where we graph the angles~$\tilde{\varphi}_{(i)}$ in~\eqref{numerSec:orderedTrueAng} against~$\varphi_{(i)}$ in~\eqref{numerSec:sortedTrueAng}.

A common method to deal with shifts in the data, is to shift it so that its center of mass (CM) is located in the center of the projection vector~\cite{centerCryoEM}. Figure~\ref{fig:sheppLoganCleanDMCM} shows the high quality reconstruction obtained by applying Algorithm~\ref{numerSec:ordAlg} after centering according to their CMs. In Figure~\ref{fig:ordCleanDMCM}, we graph the angles~$\tilde{\varphi}_{(i)}$ in~\eqref{numerSec:orderedTrueAng} against~$\varphi_{(i)}$ in~\eqref{numerSec:sortedTrueAng}, showing that after centering the projections Algorithm~\ref{numerSec:ordAlg} succeeds in retrieving the ordering of the projection angles (the jump discontinuity observed in the graph is attributed to the fact that the order of the angles can only be retrieved up to a cyclic permutation, as explained in Remark~\ref{numericSec:degreeOfFreedom}). 

We now demonstrate the performance of our proposed method with noisy projections, where we measure the amount of noise in the projections by the signal-to-noise ratio (SNR) measured in decibels,  defined here as 
\begin{eqnarray}
	\text{SNR}_{\text{db}} = 10 \log_{10}\left(\frac{\text{Var}(X)}{\sigma^2}\right), 
\end{eqnarray}
where~$\sigma^2$ is the variance of the white noise, and~$\text{Var}(X)$ is the sample variance of the data set. 
It was observed in~\cite{glRandTomography} that Algorithm~\ref{numerSec:ordAlg} performs well for $\text{SNR}_{\text{db}}\geq 10.6$ and undergoes an abrupt phase transition for $\text{SNR}_{\text{db}}\leq 10.5$, performing poorly and failing to retrieve the order of the projections. It was reasoned that this threshold effect is caused by the noise thickening the curve~$C$ of the Radon projections, making the graph Laplacian treat the data as a surface instead of a curve. 
However, by denoising the projections before applying Algorithm~\ref{numerSec:ordAlg} using a wavelet based low-pass filter, the~$10.5$dB threshold in~\cite{glRandTomography} was pushed down to~$2$dB. The latter threshold was pushed further to~$-5$dB in~\cite{singerRT}, that proposed applying several advanced preliminary denoising methods to the projections, before using Algorithm~\ref{numerSec:ordAlg}. 

To have a point of reference, we repeated the procedure described in~\cite{glRandTomography} with noisy shifted projections after centering them using their CMs, as described above. 
The resulting reconstructions for several~SNR levels are shown in~Figures~\ref{fig:sheppLogan10dBDMCM}-\ref{fig:sheppLogan3dBDMCM}. We see that after centering Algorithm~\ref{numerSec:ordAlg} performs reasonably well for $\text{SNR}_{\text{dB}}\geq 4$, although with some features of the reconstructed phantom visibly distorted already at~$\text{SNR}_{\text{dB}}=10$. 

To deal with noise in our simulations, we employed instead a method known as class-averaging~\cite{classAverage} (that provides superior results), as follows. 
After we aligned the projections by Algorithm~\ref{numerSec:globalAlign}, 
we generated a new data set~$X_{\text{CA}}$ where each projection~$x_i$ was replaced with the average of its~$32$ aligned nearest neighbors in~$\mathcal{N}_i$ (including~$x_i$ itself). The idea is that after alignment the majority of the neighbors $x_j\in \mathcal{N}_i$ are approximately equal to~$x_i$, and thus, the random white noise having a zero mean gets averaged out by averaging all the neighbors, producing a denoised version of~$x_i$. We then used the data set~$X_{\text{CA}}$ as an input to~Algorithm~\ref{numerSec:ordAlg} to estimate the angles~$\tilde{\varphi}_1,\ldots,\tilde{\varphi}_N$, and assigned to each~$x_i\in X$ the angle~$\tilde{\varphi}_i$. We then ordered the projections~$x_i\in X$ according to their assigned angles~$\tilde{\varphi}_i$. Finally, we reconstructed the image from the sorted aligned projections~$x_i\in X$, assuming that their projection angles are equally spaced in~$[0,2\pi)$ (see~\cite{glRandTomography} for a detailed justification).

\begin{figure}
	\centering
	\subfloat[Clean]    
	{ 
		\includegraphics[width=0.2\textwidth]{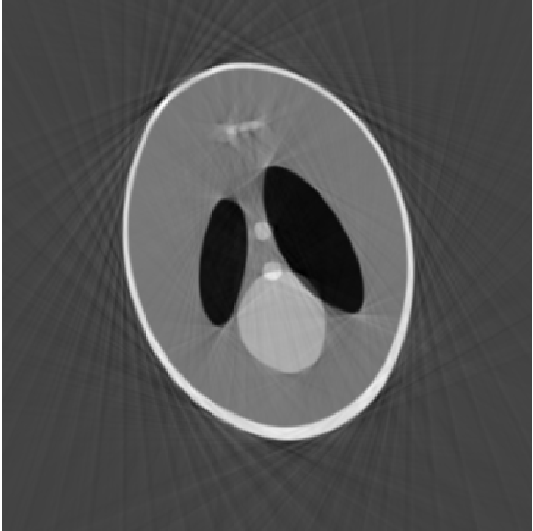}
		\label{fig:sheppLoganCleanDMCM}
	}
	\hspace{\fill}
	\subfloat[10dB]  	
	{
		\includegraphics[width=0.2\textwidth]{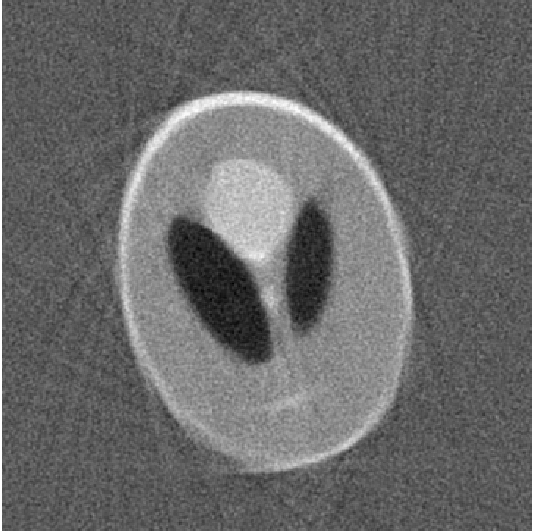}
		\label{fig:sheppLogan10dBDMCM}
	}	
	\hspace{\fill}
	\subfloat[4dB]    
	{ 
		\includegraphics[width=0.2\textwidth]{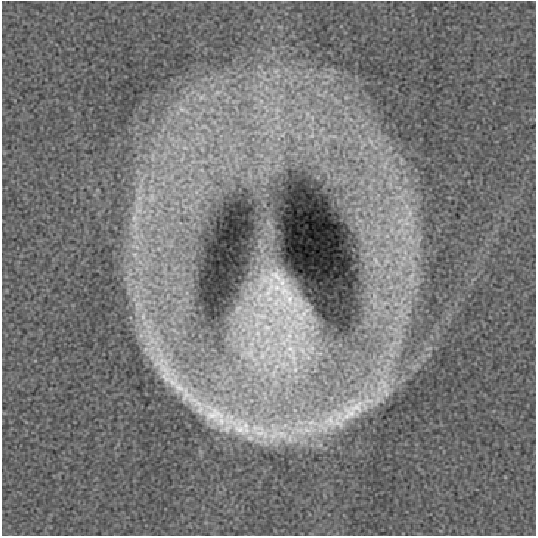}
		\label{fig:sheppLogan4dBDMCM}
	}
	\hspace{\fill}
	\subfloat[3dB]    
	{ 
		\includegraphics[width=0.2\textwidth]{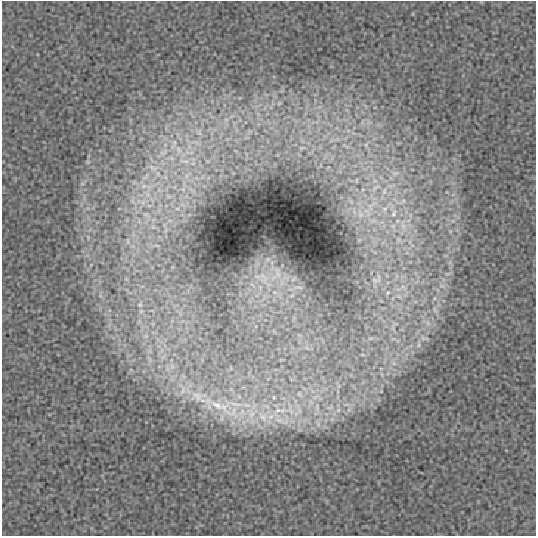}
		\label{fig:sheppLogan3dBDMCM}
	}
	\hspace{\fill}
	\subfloat[Clean]    
	{ 
		\includegraphics[width=0.21\textwidth]{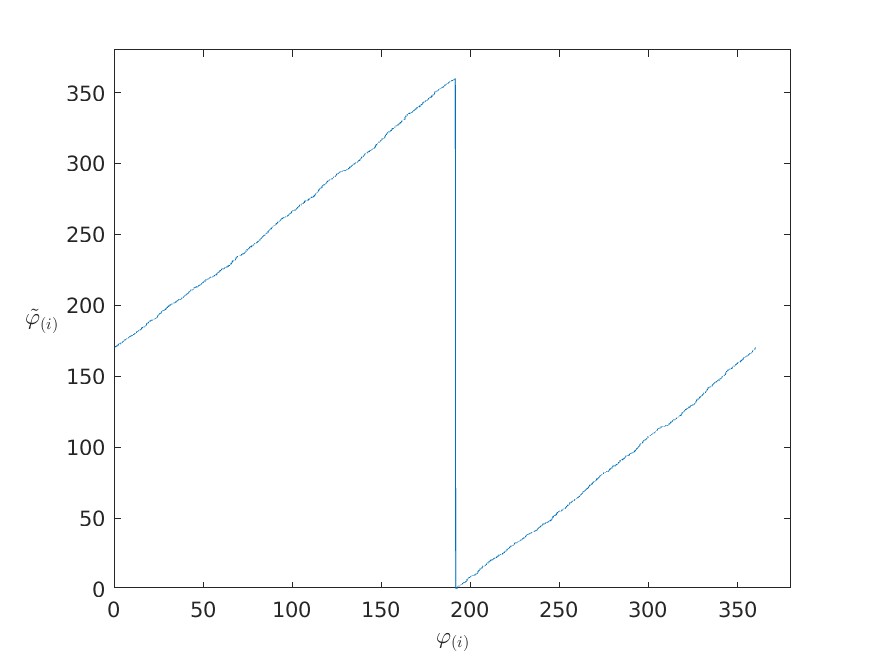}
		\label{fig:ordCleanDMCM}
	}
	\hspace{\fill}
	\subfloat[10dB]    
	{ 
		\includegraphics[width=0.21\textwidth]{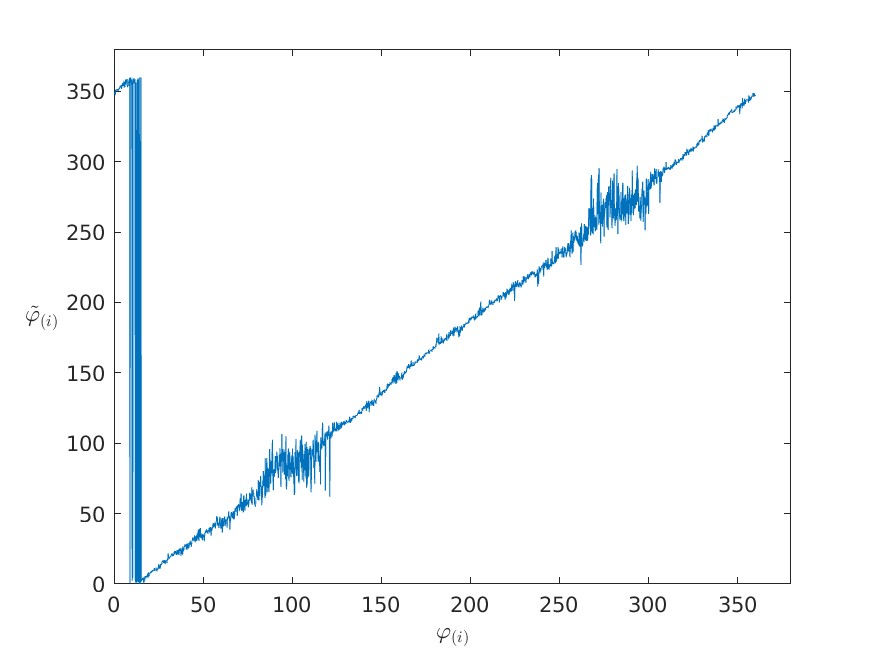}
		\label{fig:ord10dBDMCM}
	}
	\hspace{\fill}
	\subfloat[4dB]    
	{ 
		\includegraphics[width=0.21\textwidth]{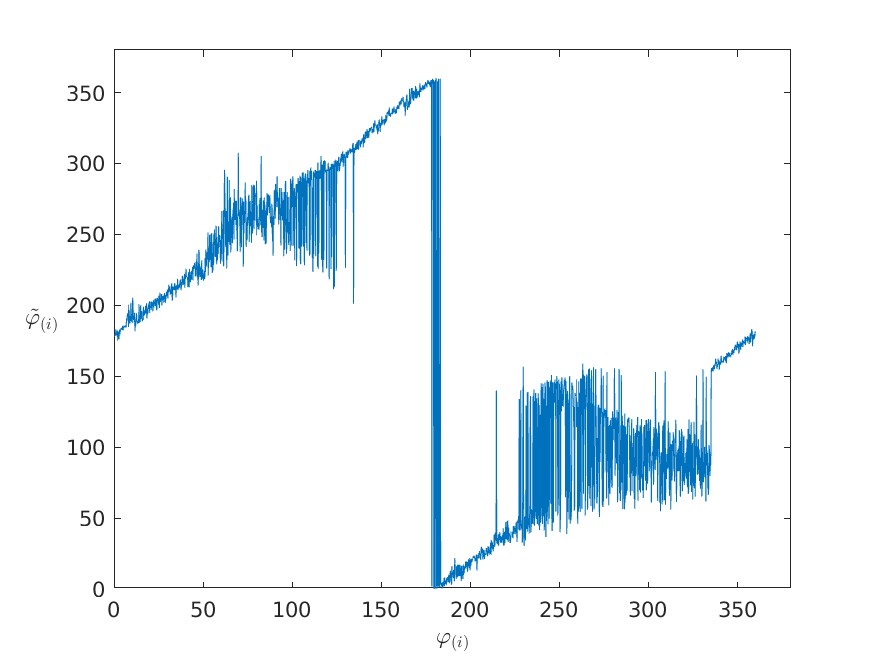}
		\label{fig:ord4dBDMCM}
	}
	\hspace{\fill}
	\subfloat[3dB]    
	{ 
		\includegraphics[width=0.21\textwidth]{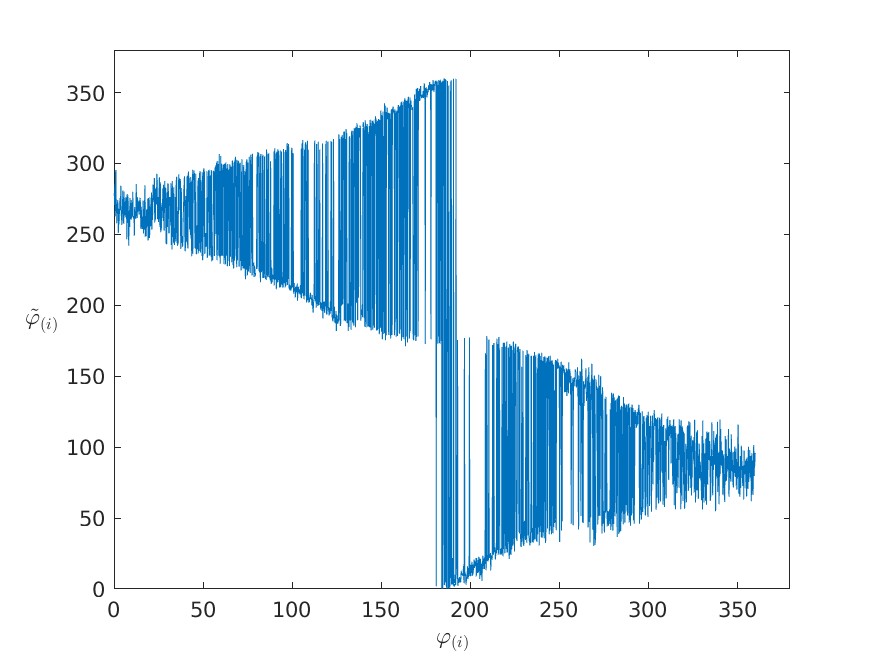}
		\label{fig:ord3dBDMCM}
	}	
	\caption{Shepp-Logan phantom reconstructed from 256 shifted random projections at various levels of noise, ordered by using Algorithm~\ref{numerSec:ordAlg} after centering each projection based on its center of mass.} 
	\label{fig:sheppLoganDMCM}
\end{figure}

\begin{figure}
	\centering
	\subfloat[2dB]  	
	{
		\includegraphics[width=0.45\textwidth]{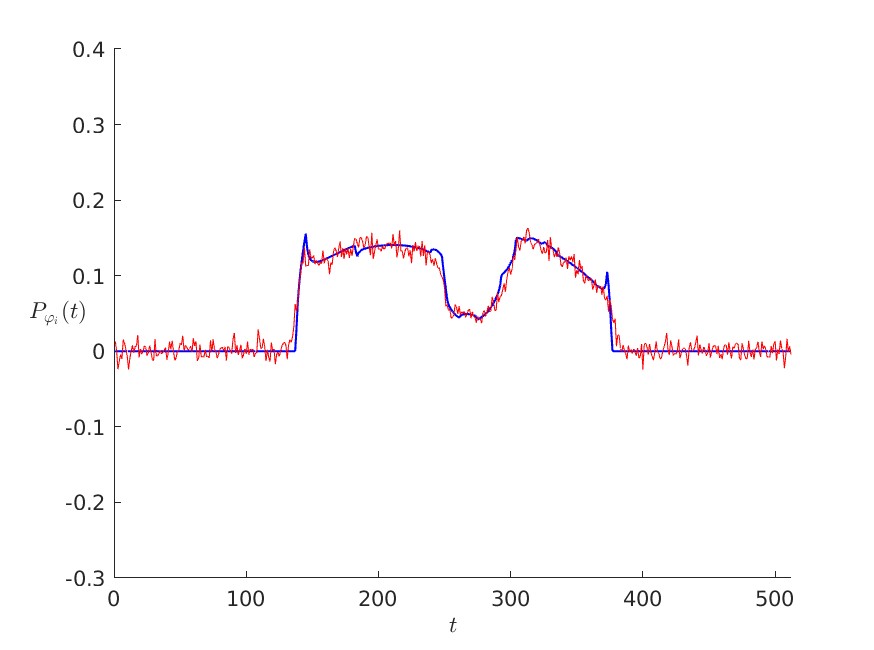}
	}\label{fig:sheppLogan10dB}
	\subfloat[2dB] 
	{
		\includegraphics[width=0.45\textwidth]{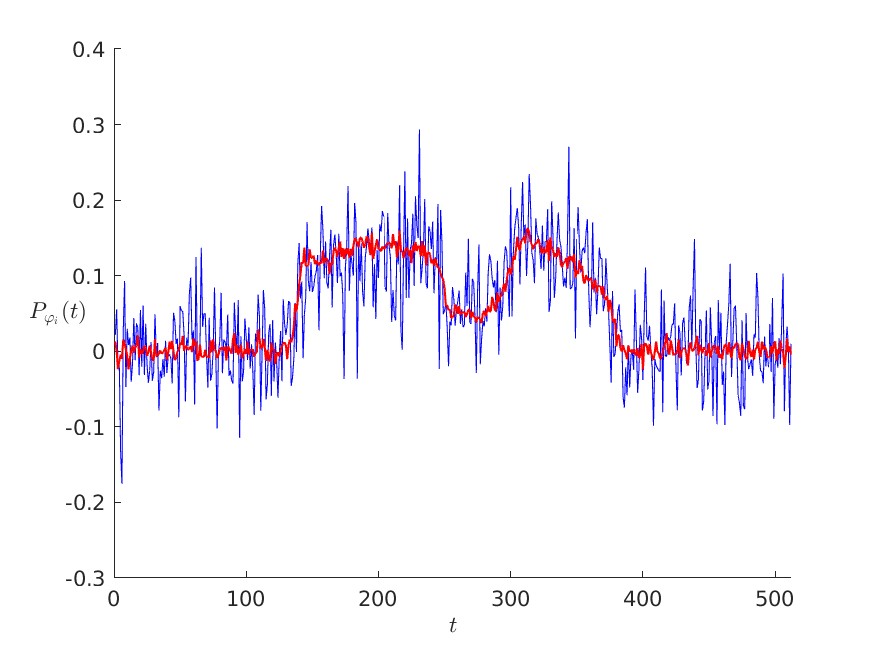}
	}\label{fig:proj10dB}
	\subfloat[-3dB]  	
	{
		\includegraphics[width=0.45\textwidth]{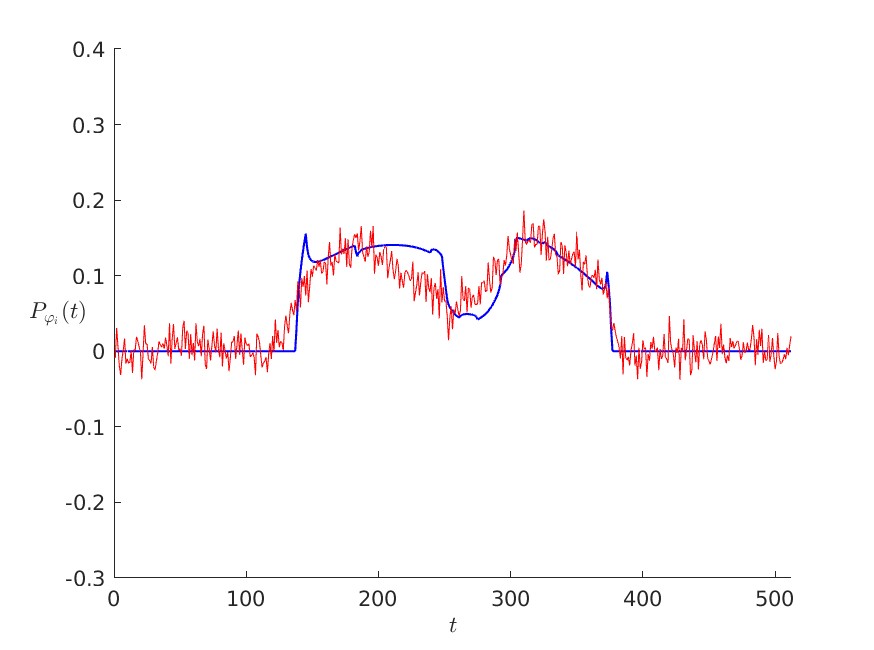}
	}\label{fig:sheppLogan10dB}
	\subfloat[-3dB] 
	{
		\includegraphics[width=0.45\textwidth]{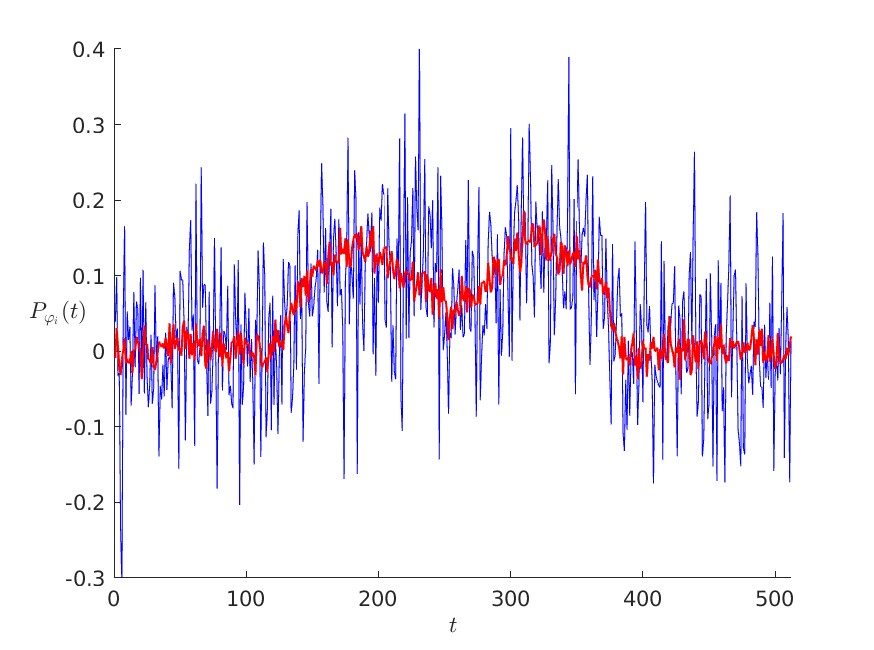}
	}\label{fig:proj10dB}
	
	\subfloat[-4dB]  	
	{
		\includegraphics[width=0.45\textwidth]{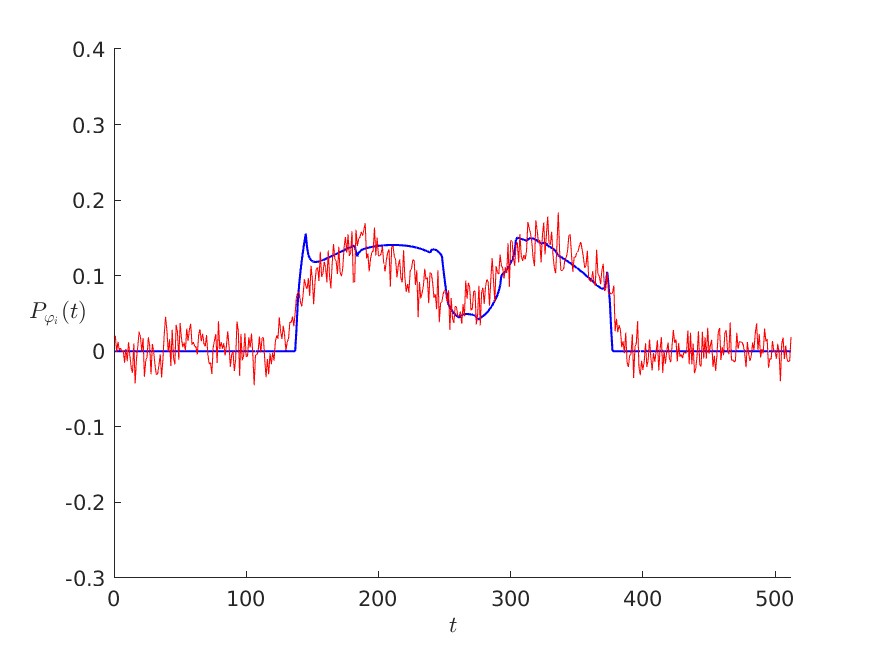}
	}\label{fig:sheppLogan10dB}
	\subfloat[-4dB] 
	{
		\includegraphics[width=0.45\textwidth]{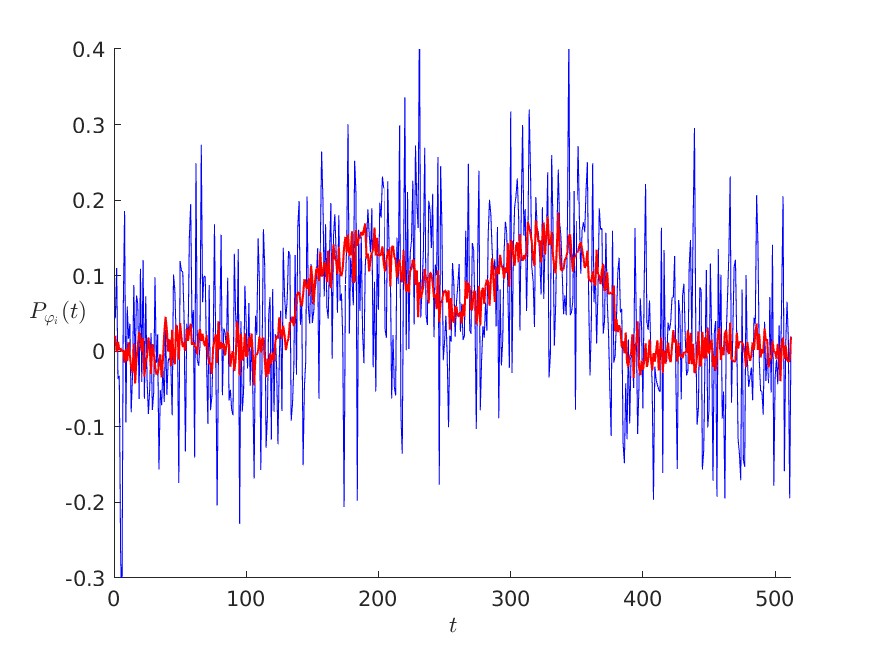}
	}\label{fig:proj10dB}

	\caption{Right column: a randomly chosen projection~$P_{\varphi_i}$ with additive white noise at various noise levels (blue), and its denoising by class-averaging (32 nearest neighbors) at each noise level (red). Left column: denoised projections from the right column (red), superimposed on the clean projection (blue).} 
	\label{fig:classAveraging}
\end{figure} 

To demonstrate our method with denoising by class-averaging, we added various amounts of additive Gaussian white noise to the data set of~$N=1024$ projections generated in the previous section. The reconstructed images corresponding to SNRs~$30$dB, $10$dB, $2$dB, $-3$dB, and~$-4$dB are shown in Figure~\ref{fig:sheppLogan}. In particular, we see that even in the presence of shifts, our method manages to go well beyond the~2dB threshold reported in~\cite{glRandTomography}, obtaining good reconstructions up to  $\text{SNR}_{\text{dB}}=-3$.
At $\text{SNR}_{\text{dB}} =-4$ the performance deteriorates, $2\text{dB}$ lower than the performance reported in~\cite{singerRT} without shifts. After a proper adaptation to account for shifts, the denoising methods used in~\cite{singerRT} can also be combined into to our algorithm as a preproccesing step, which we expect to significantly improve our results as well. We leave that for future work.

The effect of class-averaging on the performance of the method is illustrated in~Figure~\ref{fig:classAveraging}. The right column depicts a randomly chosen projection~$x_i=P_{\varphi_i}(t_1,\ldots,t_N)\in X$ at the various levels of noise (blue line), and its denoising (red line). The denoised projection was obtained by averaging~$x$ with its~$32$ nearest neighbors, determined by Algorithm~\ref{numerSec:shiftInvNearestNeighbors}. The left column shows the denoised projection in each row superimposed on the clean projection. We see that the denoised projection gives a good approximation to the clean projection even at~$\text{SNR}_{\text{db}}=-4$, where the method breaks down. 
This abrupt break down is attributed to the threshold effect observed in~\cite{glRandTomography} (discussed above), when applying Algorithm~\ref{numerSec:ordAlg} to the (aligned) projections. 
We illustrate this phenomenon in Figure~\ref{fig:ordGraph}, where for each SNR value 
we graph the angles~$\tilde{\varphi}_{(i)}$ of~\eqref{numerSec:orderedTrueAng} against the angles~$\varphi_{(i)}$ of~\eqref{numerSec:sortedTrueAng}.
Note that some of the graphs admit a jump discontinuity, and that the slope of each graph is either~$+1$ or~$-1$. This is just a manifestation of the degrees of freedom inherent to the problem, as we described in Remark~\ref{numericSec:degreeOfFreedom}. 
We see that the ordering of the projections deteriorates as the noise level grows, breaking down at~$\text{SNR}_{\text{db}} = -4$. 

Lastly, we wish to demonstrate the advantages of using the equivariant embedding~\eqref{eqvDmaps:eqvEmbeddingFinite} for pairwise alignment of projections, over directly aligning pairs of projections. For this, we repeated the simulations described above with SNR values of~$10$dB, $2$dB, $-3$dB, and $-4$dB, but with the pairwise alignment method~\eqref{numericSec:bruteInvDist} (that employs~\eqref{eqvDmaps:eqvEmbeddingFinite}) replaced by~\eqref{numericSec:bruteInvDist2} that directly aligns the projections. 
Figure~\ref{fig:sheppLoganBruteAlign} depicts the resulting reconstructed Shepp-Logan phantoms, showing that at moderate SNR levels the reconstructions are of similar quality to those in the previous simulations, but at lower SNRs alignment using the equivariant embedding performs better (compare with~Figure~\ref{fig:sheppLogan}). Furthermore, using the alignment method~\eqref{numericSec:bruteInvDist2} in Step 3 involves directly shifting the projections, where the complexity of a shift is of the order of the projections' dimension, which is~$512$. On the other hand, alignment using~\eqref{numericSec:bruteInvDist} requires applying the action of~$S^1$ to the embedded projections~$\Phi^{(p)}_{\delta,t}(i,\pi)$, which in our simulation have dimension~20. Using that the IURs of~$S^1$ are the Fourier modes~$e^{im\varphi}$ together with Proposition~\ref{eqvDmaps:truncEmbedEqvProp}, and in particular~\eqref{eqvDmaps:embedEquivar}, implies that the action of an element~$e^{i\varphi}\in S^1$ on the embedding~$\Phi^{(p)}_{\delta,t}(i,\pi)$ of each projection~$x_i$ can be computed by multiplying the embedding by a~$20\times 20$ diagonal matrix~$U$ with Fourier modes on the diagonal. The complexity of this operation is of the order of the dimension of~$\Phi^{(p)}_{\delta,t}(i,\pi)$ (since~$U$ is diagonal), which is a huge improvement over directly shifting the projections.

\begin{figure}
	\centering
	\subfloat[10dB]  	
	{
		\includegraphics[width=0.2\textwidth]{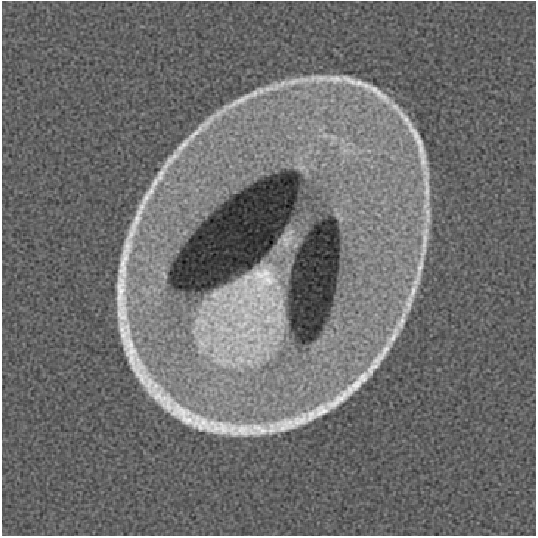}
		\label{fig:sheppLogan10dBBruteAlign}
	}	
	\subfloat[2dB]    
	{ 
		\includegraphics[width=0.2\textwidth]{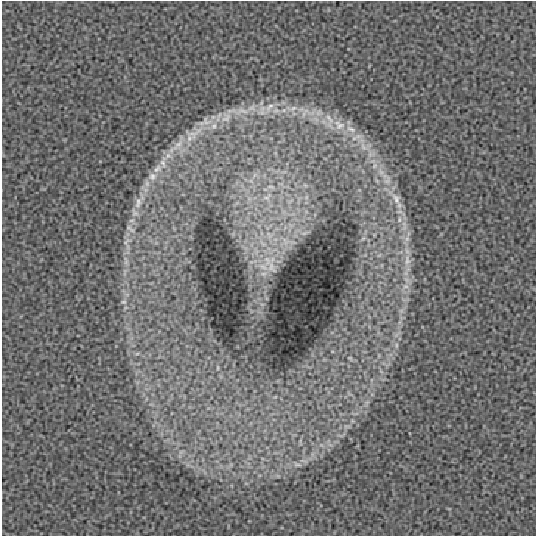}
		\label{fig:sheppLogan0dBBruteAlign}
	}
	\subfloat[-3dB]    
	{ 
		\includegraphics[width=0.2\textwidth]{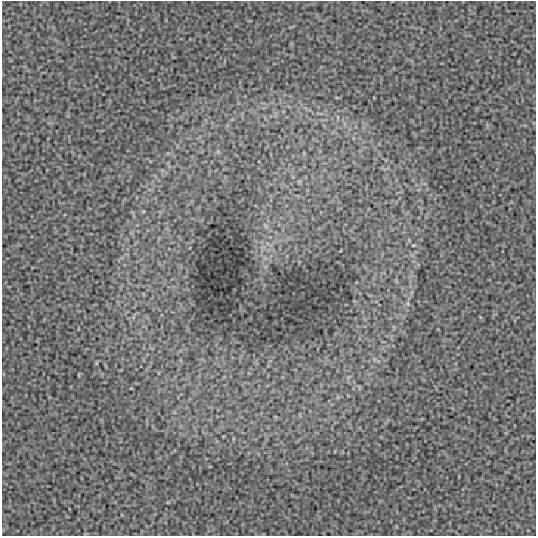}
		\label{fig:sheppLoganm1dBBruteAlign}
	}
	\subfloat[-4dB]    
	{ 
		\includegraphics[width=0.2\textwidth]{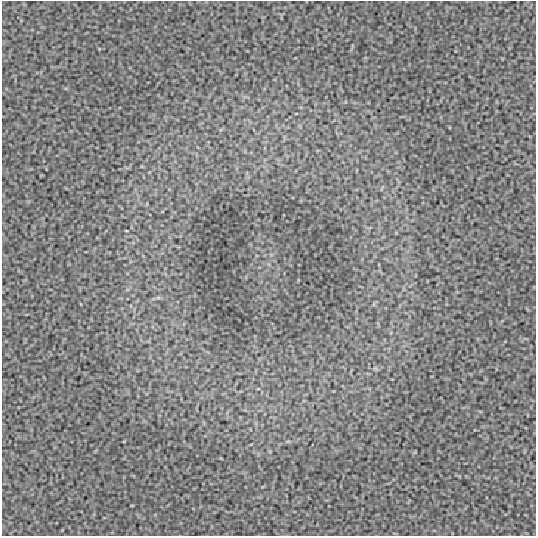}
		\label{fig:sheppLoganm2dBBruteAlign}
	}
	
	\caption{Shepp-Logan phantom reconstructed from 256 random shifted projections at various levels of noise, centered by directly aligning the nearest neighbors~$\mathcal{N}_i$ to the projection~$x_i$ in Step~3 of our method.} 
	\label{fig:sheppLoganBruteAlign}
\end{figure} 

\section{Summary and future work}\label{sec:Summary}
In this work, we generalized the diffusion maps embedding for data sets closed under the action of a compact matrix Lie group. We derived an equivariant embedding, and showed that the Euclidean distance between embedded points equals the distance between the probability densities of a pair of random walks over the orbits generated by the action of the group on the data set. Next, we derived an invariant embedding, and showed that the distance between a pair of embedded points equals the distance between the probability densities of displacements of pairs of random walks that depart from the points. We then demonstrated the utility of our framework for the problem of reconstructing a 2D image from its noisy random 1D Radon transform projections, each shifted by a random shift. 

As for future work, a natural direction is to apply $G$-invariant diffusion maps for clustering, dimensionality reduction, and alignment of data sets of images, points clouds, and volumes. Of particular interest, is the problem of class-averaging in cryoEM~\cite{classAverage}, which can be seen as a 2D analog of the random tomography problem addressed in~Section~\ref{sec:randTomography}.

\section*{Acknowledgements}
XC was supported in part by NSF DMS-2007040. ER and YS were supported by NSF-BSF award 2019733 and by the European Research Council (ERC) under the European Union's Horizon 2020 research and innovation programme (grant agreement 723991 - CRYOMATH). YS was supported also by the NIH/NIGMS Award R01GM136780-01.

\appendix
\section{Compact matrix Lie groups}
\subsection{Matrix Lie groups and actions}\label{secLieGroupAction}
In this section and the next, we review some background on compact matrix Lie groups and harmonic analysis on them. For a thorough introduction to the subject we refer the reader to~\cite{lieGroups}.

\begin{definition}
	A matrix Lie group is a smooth manifold $G$, whose points form a group of matrices. 
\end{definition}
For example, consider the special unitary group of order~$2$
\begin{equation}\label{secLieGroupAction:SU2def}
	SU(2) = \left\{\begin{pmatrix}
		z&w \\ -\overline{w}& \overline{z}
	\end{pmatrix}\; : \;  z,w\in\mathbb{C}, \; \abs{z}^2+\abs{w}^2=1\quad \right\},
\end{equation}
which consists of all $2\times2$ unitary matrices with determinant~1.
Writing~$z=x_1+ix_2$ and $w=x_3+ix_4$ we have that
\begin{equation}\label{secLieGroupAction:SU2det}
	1=	\det{\begin{pmatrix}
			z&w \\ -\overline{w}& \overline{z}
	\end{pmatrix}}= \abs{z}^2+\abs{w}^2= x_1^2+x^2_2+x_3^2+x_4^2. 
\end{equation}
Furthermore, it is readily verified that each~$z,w\in \mathbb{C}$ for which~\eqref{secLieGroupAction:SU2det} holds defines a unique element in~$SU(2)$. Hence, we conclude that~$SU(2)$ is diffeomorphic to the three-dimensional unit sphere $S^3$. Other important examples for matrix Lie groups include the group of three-dimensional rotation matrices~$SO(3)$, and the $n$-dimensional torus $\mathbb{T}^n$, which is just the group of diagonal $n\times n$ unitary matrices. 
\begin{definition}\label{secLieGroupAction:gActDef}
	The action of a group $G$ of $n\times n$ matrices on a subset $Z\subseteq \mathbb{C}^{n}$ is a map $\text{'}\cdot\text{'}:G\times Z\rightarrow Z$, defined for each $A\in G$ and $x\in Z$ by matrix multiplication on the left~$A\cdot x$. 
\end{definition}
We say that the set~$Z$ is closed under the action of $G$ or simply that $Z$ is $G$-invariant, if for every~$A\in G$ it holds that $A\cdot x\in Z$. We will assume that the data set~$Z$ was sampled from a $G$-invariant compact manifold~$\man\subset\mathbb{C}^{n}$. In other words, we assume that $A\cdot x \in \man$ for all~$x\in\man$ and~$A\in G$. We assume that~$G$ is compact as well. 

The tools we develop in this work employ Fourier series expansions over Lie groups. The expansion coefficients are obtained by integration with respect to the Haar measure, which we now define. 
\begin{definition}
	The Haar measure over a Lie group $G$ is the unique finite valued, non-negative function~$\eta(\cdot)$ over all (Borel) subsets $S\subseteq G$, such that 
	\begin{equation}\label{secLieGroupAction:leftInvar}
		\eta(A\cdot S) = \eta(S)\quad \text{ for all} \quad A\in G,
	\end{equation}
	and 
	\begin{equation}\label{secLieGroupAction:probMeasure}
		\eta(G) =1. 
	\end{equation}
\end{definition}
If the group $G$ is compact (as we assume throughout this work), then we also have right invariance (see~\cite{ChirikjianStochasticLieGroups})
\begin{equation}\label{secLieGroupAction:rightInvar}
	\eta(S\cdot A) = \eta(S)\quad \text{ for all} \quad A\in G,
\end{equation}
and furthermore, that (see~\cite{ChirikjianStochasticLieGroups})
\begin{equation}\label{secLieGroupAction:conjInvar}
	\eta(S^*) = \eta(S), \quad S\subseteq{G}, 
\end{equation}
where~$S^*=\left\{A^*\; : \; A\in S\right\}$.

As an example for a Haar integral, consider again the group~$SU(2)$ in~\eqref{secLieGroupAction:SU2def}. Any element $A\in SU(2)$ can be written using Euler angles as
\begin{equation}\label{secLieGroupAction:fundIUR}
	A(\alpha,\beta,\gamma) = 
	\begin{pmatrix}
		\cos{\frac{\beta}{2}}e^{i(\alpha+\gamma)/2} & \sin{\frac{\beta}{2}}e^{i(\alpha-\gamma)/2}\\ 
		i\sin{\frac{\beta}{2}}e^{-i(\alpha-\gamma)/2} & \cos{\frac{\beta}{2}}e^{-i(\alpha+\gamma)/2}
	\end{pmatrix},
\end{equation}
where $\alpha\in [0,2\pi), \beta \in [0,\pi)$ and $\gamma\in [-2\pi,2\pi)$.
Using~\eqref{secLieGroupAction:fundIUR}, the Haar integral of a function~$f:SU(2)\rightarrow\mathbb{C}$ can be computed by (see~\cite{nonComHarmAnalys})
\begin{equation}
	\int_{SU(2)}f(A)d\eta(A) =\frac{1}{16\pi^2}\int_0^{2\pi}\int_0^{\pi}\int_{-2\pi}^{2\pi}f(A(\alpha,\beta,\gamma))\sin\beta d\alpha d\beta d\gamma.
\end{equation}
We observe that in this case, the volume element induced by the Haar measure is given in Euler angles parametrization by~$\sin\beta d\alpha d\beta d\gamma$. 

\subsection{Harmonic analysis over matrix Lie groups}\label{secHarmAnalysis}
In this section, we give a brief introduction to Fourier series expansions of group valued functions, which arise as a consequence of the celebrated Peter-Weyl theorem~\cite{lieGroups}. The expansion of a function $f:G\rightarrow \mathbb{C}$ is obtained in terms of the elements of certain matrix valued functions, known as the irreducible unitary representations of~$G$, which we now define. 
\begin{definition}
	An $n$-dimensional unitary representation of a Lie group $G$ is matrix-valued function~$U(\cdot)$ from~$G$ into the group~U(n) of~$n\times n$ unitary matrices, such that
	\begin{equation}\label{harmAnalysisOnG:hMorphProp}
		U(A\cdot B) = U(A)\cdot U(B), \quad A,B\in G, 
	\end{equation}
	and 
	\begin{equation}\label{harmAnalysisOnG:identityMapProp}
		U(I) = I_n,
	\end{equation}
\end{definition}
where~$I\in G$ and~$I_n\in \text{U}(n)$ are the identity elements of~$G$ and~U(n), respectively. 
The homomorphism property \eqref{harmAnalysisOnG:hMorphProp} together with~\eqref{harmAnalysisOnG:identityMapProp} imply that the set~$\left\{U(A)\right\}_{A\in G}$ is also a matrix Lie group. Furthermore, by~\eqref{harmAnalysisOnG:hMorphProp} and~\eqref{harmAnalysisOnG:identityMapProp} we have that
\begin{equation}
	I_n = U(AA^*) = U(A)\cdot U(A^*),  
\end{equation} 
which implies that
\begin{equation}\label{harmAnalysisOnG:IURsUnitaryProp}
	U(A^*) = \left(U(A)\right)^*, \quad A\in G. 
\end{equation}
In other words, the matrix~$U(A^*)$ is the inverse element of~$U(A)$.
\begin{definition}
	A group representation $U(\cdot)$ is called reducible, if there exists a unitary matrix $P$, such that $PU(A)P^{-1}$ is block diagonal for all $A\in G$. An irreducible unitary representation (abbreviated IUR) is a representation that is not reducible.
\end{definition}
By the Peter-Weyl theorem~\cite{lieGroups}, there exists a countable family $\left\{U^{\alpha}\right\}$ of finite dimensional IURs of~$G$, such that the collection~$\left\{U^\alpha_{ij}(\cdot)\right\}$ of all the elements of all these IURs forms an orthogonal basis for $L^2(G)$.
This implies that any smooth function $f:G\rightarrow \mathbb{C}$ can be expanded in a series of the elements of the IURs of~$G$. 
For example, the IURs of $SU(2)$ in~\eqref{secLieGroupAction:fundIUR} are given by a sequence of matrices $\left\{U^{\ell}\right\}$, $\ell=0,1/2,1,3/2,\ldots$, where $U^{\ell}(A)$ is a $(2\ell+1)\times (2\ell+1)$ dimensional matrix for each $A\in G$ (see e.g. \cite{nonComHarmAnalys}). In particular, the matrix-valued function in~\eqref{secLieGroupAction:fundIUR} is the IUR of~$SU(2)$ that corresponds to~$\ell=1/2$.

In general, the series expansion of a function $f:G\rightarrow \mathbb{C}$ is given by 
\begin{equation}\label{secLieGroupAction:GFourierMatCoeff}
	f(A)=\sum_{\ell\in \I_G } d_\ell\cdot \text{trace}\left(\hat{f}^\ell\cdot  U^\ell(A)\right),
\end{equation}
where $\I_G$ is a countable set that enumerates the IURs of~$G$, $d_\ell$ is the dimension of the~$\ell$-th IUR, and $\hat{f}^\ell$ is the $d_\ell\times d_\ell$ matrix given by 
\begin{equation}\label{sec2:fHatDef}
	\hat{f}^\ell= \int_{G} f(A)\overline{U^\ell(A)}d\eta(A),
\end{equation}
for each $\ell\in \I_G$. Note, that the elements of the conjugate of the IUR~$U^\ell(A)$ in~\eqref{sec2:fHatDef} is defined simply by taking the conjugate of each element of~$U^\ell(A)$. 
Importantly, we have Schur's orthogonality relation (see~\cite{nonComHarmAnalys})
\begin{equation}\label{secLieGroupAction:SchurOrthogonality}
	\int_G U^\ell_{mn}(A)\overline{U^\ell_{m'n'}(A)}d\eta(A) =d_\ell^{-1}\cdot \delta_{\ell\ell'}\delta_{mm'}\delta_{nn'},
\end{equation} 
where~$\delta_{qr}$ is Kronecker's delta. 

\section{Proof of Lemma~\ref{eqvDmaps:eqvEigenVecLemma}}\label{appSec3:eqvLemma}
By assumption, there exists $B\in G$ such that $x_j = B\cdot x_i$.
First, we show that 
\begin{equation}
	\hat{W}^{\ell}_{ik} = \overline{U^\ell(B)}\cdot \hat{W}^{\ell}_{jk}, \quad k=1,\ldots,N,
\end{equation}
where $\hat{W}^{\ell}_{ik}$ is defined in~\eqref{sec2:hat{W}Def}. Indeed, by using~\eqref{GinvDef:Wdef}, we have
\begin{align}
	\hat{W}^{\ell}_{ik} = \int_G W_{ik}(I,A)\overline{U^\ell(A)}d\eta(A) &= \int_G e^{-\norm{x_i-A\cdot x_k}^2/\epsilon}\overline{U^\ell(A)}d\eta(A) \nonumber \\		
	&= \int_G e^{-\norm{x_j-BA\cdot x_k}^2/\epsilon}\overline{U^\ell(A)}d\eta(A\nonumber).
\end{align}
Making the change of variables $C=BA$, we get 
\begin{align} 
	\int_G e^{-\norm{x_j-BA\cdot x_k}^2/\epsilon}\overline{U^\ell(A)}d\eta(A\nonumber)&=\int_G e^{-\norm{x_j-C\cdot x_k}^2/\epsilon}\overline{U^\ell(B^*C)}d\eta(B^*C)\nonumber \\
	&= \overline{U^\ell(B^*)}\int_G e^{-\norm{x_j-C\cdot x_k}^2/\epsilon}\overline{U^\ell(C)}d\eta(B^*C)\label{WjkEquivEq2}\\
	&= \overline{U^\ell(B^*)}\int_G e^{-\norm{x_j-C\cdot x_k}^2/\epsilon}\overline{U^\ell(C)}d\eta(C) \label{WjkEquivEq3} \\ 
	&= \overline{U^\ell(B^*)}\cdot \hat{W}^{\ell}_{jk},		
\end{align}
where we used the homomorphism property \eqref{harmAnalysisOnG:hMorphProp} in passing to~\eqref{WjkEquivEq2}, and the translation invariance property \eqref{secLieGroupAction:leftInvar} of the Haar measure in passing to~\eqref{WjkEquivEq3}.
Next, we observe that for all $i\in [N]$, we have 
\begin{align}
	\hat{W}^{(\ell)}  v = \lambda v &\iff \sum_{k=1}^N \hat{W}^{\ell}_{ik}\cdot e^k(v) = \lambda e^i(v) \iff \overline{U^\ell(B^*)}\cdot \sum_{k=1}^{N}\hat{W}^{\ell}_{jk}e^k(v) = \lambda e^i(v) \nonumber \\ 
	&\iff \sum_{k=1}^{N}\hat{W}^{\ell}_{jk}e^k(v) = \overline{U^\ell(B)} \lambda e^i(v)\label{ekEigVec},
\end{align} 
where we used~\eqref{harmAnalysisOnG:IURsUnitaryProp} in passing to~\eqref{ekEigVec}, 
and thus, using that~$\lambda>0$, we get that
\begin{align}
	\lambda e^j(v) = \sum_{k=1}^{N}\hat{W}^{\ell}_{jk}e^k(v) = \lambda \overline{U^\ell(B)}\cdot  e^i(v) \iff e^j(v) = \overline{U^\ell(B)}\cdot  e^i(v).
\end{align}

\section{Proofs of results from Section~\ref{sec:GEqvDmaps}}
\subsection{Proof of Theorem \ref{eqvDmaps:diffDistEigProp}}\label{appendix:diffMapsEigPropPrf}
First, let us express~\eqref{eqvDmaps:eqvDistExpanded} in terms of the eigenvectors and eigenvalues of~$P_{op}^t$ of~\eqref{eqvDmaps:probOperator}.
We begin by diagonalizing the operator~$S_{op}:\mathcal{H}\rightarrow \mathcal{H}$ defined by
\begin{equation}\label{eqvDmaps:symmetricOp}
	S_{op}\left\{f\right\}(i,A) = \sum_{j=1}^{N}\int_G S((i,A),(j,B))f_j(B)d\eta(B), \quad f\in \mathcal{H},
\end{equation}
where
\begin{equation}\label{eqvDmaps:symOpDef}
	S((i,A),(j,B)) = \frac{\sqrt{D_{ii}}}{\sqrt{D_{jj}}}P((i,A),(j,B)) = \frac{W_{ij}(A,B)}{\sqrt{D_{ii}}\sqrt{D_{jj}}},
\end{equation}
where~$P$ is defined in~\eqref{secEqDiffmaps:probOperatorDef} and~$W_{ij}$ is defined in~\eqref{GinvDef:Wdef}.
By \eqref{eqvDmaps:symOpDef} and \eqref{GinvDef:Wdef}, the function~$S((i,A)(j,B))$ is symmetric, and thus, the operator $S_{op}$ is also symmetric. Now, consider the matrix $S^{(\ell)} = (D^{(\ell)})^{-1/2}\hat{W}^{(\ell)} (D^{(\ell)})^{-1/2}$, related to the matrix $K^{(\ell)}$ in \eqref{secGGL:fourierMatNorm} by 
\begin{equation}\label{eqvDmaps:StildeDef}
	S^{(\ell)} = I-(D^{(\ell)})^{1/2}K^{(\ell)} (D^{(\ell)})^{-1/2},
\end{equation}
where~$D^{(\ell)}$ and~$W^{(\ell)}$ were defined in~\eqref{secGGL:blockFourierMat} and~\eqref{secGGL:DellDef}, respectively. 
By \eqref{eqvDmaps:StildeDef} we have that $K^{(\ell)} = I-(D^{(\ell)})^{-1/2}S^{(\ell)} (D^{(\ell)})^{1/2}$, and so,
if~$\tilde{v}_{n,\ell}$ is an eigenvector of~$K^{(\ell)}$ that corresponds to the eigenvalue~$\tilde{\lambda}_{n,\ell}$ (see Theorem~\ref{secGGL:GGLdecomp}), then~\eqref{eqvDmaps:ProbOpEvals} implies that
\begin{align}\label{eqvDmaps:v2Dv}
	&\tilde{\lambda}_{n,\ell}\cdot \tilde{v}_{n,\ell} = K^{(\ell)}\cdot \tilde{v}_{n,\ell} = (I-(D^{(\ell)})^{-1/2}S^{(\ell)} (D^{(\ell)})^{1/2})\cdot \tilde{v}_{n,\ell}\iff\nonumber \\ 
	&(D^{(\ell)})^{-1/2}S^{(\ell)} (D^{(\ell)})^{1/2}\cdot \tilde{v}_{n,\ell} =\lambda_{n,\ell} \cdot\tilde{v}_{n,\ell}  \iff S^{(\ell)} ((D^{(\ell)})^{1/2} \tilde{v}_{n,\ell}) = 	\lambda_{n,\ell}\cdot  (D^{(\ell)})^{1/2} \tilde{v}_{n,\ell}.
\end{align}
Thus, we see that $\tilde{v}_{n,\ell}$ is an eigenvector of~$K^{(\ell)}$ if and only if~$(D^{(\ell)})^{1/2}\tilde{v}$ is an eigenvector of~$S^{(\ell)}$.

As we explained in Section~\ref{sec:GEqvDmaps}, the eigenfunctions~$\{\Phi^{(p)}_{\ell,m,n}\}$ in~\eqref{eqvDmaps:ProbEvecs} of~$P_{op}$ from~\eqref{secEqDiffmaps:probOperatorDef} are identical to those of~$\tilde{L}$ in~\eqref{GGL:normalizedGGLDef}, which can be computed by using the eigenvectors~$v^{(p)}_{n,\ell} = \tilde{v}_{n,\ell}$ of~$K^{(\ell)}$ via~\eqref{secGGL:eigenFuncs}. 
By~\eqref{eqvDmaps:ProbEvecs} and~\eqref{secGGL:eigenFuncs}, the eigenfunctions of~$P_{op}$ can be expressed as
\begin{equation}\label{eqvDmaps:ProbEvecsExplicit}
	\Phi^{(p)}_{\ell,m,n}(i,A) = \sqrt{d_\ell}\cdot\overline{U^\ell_{m,\cdot}(A)}\cdot e^i(v^{(p)}_{n,\ell}).
\end{equation}
We further denote the eigenvectors of~$S^{(\ell)}$ from~\eqref{eqvDmaps:StildeDef} by~$\left\{	v^{(s)}_{n,\ell}\right\}$, and by~\eqref{eqvDmaps:v2Dv}, we obtain that
\begin{equation}\label{eqvDmaps:S2PevecRelation}
	v^{(s)}_{n,\ell} = (D^{(\ell)})^{1/2} v^{(p)}_{n,\ell}, \quad n\in[N], \quad \ell\in \I_G. 
\end{equation}
The following result relates the eigendecomposition of~$S_{op}$ in~\eqref{eqvDmaps:symmetricOp} with that of~$P_{op}$. 
\begin{lemma}\label{eqvDmaps:symEigenProp}
	The functions $\Phi^{(s)}_{\ell,m,n}:[N]\times G\rightarrow \mathbb{C}$ defined by
	\begin{equation}\label{eqvDmaps:eigenFuncsTilde}
		\Phi^{(s)}_{\ell,m,n}(i,A) = \sqrt{d_\ell}\cdot\overline{U^\ell_{m,\cdot}(A)}\cdot e^i(v^{(s)}_{n,\ell}),
	\end{equation}
	for $\ell \in \I_G$, $m\in \left\{1,\ldots,d_\ell\right\}$, and~$n\in [N]$, are eigenfunctions of~$S_{op}$ which are complete in~$\mathcal{H}$, and are orthonormal with respect to the inner product
	\begin{equation}
		\dprod{f}{g}_{\Hspace} = \sum_{k=1}^{N}\int_G f(k,C)\cdot \overline{g(k,C)}d\eta(C).
	\end{equation}
	Furthermore, each eigenfunction $\Phi^{(s)}_{\ell,m,n}$ corresponds to the eigenvalue $\lambda_{n,\ell}$ of~$P_{op}$ from~\eqref{secEqDiffmaps:probOperatorDef}, and it is related to the eigenfunction $\Phi^{(p)}_{\ell,m,n}$ in~\eqref{eqvDmaps:ProbEvecs} (of~$P_{op}$) by
	\begin{equation}\label{eqvDmaps:Phi2PhiTilde}
		\Phi^{(s)}_{\ell,m,n}= D^{1/2} \Phi^{(p)}_{\ell,m,n},
	\end{equation}
	where~$D$ is defined in~\eqref{GinvDef:DiiDef}.
\end{lemma}
\begin{proof}
		By \eqref{secGGL:parVecNotation}, \eqref{secGGL:DellDef} and~\eqref{eqvDmaps:S2PevecRelation}, for all $j,n\in [N]$ we have that
	\begin{equation}\label{ejTilde2ej}
		e^j(v^{(s)}_{n,\ell}) = e^j((D^{(\ell)})^{1/2}v^{(p)}_{n,\ell}) = \sqrt{D_{jj}}\cdot e^j(v^{(p)}_{n,\ell}).
	\end{equation}
	Combining the latter with~\eqref{eqvDmaps:ProbEvecsExplicit}, \eqref{eqvDmaps:eigenFuncsTilde} and~\eqref{GinvDef:DMatAction}, we obtain~\eqref{eqvDmaps:Phi2PhiTilde}. 
	Hence, by \eqref{eqvDmaps:symmetricOp}, \eqref{eqvDmaps:symOpDef} and~\eqref{eqvDmaps:eigenFuncsTilde} we have that
	\begin{align}
		S_{op}\left\{\Phi^{(s)}_{\ell,m,n}\right\}(i,A) &= \sum_{j=1}^{N}\int_G \frac{W_{ij}(A,B)}{\sqrt{D_{ii}}\sqrt{D_{jj}}}\sqrt{d_\ell}\cdot \overline{U^\ell_{m,\cdot}(A)}\cdot e^j(v^{(s)}_{n,\ell})d\eta(B),\\
		&=\sqrt{D_{ii}}\cdot \sum_{j=1}^{N}\int_G \frac{W_{ij}(A,B)}{D_{ii}}\sqrt{d_\ell}\cdot \overline{U^\ell_{m,\cdot}(A)}\cdot e^j(v^{(p)}_{n,\ell})d\eta(B)\label{eigenProbP2SRelation:eq2} \\
		&= \sqrt{D_{ii}}\cdot P\left\{\Phi^{(p)}_{\ell,m,n}\right\}(i,A)=\sqrt{D_{ii}}\cdot \lambda_{n,\ell}\cdot \Phi^{(p)}_{\ell,m,n}(i,A)  \label{eigenProbP2SRelation:eq3}\\
		&= \lambda_{n,\ell}\cdot \Phi^{(s)}_{\ell,m,n}(i,A)\label{eqvDmaps:eigenProbP2SRelation},
	\end{align}
	where in~\eqref{eigenProbP2SRelation:eq2} we used~\eqref{ejTilde2ej}, then in~\eqref{eigenProbP2SRelation:eq3} we used~\eqref{eqvDmaps:ProbEvecsExplicit}, and in~\eqref{eqvDmaps:eigenProbP2SRelation} we used~\eqref{eqvDmaps:Phi2PhiTilde} and~\eqref{GinvDef:DMatAction}. The completeness of the eigenfunctions in~\eqref{eqvDmaps:eigenFuncsTilde} follows by combining~\eqref{eqvDmaps:ProbEvecs} and~\eqref{eqvDmaps:Phi2PhiTilde} with the completeness of the functions~$\left\{\tilde{\Phi}_{\ell,m,n}\right\}$ of~\eqref{secGGL:eigenFuncs} in~$\Hspace$, which is implied by Theorem~\ref{secGGL:GGLdecomp}. Finally, to see that~$\left\{\Phi^{(s)}_{\ell,m,n}\right\}$ are orthonormal, observe that by~\eqref{eqvDmaps:eigenFuncsTilde} we have
	\begin{align}
		\dprod{\Phi^{(s)}_{\ell,m,n}}{\Phi^{(s)}_{\ell',m',n'}}_{\Hspace} &= \sqrt{d_\ell}\cdot\sqrt{d_{\ell'}}\cdot\sum_{j=1}^N\int_G \overline{U^\ell_{m,\cdot}(A)}\cdot e^j(v^{(s)}_{n,\ell})\cdot U^\ell_{m',\cdot}(A)\overline{\cdot e^j(v_{n',\ell'})} d\eta(A)\nonumber \\
		& =  \sqrt{d_\ell\cdot d_{\ell'}}\cdot\sum_{j=1}^N  (e^j(v^{(s)}_{n,\ell}))^T\cdot\left(\int_G(\overline{U^\ell_{m,\cdot}(A)})^T\cdot U^\ell_{m',\cdot}(A)d\eta(A)\right)\overline{\cdot e^j(v_{n',\ell'})} \nonumber \\
		& =  \sqrt{d_\ell\cdot d_{\ell'}}\cdot\sum_{j=1}^N  (e^j(v^{(s)}_{n,\ell}))^T\cdot \frac{1}{d_\ell}\cdot \delta_{\ell\ell'}\delta_{mm'}I_{\ell\times\ell} \cdot\overline{ e^j(v_{n',\ell'})}\label{phiOrthonormalEq3} \\
		&=\delta_{\ell\ell'}\delta_{mm'}\dprod{v^{(s)}_{n,\ell}}{v^{(s)}_{n',\ell}} =  \delta_{\ell\ell'}\delta_{mm'}\delta_{nn'}\label{phiOrthonormalEq4},
	\end{align}
	where in passing to~\eqref{phiOrthonormalEq3} we used~\eqref{secLieGroupAction:SchurOrthogonality},
	and in passing to~\eqref{phiOrthonormalEq4} we used \eqref{secGGL:parVecNotation}, and that~$\left\{v^{(s)}_{n,\ell}\right\}$ are orthonormal eigenvectors of the symmetric matrices $S^{(\ell)}$ in~\eqref{eqvDmaps:StildeDef}.
	
	Now, consider the operator~$S_{op}^t:\Hspace \rightarrow \Hspace$ defined similarly to~\eqref{eqvDmaps:probOperator} and~\eqref{eqvDmaps:PtKernelDef} by
	\begin{equation}\label{eqvDmaps:symOperator}
		\left\{S_{op}^t f\right\}(i,A) = \sum_{j=1}^{N}\int_G S^t((i,A),(j,B))f_j(B)d\eta(B), \quad f\in \mathcal{H},
	\end{equation}
	where for each $(i,A),(j,B) \in [N]\times G$ we define~$S^1_{op} = S_{op}$ (see~\eqref{eqvDmaps:symmetricOp}), and 
	\begin{equation}\label{eqvDmaps:symPowtKernelDef}
		S^t((i,A),(j,B)) \triangleq \sum_{k=1}^{N}\int_G S^{t-1}((i,A),(k,C))\cdot S((k,C),(j,B))d\eta(C), \quad t=2,3,\ldots
	\end{equation}
	Applying Mercer's theorem to~$S_{op}^t$ in conjunction with Lemma~\ref{eqvDmaps:symEigenProp}, we obtain that
	\begin{align}\label{eqvDmaps:symmetricOp_ij}
		S^t((i,A),(j,B)) &= \sum_{\ell\in \I_G} \sum_{m=1}^{\ell} \sum_{n=1}^N \lambda_{n,\ell}^{t}\Phi^{(s)}_{\ell,m,n}(i,A) \cdot \overline{\Phi^{(s)}_{\ell,m,n}(j,B)},
	\end{align}
	for all~$(i,A),(j,B)\in [N]\times G$. 
	By induction on~$t\in \mathbb{N}$, combined with~\eqref{eqvDmaps:symOpDef} and~\eqref{eqvDmaps:symmetricOp_ij}, we obtain  for all~$t\in \mathbb{N}$ that
	\begin{equation}\label{eqvDmaps:symOpDefPowt}
		S^t((i,A),(j,B)) = \frac{\sqrt{D_{ii}}}{\sqrt{D_{jj}}}P^t((i,A),(j,B)),\quad (i,A),(j,B)\in [N]\times G. 
	\end{equation}
	Therefore, by using \eqref{eqvDmaps:symmetricOp_ij}, \eqref{eqvDmaps:symOpDefPowt} and~\eqref{GinvDef:DMatAction}, we can write~$P^t_{i,A}(k,C)$ from~\eqref{eqvDmaps:PiAtcdotcdotDef} as
	\begin{align}\label{eqvDmaps:pijExpansion}
		P^t_{i,A}(k,C)&=P^t((i,A),(k,C)) =\sum_{\ell\in \I_G} \sum_{m=1}^{\ell} \sum_{n=1}^N \lambda_{n,\ell}^{t}\left(\frac{\Phi^{(s)}_{\ell,m,n}(i,A)}{\sqrt{D_{ii}}}\right) \cdot \overline{\sqrt{D_{jj}}\cdot \Phi^{(s)}_{\ell,m,n}(k,C)}\nonumber \\
		&=\sum_{\ell\in \I_G} \sum_{m=1}^{\ell} \sum_{n=1}^N \lambda_{n,\ell}^{t}\left\{D^{-1/2}\cdot \Phi^{(s)}_{\ell,m,n}\right\}(i,A)\cdot \overline{\left\{D^{1/2}\cdot \Phi^{(s)}_{\ell,m,n}\right\}(k,C)}.
	\end{align}
	By \eqref{eqvDmaps:Phi2PhiTilde}, each function $D^{-1/2}\Phi^{(s)}_{\ell,m,n}=\Phi^{(p)}_{\ell,m,n}$ is an eigenfunction of~$P_{op}$, that corresponds to the eigenvalue~$\lambda_{n,\ell}$. By using~\eqref{eqvDmaps:pijExpansion}, it is straightforward to verify that for each~$t\in\mathbb{N}$ it is also
	an eigenfunction of~$P_{op}^t$, that corresponds to the eigenvalue~$\lambda_{n,\ell}^t$. 
\end{proof}
	
	Now, to write~$D_{p,t}$ from~\eqref{eqvDmaps:eqvDist} in terms of the functions~$\{\Phi^{(p)}_{\ell,m,n}\}$, we employ the following observation. 
	We may view~\eqref{eqvDmaps:pijExpansion} as the expansion of~$P^t_{i,A}$ in terms of the functions~$\{D^{1/2}\cdot \Phi^{(s)}_{\ell,m,n}\}$, which are orthogonal with respect to the inner product on~$\mathcal{H}$ defined by
	\begin{equation}\label{eqvDmaps:weightedHspaceDotProdDef}
		\dprod{f}{g}_{\mathcal{H},d\eta/D} = \sum_{k=1}^{N}\int_{G} f(k,C)\cdot \overline{g(k,C)}\frac{d\eta(C)}{D_{kk}},
	\end{equation}
	and with expansion coefficients given by $\{D^{-1/2}\cdot \Phi^{(s)}_{\ell,m,n}\}(i,A)=\Phi^{(p)}_{\ell,m,n}(i,A)$ for all~$\ell,m$ and~$n$. Thus, by Parseval's identity, we obtain that~$D_{p,t}((i,A),(j,B))$ from~\eqref{eqvDmaps:eqvDist} is given by
	\begin{align}\label{eqvDmaps:diffDist1}
		\norm{P_{i,A}^t-P_{j,B}^t}_{\mathcal{H},d\eta/D} =\left(\sum_{\ell\in \I_G}\sum_{m=1}^{d_\ell} \sum_{n=1}^N  \lambda^{2t}_{n,\ell} \cdot  \left|\Phi^{(p)}_{\ell,m,n}(i,A)-\Phi^{(p)}_{\ell,m,n}(j,B)\right|^2\right)^{\frac{1}{2}},
	\end{align}
as claimed.

\subsection{Proof of Proposition~\ref{eqvDmaps:truncEmbedEqvProp}}\label{appendix:truncEmbedEqvPropPrf}
First, by \eqref{secGGL:eigenFuncs} and~\eqref{eqvDmaps:ProbEvecs}, the elements of $\eqref{eqvDmaps:eqvEmbedding}$ that correspond to fixed values of~$\ell$ and~$n$ are given by the $d_\ell$-dimensional vector
\begin{equation}\label{eqvDmaps:eqvEmbeddingNL}
	\begin{pmatrix}
		\Phi^{(p)}_{\ell,1,n} \\ \vdots \\ \Phi^{(p)}_{\ell,d_\ell,n}
	\end{pmatrix} = \sqrt{d_\ell}\cdot \overline{U^\ell\left(A\right)}\cdot e^i(v^{(p)}_{n,\ell}). 
\end{equation}
Hence, we may write \eqref{eqvDmaps:eqvEmbedding} as 
\begin{equation}\label{eqvDmaps:eqvEmbeddingVec}
	\Phi^{(p)}_t(i,A) = \left(\lambda_{n,\ell}^t\sqrt{d_\ell}\cdot \overline{U^\ell\left(A\right)}\cdot e^i(v^{(p)}_{n,\ell})\right)_{n=1,\ell\in \I_G}^{N},
\end{equation}
where we perceive $\Phi^{(p)}_t(i,A)$ as an infinite dimensional vector, obtained by concatenating all the $d_\ell$-dimensional vectors in \eqref{eqvDmaps:eqvEmbeddingNL}. 
Next, by Lemma~\ref{eqvDmaps:eqvEigenVecLemma} and~\eqref{harmAnalysisOnG:identityMapProp}, we have that
\begin{equation}
		\Phi^{(p)}_t(j,I) = \left(\lambda_{n,\ell}^t\sqrt{d_\ell}\cdot \overline{U^\ell\left(I\right)}\cdot e^j(v^{(p)}_{n,\ell})\right)_{n=1,\ell\in \I_G}^{N} = \left(\lambda_{n,\ell}^t\sqrt{d_\ell}\cdot \overline{U^\ell\left(B\right)}\cdot e^i(v^{(p)}_{n,\ell})\right)_{n=1,\ell\in \I_G}^{N}.
\end{equation}
Furthermore, since each block~$U^\ell(B)$ on the diagonal of~$U(B)$ is irreducible, then so is~$U(B)$, which implies that the function~$U(\cdot)$ defined by~$B\mapsto U(B)$ is an IUR of~$G$.

\section{Proofs for Section~\ref{sec:invDmaps}}
\subsection{Proof of proposition~\ref{invDmaps:eqvInvDistanceProp}}\label{appendix:eqvInvDistancePropPrf}
We first require the following result. 
\begin{lemma}\label{eqvDmaps:eqvEmbeddingProp}
	Suppose that $x_j=B\cdot x_i$ for some $B\in G$. Then,  we have that
	\begin{equation}
		\Phi^{(p)}_t(j,A) = \Phi^{(p)}_t(i,AB).
	\end{equation}
	In particular, if~$A=I$, then we get that $\Phi^{(p)}_t(j,I) = \Phi^{(p)}_t(i,B)$. 
\end{lemma}
\begin{proof}
	Suppose that $x_j = Bx_i$. Then, by Lemma \ref{eqvDmaps:eqvEigenVecLemma} and~\eqref{harmAnalysisOnG:IURsUnitaryProp} we have
	\begin{align*}
		\left(\Phi^{(p)}_t(j,A)\right)_{\ell,m,n} &=\lambda_{n,\ell}^t\sqrt{d_\ell}\left(e^j(\tilde{v}_{n,\ell})\right)^T\cdot U^\ell_{\cdot,m}(A^*) =\lambda_{n,\ell}^t\sqrt{d} \left(\overline{U^\ell\left(B\right)}e^i(\tilde{v}_{n,\ell})\right)^T\cdot U^\ell_{\cdot,m}(A^*)\nonumber \\
		&= \lambda_{n,\ell}^t\sqrt{d_\ell}\left(e^i(\tilde{v}_{n,\ell})\right)^TU^\ell\left(B^*\right)\cdot U^\ell_{\cdot,m}(A^*)\nonumber \\ 
		&= \lambda_{n,\ell}^t\sqrt{d_\ell}\left(e^i(\tilde{v}_{n,\ell})\right)^T U^\ell_{\cdot,m}(B^*A^*)
		=\left(\Phi^{(p)}_t(i,AB)\right)_{\ell,m,n}.
	\end{align*}
\end{proof}

Next, we observe that by \eqref{eqvDmaps:prodDiffDist} and \eqref{eqvDmaps:eqvEmbeddingVec}, for any $A,B\in G$ we have that 
\begin{equation}\label{eqvDmaps:indentityForMinDist}
	\norm{P^t_{i,A}-P^t_{j,B}}^2_{\mathcal{H},d\eta/D} = \sum_{\ell\in \I_G}\sum_{n=1}^{N} \lambda_{\ell,n}^{2t}\cdot d_\ell\cdot  \norm{\overline{U^{\ell}(A)}e^i(\tilde{v}_{n,\ell})-\overline{U^{\ell}(B)}e^j(\tilde{v}_{n,\ell})}^2.
\end{equation}
By using that~$U^\ell(A)$ is unitary for all~$A\in G$, and by the homomorphism property~\eqref{harmAnalysisOnG:hMorphProp}, equation~\eqref{eqvDmaps:indentityForMinDist} implies that
\begin{equation}
		\norm{P^t_{i,A}-P^t_{j,B}}^2_{\mathcal{H},d\eta/D} = \norm{P^t_{i,I}-P^t_{j,A^*B}}^2_{\mathcal{H},d\eta/D},
\end{equation}
and since the map~$B \mapsto A^*B$ is a homeomorphism from~$G$ onto itself, we get that 
\begin{equation}\label{eqvDmaps:minDistOneDimOpt}
\min_{A,B\in G}\norm{P^t_{i,A}-P^t_{j,B}}^2_{\mathcal{H},d\eta/D} = \min_{Q\in G}\norm{P^t_{i,I}-P^t_{j,Q}}^2_{\mathcal{H},d\eta/D}. 
\end{equation}
Therefore, by~\eqref{eqvDmaps:minEqvDist}, \eqref{eqvDmaps:prodDiffDist}, Lemma \ref{eqvDmaps:eqvEmbeddingProp}, \eqref{eqvDmaps:minDistOneDimOpt},  and~\eqref{eqvDmaps:indentityForMinDist} we have that
\begin{align}
	M_{p,t}^2(k,r) &= \min_{Q\in G}\norm{P^t_{k,I}-P^t_{r,Q}}^2_{\mathcal{H},d\eta/D}  \nonumber \\
	& = \min_{Q\in G} \norm{\Phi^{(p)}_t(k,I)-\Phi^{(p)}_t(r,Q)}^2_{\ell^2} = \min_{Q\in G} \norm{\Phi^{(p)}_t(i,A)-\Phi^{(p)}_t(j,QB)}^2_{\ell^2}\nonumber  \\		
	&=\min_{Q\in G}\sum_{\ell\in\I_G}\sum_{n=1}^{N} \lambda_{\ell,n}^{2t}\cdot d_\ell\cdot  \norm{\overline{U^{\ell}(A)}e^i(\tilde{v}_{n,\ell})-\overline{U^{\ell}(QB)}e^j(\tilde{v}_{n,\ell})}_{\mathbb{C}^{d_\ell}}^2 \nonumber \\
	&=\min_{Q\in G}\sum_{\ell\in\I_G}\sum_{n=1}^{N} \lambda_{\ell,n}^{2t}\cdot d_\ell\cdot  \norm{e^i(\tilde{v}_{n,\ell})-\overline{U^{\ell}(A^*QB)}e^j(\tilde{v}_{n,\ell})}_{\mathbb{C}^{d_\ell}}^2 \label{eqvDmaps:MtPrfeq5} \\
	&=\min_{Q\in G}\sum_{\ell\in\I_G}\sum_{n=1}^{N} \lambda_{\ell,n}^{2t}\cdot d_\ell\cdot  \norm{e^i(\tilde{v}_{n,\ell})-\overline{U^{\ell}(Q)}e^j(\tilde{v}_{n,\ell})}_{\mathbb{C}^{d_\ell}}^2 \label{eqvDmaps:MtPrfeq6} \\
	&=\min_{Q\in G} \norm{P^t_{i,I}-P^t_{j,Q}}^2_{\mathcal{H},d\eta/D}  = M_{p,t}^2(i,j),
\end{align}
where in passing to~\eqref{eqvDmaps:MtPrfeq5} we used the fact that the 2-norm is invariant to unitary transformations,
and in passing to~\eqref{eqvDmaps:MtPrfeq6} we used the fact that the map $Q\mapsto A^*QB $ is a homeomorphism of~$G$ onto itself. 

\subsection{Proof of Proposition~\ref{eqvDmaps:eqvActLemma}}\label{eqvDmaps:eqvActLemmaPrf}
We begin with the following auxiliary result. 
\begin{lemma}\label{eqvDmaps:eqvActLemma1}
	For any $Q\in G$, $k\in [N]$, and $t\in\mathbb{N}$ we have that
	\begin{equation}
		P^t_{i,I}(k,Q^*C) = P^t_{i,Q}(k,C), \quad C\in G. 
	\end{equation}
\end{lemma}
\begin{proof}
	By~\eqref{GinvDef:WijShiftInvar}, \eqref{eqvDmaps:PopExplicitDef}, \eqref{eqvDmaps:PtKernelDef} and~\eqref{eqvDmaps:PiAtcdotcdotDef}, for $t=1$ we have that
	\begin{align}
		P^1_{i,Q}(k,C)&= P^1((i,Q),(k,C)) = P((i,Q),(k,C)) = \frac{W_{ik}(Q,C)}{D_{ii}} = \frac{W_{ik}(I,Q^*C)}{D_{ii}}\nonumber \\ &= P((i,I),(k,Q^*C)) = P_{i,I}(k,Q^*C) = P^1_{i,I}(k,Q^*C).
	\end{align}
	For~$t=2,3,\ldots$, the proof follows from~\eqref{eqvDmaps:PtKernelDef} by induction on~$t$. 
\end{proof}
Now, suppose that~$x_j=Q\cdot x_i$ for some~$Q\in G$. Then, by~\eqref{GinvDef:Wdef}, \eqref{eqvDmaps:PopExplicitDef}, and~\eqref{eqvDmaps:PtKernelDef}, we have that
\begin{equation}\label{eqvDmaps:probDensityRelation1}
	P^t_{j,I}(k,C) =P^t_{i,Q}(k,C), \quad (k,C)\in [N]\times G. 
\end{equation}
Hence, by \eqref{eqvDmaps:probDensityRelation1}, Lemma~\ref{eqvDmaps:eqvActLemma1}, and~\eqref{eqvDmaps:fTranslationDef}, for any~$t\in\mathbb{N}$ we get that
\begin{equation}
	P^t_{j,I}(k,C) =P^t_{i,Q}(k,C)= P^t_{i,I}(k,Q^*C)=\left\{Q\circ P^t_{i,I}\right\}(k,C), \quad (k,C)\in [N]\times G,
\end{equation}
as claimed. 

\subsection{Proof of Proposition~\ref{invDmaps:invMapInvarianceProp}}\label{appendix:invMapInvariancePropPrf}
		By~\eqref{invDmaps:coupledRandomWalkDensity1}, \eqref{eqvDmaps:probDensityRelation}, and Lemma~\ref{eqvDmaps:eqvActLemma},  for all $k,r\in [N]$ and $R\in G$ we have that
		\begin{align}
			\left\{P_{j,I}^t \gconv P_{j,I}^t\right\}(k,r,R) &=\int_G P_{j,I}^t(k,C)\cdot P_{j,I}^t(r,CR)d\eta(C)\nonumber\\
			&=\int_G P_{i,Q}^t(k,C)\cdot P_{i,Q}^t(r,CR)d\eta(C)\label{crossCorrInvarPrfEq2}\\
			&= \int_G P_{i,I}^t(k,Q^*C)\cdot P_{i,I}^t(r,Q^*CR)d\eta(C) \label{crossCorrInvarPrfEq3},
		\end{align}
		where we used that~$x_j = Q\cdot x_i$ in passing to~\eqref{crossCorrInvarPrfEq2}. 
		Applying the change of variables~$\tilde{C} = Q^*C$, we obtain that
		\begin{align}
			\left\{P_{j,I}^t \gconv P_{j,I}^t\right\}(k,r,R) &= \int_G P_{i,I}^t(k,\tilde{C})\cdot P_{i,I}^t(r,\tilde{C}R)d\eta(Q\tilde{C})\label{crossCorrInvarPrfEq4} \\
			&= \int_G P_{i,I}^t(k,\tilde{C})\cdot P_{i,I}^t(r,\tilde{C}R)d\eta(\tilde{C})\label{crossCorrInvarPrfEq5} \\ &=\left\{P_{i,I}^t \gconv P_{i,I}^t\right\}(k,r,R),
		\end{align}
		where we used the translation-invariance property of the Haar measure~\eqref{secLieGroupAction:leftInvar} in passing to~\eqref{crossCorrInvarPrfEq5}.

\subsection{Proof of Theorem~\ref{invDaps:distEquivalenceThrm}}\label{appSec4:invDiffDistPrf}
For the first assertion, by \eqref{invDmaps:coupledRandomWalkDensity} and~\eqref{eqvDmaps:PtIsAProbDensity}, we have that the left hand side of~\eqref{invDmaps:jointDensEq2} equals
\begin{align}\label{invDmaps:probSumsToOne}
	&\sum_{k,r=1}^{N} \int_G \left\{P_{i,I}^t(k,\cdot)\star P_{i,I}^t(r,\cdot)\right\}(R)d\eta(R)=\nonumber \\
	&\sum_{k,r=1}^{N} \int_G \left(\int_G P_{i,I}^t(k,C)\cdot P_{i,I}^t(r,CR)d\eta(C)\right)d\eta(R)=\nonumber \\
	&\sum_{k=1}^{N}\int_G  P_{i,I}^t(k,C)\cdot\left(\sum_{r=1}^{N} \int_GP_{i,I}^t(r,CR)d\eta(R)\right)d\eta(C).
\end{align}
Then, by using the change of variables $\tilde{R} = CR$ combined with the translation invariance property~\eqref{secLieGroupAction:leftInvar}, we obtain that~\eqref{invDmaps:probSumsToOne} equals
\begin{align}
	&\sum_{k=1}^{N}\int_G  P_{i,I}^t(k,C)\cdot\left(\sum_{r=1}^{N} \int_GP_{i,r}^t(I,\tilde{R})d\eta(C^*\tilde{R})\right)d\eta(C)=\nonumber \\
	&\sum_{k=1}^{N}\int_G  P_{i,I}^t(k,C)\cdot\left(\sum_{r=1}^{N} \int_GP_{i,I}^t(r,\tilde{R})d\eta(\tilde{R})\right)d\eta(C)= \sum_{k=1}^{N}\int_G  P_{i,I}^t(k,C)d\eta(C)=1.
\end{align}
Furthermore, by~\eqref{invDmaps:coupledRandomWalkDensity1}, \eqref{eqvDmaps:PiAtcdotcdotDef}, and the fact that~$P^t((i,A),(j,B))$ from~\eqref{eqvDmaps:PtKernelDef} is non-negative, we have that~$P_{i,I}^t\gconv P_{i,I}^t\geq 0$.

For the second assertion, we begin with a technical result. Let~`$\otimes$' denote the Kronecker product defined for any~$f,g\in \Hspace$ by
\begin{equation}\label{kronProdDef}
	\left\{f\otimes g\right\}((i,A),(j,B)) = f(i,A)\cdot g(j,B), \quad (i,A),(j,B)\in [N]\times G. 
\end{equation}
\begin{lemma}\label{prodKernel2CrossCorrLemma}
	For fixed~$k$ and~$r$, and all $A,B\in G$ we have that
	\begin{equation}\label{starToKronProdIdentity}
		\left\{P_{i,I}^t(k,\cdot)\star P_{i,I}^t(r,\cdot)\right\}(A^*B) = \left\{\int_G P_{i,C}^t\otimes P_{i,C}^t d\eta(C)\right\}((k,A),(r,B)),
	\end{equation}
where the expression on the right hand side of\eqref{starToKronProdIdentity} is the function in~$\Hspace\times \Hspace$ defined by
\begin{equation}\label{probKerKronProdDef}
	\left\{\int_G P_{i,C}^t\otimes P_{i,C}^td\eta(C)\right\}((k,A),(r,B)) \triangleq\int_G P_{i,C}^t(k,A)\cdot  P_{i,C}^t(r,B)d\eta(C).
\end{equation}
	Furthermore, denoting 
	\begin{equation}\label{probEmbeddingProp:alternativeDistProb2}
		\tilde{E}_{p,t}(i,j) = \norm{ P_{i,I}^t\gconv P_{i,I}^t-P_{j,I}^t\gconv P_{j,I}^t}_{L^2\left([N]^2\times G\right),d\eta/D\otimes D},
	\end{equation}
we have that
\begin{equation}\label{probEmbeddingProp:alternativeDistProb}
\tilde{E}_{p,t}(i,j) = \norm{\int_G P_{i,C}^t\otimes P_{i,C}^t d\eta(C)-\int_G P_{j,C}^t\otimes P_{j,C}^t d\eta(C)}^2_{(\mathcal{H},d\eta/D)\times (\mathcal{H},d\eta/D)},
\end{equation}
where~$(\mathcal{H},d\eta/D)$ is the Hilbert space of square integrable functions over~$[N]\times G$ with inner product given by~\eqref{eqvDmaps:weightedHspaceDotProdDef}.
\end{lemma}
\begin{proof}
	First, by Lemma~\ref{eqvDmaps:eqvActLemma}, for any $(k,A),(r,B)\in [N]\times G$, we have that
	\begin{align}
		\int_G P_{i,C}^t(k,A)\cdot P_{i,C}^t(r,B)d\eta(C) &= 
		\int_G P_{i,I}^t(k,C^*A)\cdot P_{i,I}^t(r,C^*B)d\eta(C) \label{invDmaps:cCorrIdentityEq0}\\
		&= \int_G P_{i,I}^t(k,Q)\cdot P_{i,I}^t(r,QA^*B)d\eta(AQ^*)\label{invDmaps:cCorrIdentityEq1} \\
		&= \int_G P_{i,I}^t(k,Q)\cdot P_{i,I}^t(r,QA^*B)d\eta(Q) \label{kronProdKerIsCrossCorrEq5}\\
		&=\left\{P_{i,I}^t(k,\cdot)\star P_{i,I}^t(r,\cdot)\right\}(A^*B)\label{invDmaps:cCorrIdentity},
	\end{align}
	where in passing to~\eqref{invDmaps:cCorrIdentityEq1} we used the change of variables~$Q=C^*A$, then in passing to~\eqref{kronProdKerIsCrossCorrEq5} we used the left-invariance property~\eqref{secLieGroupAction:leftInvar} of the Haar measure $\eta$ combined with~\eqref{secLieGroupAction:conjInvar}, and finally, we used~\eqref{invDmaps:coupledRandomWalkDensity1} in passing to~\eqref{invDmaps:cCorrIdentity}.
	
	Next, by~\eqref{kronProdDef} and~\eqref{invDmaps:cCorrIdentity}, we get that
	\begin{align}
		&\norm{\int_G P_{i,C}^t\otimes P_{i,C}^td\eta(C)-\int_G P_{j,C}^t\otimes P_{j,C}^td\eta(C)}_{\Hspace\times \Hspace}^2 \nonumber \\
		=&\sum_{k,r=1}^N\int_G \int_G \left(\int_G P_{i,C}^t(k,A)\cdot P_{i,C}^t(r,B)d\eta(C)-\int_G P_{j,C}^t(k,A)\cdot P_{j,C}^t(r,B)d\eta(C)\right)^2 \frac{d\eta(B)}{D_{jj}}\frac{d\eta(A)}{D_{ii}} \nonumber \\
		=&\sum_{k,r=1}^N\int_G \int_G \left(\left\{P_{i,I}^t(k,\cdot)\star P_{i,I}^t(r,\cdot)\right\}(A^*B)-\left\{P_{j,I}^t(k,\cdot)\star P_{j,I}^t(r,\cdot)\right\}(A^*B)\right)^2 \frac{d\eta(B)}{D_{jj}}\frac{d\eta(A)}{D_{ii}}. \label{kronDistLastEq}
	\end{align}
Then, using the change of variables~$R= A^*B$, the last expression in~\eqref{kronDistLastEq} becomes
\begin{align}
		&\sum_{k,r=1}^N\int_G \int_G \left(\left\{P_{i,I}^t(k,\cdot)\star P_{i,I}^t(r,\cdot)\right\}(R)-\left\{P_{j,I}^t(k,\cdot)\star P_{j,I}^t(r,\cdot)\right\}(R)\right)^2\frac{d\eta(AR)}{D_{jj}}\frac{d\eta(A)}{D_{jj}} \label{kerDistIsCrossCorrDistEq3} \\
		=&\sum_{k,r=1}^N\int_G \int_G \left(\left\{P_{i,I}^t(k,\cdot)\star P_{i,I}^t(r,\cdot)\right\}(R)-\left\{P_{j,I}^t(k,\cdot)\star P_{j,I}^t(r,\cdot)\right\}(R)\right)^2\frac{d\eta(R)}{D_{jj}}\frac{d\eta(A)}{D_{ii}} \label{kerDistIsCrossCorrDistEq4} \\
		=&\sum_{k,r=1}^N\int_G \left(\left\{P_{i,I}^t(k,\cdot)\star P_{i,I}^t(r,\cdot)\right\}(R)-\left\{P_{j,I}^t(k,\cdot)\star P_{j,I}^t(r,\cdot)\right\}(R)\right)^2\frac{d\eta(R)}{D_{ii}\cdot D_{jj}}\label{kerDistIsCrossCorrDistEq5} \\
		=& \norm{ P_{i,I}^t\gconv P_{i,I}^t-P_{j,I}^t\gconv P_{j,I}^t}_{L^2\left([N]^2\times G\right),d\eta/D\otimes D}^2 = \tilde{E}_{p,t}(i,j), 
	\end{align}
	where we used the translation-invariance property~\eqref{secLieGroupAction:leftInvar} of the Haar measure in passing to~\eqref{kerDistIsCrossCorrDistEq4}, and in passing to~\eqref{kerDistIsCrossCorrDistEq5} we used~\eqref{secLieGroupAction:probMeasure} coupled with the fact that the integrand in~\eqref{kerDistIsCrossCorrDistEq4} is independent of the integration variable~$A$.
\end{proof}
Now, let us write the function in~\eqref{probKerKronProdDef}, appearing in the quantity~$\tilde{E}_{p,t}$ in~\eqref{probEmbeddingProp:alternativeDistProb}  
in terms of the eigenfunctions and eigenvalues of the operator~$P_{op}^t$ in~\eqref{eqvDmaps:probOperator}.
To that end, first consider the function 
\begin{equation}\label{symKerKronProdDef}
	\left\{\int_G S_{i,C}^t\otimes S_{i,C}^td\eta(C)\right\}((k,A),(r,B)) \triangleq\int_G S_{i,C}^t(k,A)\cdot  S_{i,C}^t(r,B)d\eta(C),
\end{equation}
where
\begin{equation}\label{eqvDmaps:SiAtcdotcdotDef}
	S^t_{i,A}(k,C)\triangleq  S^t((i,A),(k,C)), \quad (k,C)\in[N]\times G,
\end{equation}
and~$S^t$ was defined in~\eqref{eqvDmaps:symPowtKernelDef} via~\eqref{eqvDmaps:symOpDef}. Note that the function~$S^t_{i,C}$ is related to~$P^t_{i,A}$ through~\eqref{eqvDmaps:symOpDef}, \eqref{eqvDmaps:PtKernelDef}, and~\eqref{eqvDmaps:PiAtcdotcdotDef}.

By~\eqref{eqvDmaps:symOpDef}, it follows by induction over~$t$ that
\begin{equation}\label{symKerSymProp}
	S^t_{i,C}(k,A)=S^t((i,C),(k,A)) = S^t((k,A),(i,C)) =S^t_{k,A}(i,C),
\end{equation}
for all~$(i,C),(k,A)\in [N]\times G$. Therefore, by~\eqref{symKerKronProdDef}, \eqref{eqvDmaps:SiAtcdotcdotDef}, \eqref{symKerSymProp}, and \eqref{eqvDmaps:symmetricOp_ij}, for all~$(k,A),(r,B)\in [N]\times G$ we have that
\small
\begin{align}\label{invDmaps:kerDistPropEq1}
		&\left\{\int_G S_{i,C}^t\otimes S_{i,C}^t d\eta(C)\right\}((k,A),(r,B)) =\int_G S_{i,C}^t(k,A)\cdot S_{i,C}^t(r,B)d\eta(C)\nonumber\\
		=&  \int_G S^t((i,C),(k,A))\cdot S^t((i,C),(r,B))=  \int_G S^t((k,A),(i,C))\cdot S^t((i,C),(r,B)) \nonumber \\
	=&\int_G \sum_{\ell,\ell'\in \I_G} \sum_{m,m'=1}^{d_\ell} \sum_{n,n'=1}^N \lambda_{n,\ell}^{t}\lambda_{n',\ell'}^{t}\Phi^{(s)}_{\ell,m,n}(k,A) \cdot \overline{\Phi^{(s)}_{\ell,m,n}(i,C)}
	\Phi^{(s)}_{\ell',m',n'}(i,C) \cdot \overline{\Phi^{(s)}_{\ell',m',n'}(r,B)}d\eta(C)  	\nonumber \\
	=& \sum_{\ell,\ell'\in \I_G} \sum_{m,m'=1}^{d_\ell} \sum_{n,n'=1}^N \lambda_{n,\ell}^{t}\lambda_{n',\ell'}^{t}\Phi^{(s)}_{\ell,m,n}(k,A) \cdot\overline{\Phi^{(s)}_{\ell',m',n'}(r,B)}\cdot \int_G\overline{\Phi^{(s)}_{\ell,m,n}(i,C)}\cdot 
	\Phi^{(s)}_{\ell',m',n'}(i,C)d\eta(C).
\end{align}
\normalsize
Now, by \eqref{secLieGroupAction:SchurOrthogonality}, we have that
\begin{equation}\label{outerProdOrthogonality}
	\left(\int_G \left(U^\ell_{m,\cdot}(C)\right)^T\overline{U^{\ell'}_{m',\cdot}(C)}d\eta(C)\right) = d_\ell^{-1}\cdot\delta_{\ell \ell'}\delta_{mm'} I_{d_\ell\times d_\ell}, \quad \ell\in \I_G. 
\end{equation}
Hence, by using \eqref{eqvDmaps:eigenFuncsTilde} and~\eqref{outerProdOrthogonality}, the following expression, that appears in~\eqref{invDmaps:kerDistPropEq1}, evaluates as 
\begin{align}\label{Phi_lmnProdOverG}
	&\int_G \overline{\Phi^{(s)}_{\ell,m,n}(i,C)}\cdot\Phi^{(s)}_{\ell',m',n'}(i,C)d\eta(C)  \nonumber \\
	=&\sqrt{d_\ell\cdot d_{\ell'}}\int_G \overline{\left(\overline{U^\ell_{m,\cdot}(C)}e^i(v^{(s)}_{n,\ell})\right)}\cdot \overline{U^{\ell'}_{m',\cdot}(C)}e^i(v^{(s)}_{n',\ell'})d\eta(C)   \nonumber \\		 
	=&\sqrt{d_\ell\cdot d_{\ell'}}\left(e^i(v^{(s)}_{n,\ell})\right)^*\left(\int_G \left(U^\ell_{m,\cdot}(C)\right)^T\overline{U^{\ell'}_{m',\cdot}(C)}d\eta(C)\right) e^i(v^{(s)}_{n',\ell'}) \nonumber \\ 
	=&\sqrt{d_\ell\cdot d_{\ell'}}\frac{1}{d_\ell}\cdot\left(e^i(v^{(s)}_{n,\ell})\right)^* e^i(v^{(s)}_{n',\ell'})\delta_{\ell \ell'}\delta_{mm'} = \left(e^i(v^{(s)}_{n,\ell})\right)^* e^i(v^{(s)}_{n',\ell'})\delta_{\ell \ell'}\delta_{mm'}.
\end{align}
Plugging~\eqref{Phi_lmnProdOverG} back into \eqref{invDmaps:kerDistPropEq1}, and using~\eqref{kronProdDef} we obtain that
\begin{align}\label{invDmaps:kerDistPropEq2}
	&\left\{\int_G S_{i,C}^t\otimes S_{i,C}^t d\eta(C)\right\}((k,A),(r,B)) \nonumber\\
	=& \sum_{\ell,\ell'\in \I_G} \sum_{m,m'=1}^{d_\ell} \sum_{n,n'=1}^N \lambda_{n,\ell}^{t}\lambda_{n',\ell'}^{t}\Phi^{(s)}_{\ell,m,n}(k,A) \cdot\overline{\Phi^{(s)}_{\ell',m',n'}(r,B)}\cdot \left(e^i(v^{(s)}_{n,\ell})\right)^* e^i(v^{(s)}_{n',\ell'})\delta_{\ell \ell'}\delta_{mm'}\nonumber \\ 
	=&\sum_{\ell\in \I_G} \sum_{n,n'=1}^N  \lambda_{n,\ell}^{t}\lambda_{n',\ell}^{t} \left(e^i(v^{(s)}_{n,\ell})\right)^* e^i(v^{(s)}_{n',\ell})\sum_{m=1}^{d_\ell}\Phi^{(s)}_{\ell,m,n}(k,A) \cdot\overline{\Phi^{(s)}_{\ell,m,n'}(r,B)}\nonumber \\
	=&\sum_{\ell\in \I_G} \sum_{n,n'=1}^N  \lambda_{n,\ell}^{t}\lambda_{n',\ell}^{t} \left(e^i(v^{(s)}_{n,\ell})\right)^* e^i(v^{(s)}_{n',\ell})\sum_{m=1}^{d_\ell}\left\{\Phi^{(s)}_{\ell,m,n} \otimes\left(\Phi^{(s)}_{\ell,m,n'}\right)^*\right\}((k,A)(r,B)).
\end{align}

On the other hand, by~\eqref{kronProdDef}, \eqref{eqvDmaps:PiAtcdotcdotDef}, \eqref{eqvDmaps:symOpDefPowt}, and~\eqref{eqvDmaps:SiAtcdotcdotDef} we get that
\begin{align}\label{invDmaps:probKerDistPropEq2}
	&\left\{\int_G P_{i,C}^t\otimes P_{i,C}^t d\eta(C)\right\}((k,A),(r,B)) =\int_G P_{i,C}^t(k,A)\cdot P_{i,C}^t(r,B)d\eta(C)= \nonumber \\
	=& \int_G P^t((i,C),(k,A))\cdot P^t((i,C),(r,B))d\eta(C)  \nonumber\\
	=&\frac{\sqrt{D_{kk}}\sqrt{D_{rr}}}{D_{ii}}\int_G S^t((i,C),(k,A))\cdot S^t((i,C),(r,B))d\eta(C)\nonumber\\
	=& \frac{\sqrt{D_{kk}}\sqrt{D_{rr}}}{D_{ii}}\left\{\int_G S_{i,C}^t\otimes S_{i,C}^t d\eta(C)\right\}((k,A),(r,B)). 
\end{align}
Next, by using~\eqref{invDmaps:kerDistPropEq2} together with~\eqref{ejTilde2ej}, the expression in~\eqref{invDmaps:probKerDistPropEq2} evaluates as
\begin{align}\label{DhalfPhilmnKronIdentity}
	&\frac{\sqrt{D_{kk}}\sqrt{D_{rr}}}{D_{ii}}\left\{\int_G S_{i,C}^t\otimes S_{i,C}^t d\eta(C)\right\}((k,A),(r,B))  \nonumber \\
	=&\frac{\sqrt{D_{kk}}\sqrt{D_{rr}}}{D_{ii}}\sum_{\ell\in \I_G} \sum_{n,n'=1}^N  \lambda_{n,\ell}^{t}\lambda_{n',\ell}^{t} \left(e^i(v^{(s)}_{n,\ell})\right)^* e^i(v^{(s)}_{n',\ell})\sum_{m=1}^{d_\ell}\left\{\Phi^{(s)}_{\ell,m,n} \otimes\left(\Phi^{(s)}_{\ell,m,n'}\right)^*\right\}((k,A)(r,B))\nonumber\\
	=&\sqrt{D_{kk}}\sqrt{D_{rr}}\sum_{\ell\in \I_G} \sum_{n,n'=1}^N  \lambda_{n,\ell}^{t}\lambda_{n',\ell}^{t} \left(e^i(v^{(p)}_{n,\ell})\right)^* e^i(v^{(p)}_{n',\ell})\sum_{m=1}^{d_\ell}\left\{\Phi^{(s)}_{\ell,m,n} \otimes\left(\Phi^{(s)}_{\ell,m,n'}\right)^*\right\}((k,A)(r,B))\nonumber \\
	=&\sum_{\ell\in \I_G} \sum_{n,n'=1}^N  \lambda_{n,\ell}^{t}\lambda_{n',\ell}^{t} \left(e^i(v^{(p)}_{n,\ell})\right)^* e^i(v^{(p)}_{n',\ell})\sum_{m=1}^{d_\ell}\left\{D^{1/2}\Phi^{(s)}_{\ell,m,n}\otimes \left(D^{1/2}\Phi^{(s)}_{\ell,m,n'}\right)^*\right\}((k,A),(r,B)),
\end{align}
where in the last equality we used that by~\eqref{kronProdDef} and~\eqref{GinvDef:DMatAction}, we have
\begin{align}
	&\left\{D^{1/2}\Phi^{(s)}_{\ell,m,n}\otimes \left(D^{1/2}\Phi^{(s)}_{\ell,m,n'}\right)^*\right\}((k,A),(r,B))=  \left\{D^{1/2} \Phi^{(s)}_{\ell,m,n'}\right\}(k,A)\cdot \left\{\overline{D^{1/2} \Phi^{(s)}_{\ell,m,n'}}\right\}(r,B) \nonumber \\
	=&\sqrt{D_{kk}}\cdot \Phi^{(s)}_{\ell,m,n'}(k,A)\cdot \overline{\sqrt{D_{rr}}\cdot  \Phi^{(s)}_{\ell,m,n'}}(r,B)\\ \nonumber
	=&\sqrt{D_{kk}}\cdot\sqrt{D_{rr}}\left\{\Phi^{(s)}_{\ell,m,n}\otimes \left(\Phi^{(s)}_{\ell,m,n'}\right)^*\right\}((k,A),(r,B)).
\end{align}
Thus, by~\eqref{invDmaps:probKerDistPropEq2}, and~\eqref{DhalfPhilmnKronIdentity} we have that
\begin{align}\label{PtKronProd2PhilmnKronProd}
	\int_G P_{i,C}^t\otimes P_{i,C}^t d\eta(C) =
	\sum_{\ell\in \I_G} \sum_{n,n'=1}^N  \lambda_{n,\ell}^{t}\lambda_{n',\ell}^{t} \left(e^i(v^{(p)}_{n,\ell})\right)^* e^i(v^{(p)}_{n',\ell})\sum_{m=1}^{d_\ell}D^{1/2}\Phi^{(s)}_{\ell,m,n}\otimes \left(D^{1/2}\Phi^{(s)}_{\ell,m,n'}\right)^*.
\end{align}
Plugging~\eqref{PtKronProd2PhilmnKronProd} into~\eqref{probEmbeddingProp:alternativeDistProb}, we get that 
\begin{align}\label{invDmaps:probKerDistPropEq3}
	\tilde{E}_{p,t}^2(i,j) =\bigg\Vert	&\sum_{\ell\in \I_G} \sum_{n,n'=1}^N  \lambda_{n,\ell}^{t}\lambda_{n',\ell}^{t} \left(e^i(v^{(p)}_{n,\ell})\right)^* e^i(v^{(p)}_{n',\ell})\sum_{m=1}^{d_\ell}D^{1/2}\Phi^{(s)}_{\ell,m,n}\otimes \left(D^{1/2}\Phi^{(s)}_{\ell,m,n'}\right)^*-\nonumber\\
	&\sum_{\ell\in \I_G} \sum_{n,n'=1}^N  \lambda_{n,\ell}^{t}\lambda_{n',\ell}^{t} \left(e^j(v^{(p)}_{n,\ell})\right)^* e^j(v^{(p)}_{n',\ell})\sum_{m=1}^{d_\ell}D^{1/2}\Phi^{(s)}_{\ell,m,n}\otimes \left(D^{1/2}\Phi^{(s)}_{\ell,m,n'}\right)^*
	\bigg\Vert^2_{(\mathcal{H},d\eta/D)\times (\mathcal{H},d\eta/D)}.
\end{align}
Next, denoting 
\begin{equation*}
	a^{(p)}_{\ell,n,n'} =  \lambda_{n,\ell}^{t}\lambda_{n',\ell}^{t} \left(\left(e^i(v^{(p)}_{n,\ell})\right)^* e^i(v^{(p)}_{n',\ell})-
	\left(e^j(v^{(p)}_{n,\ell})\right)^* e^j(v^{(p)}_{n',\ell})\right),
\end{equation*}
and using~\eqref{invDmaps:probKerDistPropEq3} and \eqref{probEmbeddingProp:alternativeDistProb2}, gives us that
\begin{align}\label{invDmaps:probKerDistPropEq4}
		&\norm{ P_{i,I}^t\gconv P_{i,I}^t-P_{j,I}^t\gconv P_{j,I}^t}^2_{L^2\left([N]^2\times G\right),d\eta/D\otimes D} = \tilde{E}_{p,t}^2(i,j) \\ \nonumber	
	=&\Vert	\sum_{\ell\in \I_G}  \sum_{n,n'=1}^N a^{(p)}_{\ell,n,n'}\cdot \sum_{m=1}^{d_\ell}D^{1/2}\Phi^{(s)}_{\ell,m,n}\otimes \left(D^{1/2}\Phi^{(s)}_{\ell,m,n'}\right)^*\Vert^2_{(\mathcal{H},d\eta/D)\times (\mathcal{H},d\eta/D)}\nonumber \\
	=&\sum_{\ell\in \I_G} \sum_{n,n'=1}^Nd_{\ell}\cdot  \abs{a^{(p)}_{\ell,n,n'}}^2 
	\nonumber \\ =&\sum_{\ell\in \I_G} \sum_{n,n'=1}^N d_{\ell}\cdot \lambda_{n,\ell}^{2t}\lambda_{n',\ell}^{2t} \left|\left(e^i(v^{(p)}_{n,\ell})\right)^* e^i(v^{(p)}_{n',\ell})-
	\left(e^j(v^{(p)}_{n,\ell})\right)^* e^j(v^{(p)}_{n',\ell})\right|^2 \nonumber \\
	=& \sum_{\ell\in \I_G} \sum_{n,n'=1}^N d_{\ell}\cdot \lambda_{n,\ell}^{2t}\lambda_{n',\ell}^{2t} \left|\dprod{e^i(v^{(p)}_{n,\ell})}{ e^i(v^{(p)}_{n',\ell})}-
	\dprod{e^j(v^{(p)}_{n,\ell})}{e^j(v^{(p)}_{n',\ell})}\right|^2 \nonumber \\
	=&\norm{\Psi_t^{(p)}(i)-\Psi_t^{(p)}(j)}^2_{\ell^2}= E_{p,t}^2(i,j), 
\end{align}
where the third equality stems from the fact that since the functions $\left\{\Phi^{(s)}_{\ell,m,n}\otimes \Phi^{(s)}_{\ell,m,n'} \right\}$ are orthonormal in~$\mathcal{H}\times \mathcal{H} $, then the functions $D^{1/2}\Phi^{(s)}_{\ell,m,n}\otimes \left(D^{1/2}\Phi^{(s)}_{\ell,m,n'}\right)^*$ are orthonormal in~$(\mathcal{H},d\eta/D)\times (\mathcal{H},d\eta/D)$, and the last equality is due to~\eqref{invDmaps:probInvEmbeding}. 

\subsection{Proof of Proposition~\ref{invDmaps:genDiplacementLemma}}\label{invDmaps:genDiplacementLemmaPrf}
We begin by proving the following result. 
\begin{lemma}\label{displaceMeasureLemma}
	Let~$\mu$ and~$\nu$ be a pair of Borel measures over~$G$ given by  
	\begin{equation}\label{measuresOverG}
		\mu(H) = \int_H f(A)d\eta(A), \quad \nu(H) = \int_H g(A)d\eta(A), \quad H\subseteq G,
	\end{equation}
	where~$\eta$ is the Haar measure, and~$f$ and~$g$ are non-negative functions\footnote{The function~$f$ is known as the Radon-Nikodym derivative~$\frac{d\mu}{d\eta}$  of~$\mu$ with respect to~$\eta$.} over~$G$ such that $\mu(G),\nu(G)<\infty$. 
	Let~$H\subseteq G$ be a Borel-measurable subset, and consider the subset~$H_\times\subseteq G\times G$ defined by
	\begin{equation}\label{HtimesDef}
		H_{\times} \triangleq \left\{ (A,B)\in G\times G: \; A^*B\in H \right\}. 
	\end{equation}
	Then, we have that
	\begin{equation}\label{prodMeasureOfHx}
		\left\{\mu\otimes \nu\right\}(H_\times) = \int_H f\star g(C)d\eta(C), 
	\end{equation}
	where $\mu\otimes \nu$ is the product measure over~$G\times G$ defined by 
	\begin{equation}
		\left\{\mu\otimes \nu\right\}(H_1\times H_2) = \mu(H_1)\cdot \nu(H_2), \quad H_1,H_2\in G. 
	\end{equation}
\end{lemma}
\begin{proof}
	Let us denote by $\mathbbm{1}_{H_\times}(A,B)$ the indicator function of~$H_\times$ over the space~$G\times G$. 
	Then, we have that
	\begin{align}
		\left\{\mu\otimes \nu\right\}(H_\times) &= \int_{G\times G} \mathbbm{1}_{H_\times}(A,B)d\mu(A)d\nu(B)  =\int_G \left(\int_G \mathbbm{1}_{H_\times}(A,B)d\nu(B)\right)d\mu(A) \nonumber \\
		&= \int_G \nu\left\{B\in G:\; A^*B\in H\right\}d\mu(A)	=\int_G \nu(A\cdot H)d\mu(A)\nonumber\\
		&= \int_G \left(\int_{A\cdot H}g(B)d\eta(B)\right)d\mu(A). 	
	\end{align}
	Applying the change of variables~$C=A^*B$, and using the translation-invariance property~\eqref{secLieGroupAction:leftInvar} of~$\eta$, we obtain that
	\begin{align}
		\left\{\mu\otimes \nu\right\}(H_\times)&= \int_G \left(\int_{H}g(AC)d\eta(C)\right)d\mu(A)= \int_G \left(\int_{H}f(A)g(AC)d\eta(C)\right)d\eta(A)\nonumber \\
		&= \int_H\left(\int_G f(A)g(AC)d\eta(A)\right)d\eta(C) = \int_H f\star g(C)d\eta(C).	
	\end{align}
\end{proof}

Now, since $X_{1,t}$ and~$X_{2,t}$ are i.i.d with probability distributions given by~$P_{i,I}$ from~\eqref{eqvDmaps:PiAtcdotcdotDef}, we get by using~\eqref{HtimesDef} that
\begin{align}\label{genDisplaceLemmaPrf:eq1}
	\prob{H_{kr}} &= \int_{\left\{ (A,B)\in G\times G: \; A^*B\in H \right\}}P_{i,I}^t(k,A)\cdot P_{i,I}^t(r,B)d\eta(A)d\eta(B)\nonumber \\ 		
	& =  \int_{H_\times}P_{i,I}^t(k,A)\cdot P_{i,I}^t(r,B)d\eta(A)d\eta(B).
\end{align}
Now, let us define the pair of measures over~$G$
\begin{equation}\label{PmeasuresOverG}
	\mu(S) = \int_S P^t_{i,I}(k,A)d\eta(A), \quad \nu(S) = \int_S P^t_{i,I}(r,B)d\eta(B), \quad S\subseteq G.
\end{equation}
Then, continuing from~\eqref{genDisplaceLemmaPrf:eq1} and using Lemma~\ref{displaceMeasureLemma} together with~\eqref{invDmaps:coupledRandomWalkDensity}, we obtain that
\begin{align}
	\prob{H_{kr}} &= \int_{H_\times} 1 d\mu(A)d\nu(B) = \int_{G\times G} \mathbbm{1}_{H_\times}(A,B) d\mu(A)d\nu(B) = \mu \otimes \nu (H_\times) \nonumber \\
	&= \int_H \left\{P^t_{i,I}(k,\cdot)\star P^t_{i,I}(r,\cdot)\right\}(R)d\eta(R)\nonumber\\
	&=\int_H\left\{P_{i,I}^t \gconv P_{i,I}^t\right\}(k,r,R)d\eta(R),
\end{align}	
as claimed. 

\bibliographystyle{plain}
\bibliography{GDM}

\end{document}